\documentclass[preprint,3p,11pt,sort&compress]{elsarticle}

\usepackage{times}
\usepackage{helvet}
\usepackage{courier}

\usepackage{xcolor}
\definecolor{mydarkblue}{rgb}{0.0,0.08,0.45}
\usepackage[pdftex,breaklinks,colorlinks,pdfdisplaydoctitle,citecolor=mydarkblue,linkcolor=mydarkblue,urlcolor=mydarkblue,pdfborder={0
0 0}, baseurl=http://]{hyperref} 

\usepackage{array}
\usepackage{booktabs}
\usepackage{multirow}
\usepackage{threeparttable}

\usepackage[neverdecrease]{paralist}

\usepackage{amsmath, amssymb, amsthm, latexsym}

\usepackage{braket}

\usepackage{extpfeil}

\usepackage[mathscr]{eucal}

\usepackage{pifont}
\newcommand{\tick}{\ding{51}}

\usepackage{tikz}
\usetikzlibrary{arrows}

\usepackage{pdflscape}

\theoremstyle{plain}
\newtheorem{theorem}{Theorem}
\newtheorem{lemma}[theorem]{Lemma}
\newtheorem{proposition}[theorem]{Proposition}
\newtheorem{corollary}[theorem]{Corollary}
\newtheorem{property}{Syntactic Property}

\theoremstyle{definition}
\newtheorem{definition}[theorem]{Definition}
\newtheorem{example}[theorem]{Example}
\newtheorem{remark}[theorem]{Remark}

\usepackage{ifthen}
\newcommand{\brifnotempty}[1]{\ifthenelse{\equal{#1}{}}{}{ \br{#1}}}
\newenvironment{lemma*}[2][]
	{\pagebreak[2] \par \noindent \textbf{Lemma~\ref{#2}}\brifnotempty{#1}.\it}{\par}
\newenvironment{theorem*}[2][]
	{\pagebreak[2] \par \noindent \textbf{Theorem~\ref{#2}}\brifnotempty{#1}.\it}{\par}
\newenvironment{proposition*}[2][]
	{\pagebreak[2] \par \noindent \textbf{Proposition~\ref{#2}}\brifnotempty{#1}.\it}{\par}
\newenvironment{corollary*}[2][]
	{\pagebreak[2] \par \noindent \textbf{Corollary~\ref{#2}}\brifnotempty{#1}.\it}{\par}

\newenvironment{textitem}
	{\setlength{\pltopsep}{.7ex}\setlength{\plitemsep}{.7ex}\begin{compactitem}}
	{\end{compactitem}}

\newenvironment{textenum}[1][]
	{\setlength{\pltopsep}{.7ex}\setlength{\plitemsep}{.7ex}\begin{compactenum}[#1]}
	{\end{compactenum}}

\newcommand{\llbracket}{[\kern-.3ex[}
\newcommand{\llbracketscr}{[\kern-.25ex[}
\newcommand{\Llbracket}{\left[\kern-.6ex\left[}
\newcommand{\rrbracket}{]\kern-.3ex]}
\newcommand{\rrbracketscr}{]\kern-.25ex]}
\newcommand{\Rrbracket}{\right]\kern-.6ex\right]}
\newcommand{\sqbbr}[1]{\llbracket\hspace{.2ex}#1\hspace{.2ex}\rrbracket}

\newcommand{\llangle}{\langle\kern-.5ex\langle}
\newcommand{\llanglescr}{\langle\kern-.4ex\langle}
\newcommand{\Llangle}{\left\langle\kern-.8ex\left\langle}
\newcommand{\rrangle}{\rangle\kern-.5ex\rangle}
\newcommand{\rranglescr}{\rangle\kern-.4ex\rangle}
\newcommand{\Rrangle}{\right\rangle\kern-.8ex\right\rangle}
\newcommand{\anbbr}[1]{\llangle #1 \rrangle}
\newcommand{\anbbrscr}[1]{\llanglescr #1 \rranglescr}
\newcommand{\Anbbr}[1]{\Llangle #1 \Rrangle}

\newcommand{\llbr}{(\mkern-3.5mu(}

\newcommand{\rrbr}{)\mkern-3.5mu)}

\newcommand{\bbr}[1]{\llbr #1 \rrbr}

\newcommand{\br}[1]{(#1)}
\newcommand{\Br}[1]{\left(#1\right)}

\newcommand{\tpl}[1]{\br{#1}}

\newcommand{\seq}[1]{\langle #1 \rangle}

\newcommand{\lng}{n}
\newcommand{\lnga}{m}
\newcommand{\lngb}{n}
\newcommand{\lia}{i}
\newcommand{\lib}{j}
\newcommand{\lic}{k}

\newcommand{\pws}[1]{2^{#1}}

\newcommand{\na}{\textsf{\small n/a}}

\newextarrow{\myxlongequal}{3300}{\Relbar\Relbar\Relbar}

\newcommand{\mlthen}{\Rightarrow}

\newcommand{\mymodels}{\mathrel\mid\joinrel=}
\newcommand{\ent}{\mymodels}
\newcommand{\nent}{\not\ent}
\newcommand{\eq}{\equiv}

\newcommand{\prefix}[1]{\textsf{\small #1}}
\newcommand{\prefm}[1]{\ensuremath{\scriptscriptstyle \mathsf{#1}}}
\newcommand{\prefixm}[1]{\prefm{#1}}

\newcommand{\pstl}[2][]{\textsf{#1#2}}

\newcommand{\bigland}{\bigwedge}
\newcommand{\biglor}{\bigvee}

\newcommand{\frm}{\phi}
\newcommand{\frma}{\phi}
\newcommand{\frmb}{\psi}
\newcommand{\frmu}{\mu}
\newcommand{\frmv}{\nu}

\newcommand{\sfrm}{S}

\newcommand{\thr}{\mathcal{T}}

\newcommand{\twi}{I}
\newcommand{\twia}{I}
\newcommand{\twib}{J}
\newcommand{\twic}{K}

\newcommand{\twis}{\mathscr{I}}

\newcommand{\stwi}{\mathcal{M}}

\newcommand{\sstwia}{\mathcal{S}}
\newcommand{\sstwib}{\mathcal{T}}

\renewcommand{\mod}[1]{\sqbbr{#1}}

\newcommand{\atm}{p}
\newcommand{\atma}{p}
\newcommand{\atmb}{q}
\newcommand{\atmc}{r}
\newcommand{\atmd}{s}
\newcommand{\atms}{\mathscr{A}}

\newcommand{\olits}{\mathscr{L}}

\newcommand{\lcmp}[1]{\overline{#1}}

\newcommand{\lit}{L}
\newcommand{\lits}{\olits}
\newcommand{\slit}{S}
\newcommand{\slitp}[1][\slit]{#1^+}
\newcommand{\slitn}[1][\slit]{#1^-}
\newcommand{\slith}{H}
\newcommand{\slitb}{B}

\newcommand{\lpnot}{\mathop{\sim}}
\newcommand{\lpif}{\leftarrow}

\newcommand{\rl}{\pi}
\newcommand{\rla}{\pi}
\newcommand{\rlb}{\sigma}
\newcommand{\rlc}{\rho}

\newcommand{\hrl}[1][\rl]{\textrm{\sf \small H}_{#1}}

\newcommand{\hrla}{\hrl[\rla]}
\newcommand{\hrlb}{\hrl[\rlb]}
\newcommand{\hrlc}{\hrl[\rlc]}
\newcommand{\hrlp}{\slitp[\hrl]}
\newcommand{\hrln}{\slitn[\hrl]}

\newcommand{\brl}[1][\rl]{\textrm{\sf \small B}_{#1}}

\newcommand{\brla}{\brl[\rla]}
\newcommand{\brlb}{\brl[\rlb]}
\newcommand{\brlc}{\brl[\rlc]}
\newcommand{\brlp}{\slitp[\brl]}
\newcommand{\brln}{\slitn[\brl]}

\newcommand{\prg}{\mathcal{P}\hspace{-.1ex}}
\newcommand{\prga}{\mathcal{P}\hspace{-.1ex}}
\newcommand{\prgb}{\mathcal{Q}}
\newcommand{\prgc}{\mathcal{R}}
\newcommand{\prgu}{\mathcal{U}}
\newcommand{\prgv}{\mathcal{V}}

\newcommand{\modsm}[1]{\sqbbr{#1}_{\mSM}}

\newcommand{\labbu}[1]{\pstl[\small]{(B#1)}}
\newcommand{\labbutwotop}{\pstl[\small]{(B2.$\top$)}}
\newcommand{\bu}[1]{\hyperref[tbl:belief update postulates]{\labbu{#1}}}
\newcommand{\butwotop}{\hyperref[tbl:belief update postulates]{\labbutwotop}}

\newcommand{\incorp}{\mathsf{incorporate}}

\newcommand{\uopb}{\mathbin{\diamond}}
\newcommand{\biguopb}{\mathop{\raisebox{-1pt}{$\Diamond$}}}

\newcommand{\kb}{\mathcal{B}}
\newcommand{\kba}{\mathcal{B}}
\newcommand{\kbb}{\mathcal{C}}
\newcommand{\kbu}{\mathcal{U}}
\newcommand{\kbv}{\mathcal{V}}
\newcommand{\dkb}{\mathcal{D}}

\newcommand{\uopf}{\mathbin{\circ}}
\newcommand{\biguopf}{\mathop{\bigcirc}}

\newcommand{\remof}[1]{\mathsf{rem}(#1)}
\newcommand{\srem}{\mathcal{R}}
\newcommand{\uopwidtio}{\mathbin{\uopf_{\prefixm{widtio}}}}
\newcommand{\uopcross}{\mathbin{\uopf_{\prefixm{cp}}}}
\renewcommand{\sfn}{s}
\newcommand{\uopbold}{\mathbin{\uopf_{\prefixm{bold}}^\sfn}}

\newcommand{\dprg}{\boldsymbol{P}}

\newcommand{\expv}[1]{#1}

\newcommand{\confl}[2]{\Join^{#1}_{#2}}

\newcommand{\all}[1]{\mathsf{all}(#1)}

\newcommand{\rej}[2][]{\mathsf{rej}_{#1}(#2)}

\renewcommand{\S}[2][S]{\prefix{#1}\protect\nobreakdash#2\hspace{0pt}}

\newcommand{\mS}{\prefixm{S}}

\newcommand{\mods}[1]{\sqbbr{#1}_{\mS}}

\newcommand{\JU}[1]{\prefix{JU}\protect\nobreakdash#1\hspace{0pt}}
\newcommand{\mJU}{\prefixm{JU}}
\newcommand{\rejju}[1]{\rej[\mJU]{#1}}
\newcommand{\modju}[1]{\sqbbr{#1}_{\mJU}}

\newcommand{\AS}[1]{\prefix{UA}\protect\nobreakdash#1\hspace{0pt}}
\newcommand{\mAS}{\prefixm{UA}}
\newcommand{\rejas}[1]{\rej[\mAS]{#1}}
\newcommand{\modas}[1]{\sqbbr{#1}_{\mAS}}

\newcommand{\ctau}{\tau}

\newcommand{\uopr}{\mathbin{\oplus}}
\newcommand{\biguopr}{\mathop{\bigoplus}}

\newcommand{\uoprev}{\star}
\newcommand{\cn}{\mathit{Cn}}

\newcommand{\tri}{X}

\newcommand{\tris}{\mathscr{X}}
\newcommand{\twiab}{\tpl{\twia, \twib}}

\newcommand{\stri}{\mathcal{M}}
\newcommand{\stria}{\mathcal{M}}
\newcommand{\strib}{\mathcal{N}}

\newcommand{\sstri}{\mathcal{S}}
\newcommand{\sstria}{\mathcal{S}}
\newcommand{\sstrib}{\mathcal{T}}

\newcommand{\tr}{\mathsf{T}}
\newcommand{\fa}{\mathsf{F}}
\newcommand{\un}{\mathsf{U}}
\newcommand{\val}{\mathsf{V}}

\newcommand{\SE}[1]{\prefix{SE}\protect\nobreakdash#1\hspace{0pt}}

\newcommand{\mSE}{\prefixm{SE}}
\newcommand{\eqSE}{\equiv_{\mSE}}
\newcommand{\entSE}{\ent_{\mSE}}
\newcommand{\modse}[1]{\sqbbr{#1}_{\mSE}}

\newcommand{\smallpuse}[1]{\hyperref[pstl:puse:#1]{\pstl[\scriptsize]{(P#1)$_{\prefix{\tiny SE}}$}}}

\newcommand{\canse}[1]{\mathsf{can}_{\mSE}(#1)}
\newcommand{\canre}[1]{\mathsf{can}_{\mRE}(#1)}

\newcommand{\hstrire}{H_{\mRE}(\stri)}
\newcommand{\bstrire}{B_{\mRE}(\stri)}

\newcommand{\rulere}[1]{\| #1 \|_{\mRE}}

\newcommand{\gprg}{\Pi}
\newcommand{\gprga}{\Pi}
\newcommand{\gprgb}{\Sigma}

\newcommand{\modr}[1]{\anbbr{#1}}
\newcommand{\Modr}[1]{\Anbbr{#1}}
\newcommand{\modrr}[1]{\bbr{#1}}

\newcommand{\X}[2][X]{\prefix{#1}\protect\nobreakdash#2\hspace{0pt}}
\newcommand{\mX}{\prefixm{X}}
\newcommand{\eqX}{\equiv_{\mX}}
\newcommand{\entX}{\ent_{\mX}}

\newcommand{\modxr}[1]{\anbbr{#1}_{\mX}}

\newcommand{\Y}[2][Y]{\prefix{#1}\protect\nobreakdash#2\hspace{0pt}}
\newcommand{\mY}{\prefixm{Y}}
\newcommand{\eqY}{\equiv_{\mY}}
\newcommand{\entY}{\ent_{\mY}}
\newcommand{\mody}[1]{\sqbbr{#1}_{\mY}}
\newcommand{\modyr}[1]{\anbbr{#1}_{\mY}}

\newcommand{\modc}[1]{\sqbbr{#1}}

\newcommand{\SM}[1]{\prefix{SM}\protect\nobreakdash#1\hspace{0pt}}
\newcommand{\mSM}{\prefixm{SM}}
\newcommand{\eqSM}{\equiv_{\mSM}}

\newcommand{\SR}[1]{\prefix{SR}\protect\nobreakdash#1\hspace{0pt}}
\newcommand{\mSR}{\prefixm{SR}}
\newcommand{\eqSR}{\equiv_{\mSR}}
\newcommand{\entSR}{\ent_{\mSR}}
\newcommand{\modser}[1]{\anbbr{#1}_{\mSE}}

\newcommand{\SMR}[1]{\prefix{SMR}\protect\nobreakdash#1\hspace{0pt}}
\newcommand{\mSMR}{\prefixm{SMR}}
\newcommand{\eqSMR}{\equiv_{\mSMR}}
\newcommand{\entSMR}{\ent_{\mSMR}}

\newcommand{\SU}[1]{\prefix{SU}\protect\nobreakdash#1\hspace{0pt}}
\newcommand{\mSU}{\prefixm{SU}}
\newcommand{\eqSU}{\equiv_{\mSU}}
\newcommand{\entSU}{\ent_{\mSU}}

\newcommand{\RE}[1]{\prefix{RE}\protect\nobreakdash#1\hspace{0pt}}

\newcommand{\sRE}{\prefix{RE}}
\newcommand{\mRE}{\prefixm{RE}}
\newcommand{\eqRE}{\equiv_{\mRE}}
\newcommand{\entRE}{\ent_{\mRE}}
\newcommand{\modre}[1]{\sqbbr{#1}_{\mRE}}
\newcommand{\modrer}[1]{\anbbr{#1}_{\mRE}}
\newcommand{\modrerscr}[1]{\anbbrscr{#1}_{\mRE}}

\newcommand{\RR}[1]{\prefix{RR}\protect\nobreakdash#1\hspace{0pt}}
\newcommand{\mRR}{\prefixm{RR}}
\newcommand{\eqRR}{\equiv_{\mRR}}
\newcommand{\entRR}{\ent_{\mRR}}

\newcommand{\RMR}[1]{\prefix{RMR}\protect\nobreakdash#1\hspace{0pt}}
\newcommand{\mRMR}{\prefixm{RMR}}
\newcommand{\eqRMR}{\equiv_{\mRMR}}
\newcommand{\entRMR}{\ent_{\mRMR}}

\newcommand{\strle}{\,\preceq\,}
\newcommand{\nstrle}{\,\npreceq\,}
\newcommand{\strl}{\,\prec\,}

\newcommand{\rb}{\mathcal{R}}
\newcommand{\rba}{\mathcal{R}}
\newcommand{\rbb}{\mathcal{S}}

\newcommand{\rbu}{\mathcal{U}}
\newcommand{\rbv}{\mathcal{V}}

\newcommand{\drb}{\boldsymbol{R}}

\newcommand{\labfu}[1]{\pstl[\small]{(F#1)}}
\newcommand{\fu}[1]{\ref{pstl:fu:#1}}

\newcommand{\gfrm}{\alpha}

\newcommand{\gfrms}{\Omega}

\newcommand{\gkb}{\mathcal{K}}
\newcommand{\gkba}{\mathcal{K}}
\newcommand{\gkbu}{\mathcal{U}}

\newcommand{\uope}{\oplus}
\newcommand{\biguope}{\mathop{\textstyle \bigoplus}}

\newcommand{\e}{\varepsilon}
\newcommand{\te}[1]{\ensuremath{\e}\protect\nobreakdash#1\hspace{0pt}}

\newcommand{\er}{\delta}
\newcommand{\ter}[1]{\ensuremath{\er}\protect\nobreakdash#1\hspace{0pt}}

\newcommand{\erone}{\er_\mathsf{a}}
\newcommand{\terone}[1]{$\delta_\mathsf{a}$\protect\nobreakdash#1\hspace{0pt}}
\newcommand{\ertwo}{\er_\mathsf{b}}
\newcommand{\tertwo}[1]{$\delta_\mathsf{b}$\protect\nobreakdash#1\hspace{0pt}}
\newcommand{\erthree}{\er_\mathsf{c}}
\newcommand{\terthree}[1]{$\delta_\mathsf{c}$\protect\nobreakdash#1\hspace{0pt}}
\newcommand{\erfour}{\er_\mathsf{d}}
\newcommand{\terfour}[1]{$\delta_\mathsf{d}$\protect\nobreakdash#1\hspace{0pt}}
\newcommand{\erfive}{\er_\mathsf{e}}
\newcommand{\terfive}[1]{$\delta_\mathsf{e}$\protect\nobreakdash#1\hspace{0pt}}

\newcommand{\labpu}[1]{\pstl[\small]{(P#1)}}
\newcommand{\labputwotop}{\pstl[\small]{(P2.$\top$)}}
\newcommand{\labpup}[1]{\pstl[\small]{(#1)}}
\newcommand{\pu}[1]{\hyperref[tbl:semantic properties]{\labpu{#1}}}
\newcommand{\putwotop}{\hyperref[tbl:semantic properties]{\labputwotop}}
\newcommand{\pua}[1]{\hyperref[pstl:pua]{\labpu{#1}}}
\newcommand{\pup}[1]{\hyperref[tbl:semantic properties]{\labpup{#1}}}

\newcommand{\after}[2]{\mathsf{after}_{#1}(#2)}
\newcommand{\aug}[2][\er]{\mathsf{aug}_{#1}(#2)}
\newcommand{\Aug}[2][\er]{\mathsf{aug}_{#1}\left(#2\right)}
\newcommand{\augone}[1]{\aug[\erone]{#1}}
\newcommand{\Augone}[1]{\Aug[\erone]{#1}}
\newcommand{\augtwo}[1]{\aug[\ertwo]{#1}}
\newcommand{\Augtwo}[1]{\Aug[\ertwo]{#1}}

\newcommand{\uoprtwo}{\uopr_{\mathsf{b}}}
\newcommand{\uoprthree}{\uopr_{\mathsf{c}}}
\newcommand{\uoprfour}{\uopr_{\mathsf{d}}}
\newcommand{\uoprfive}{\uopr_{\mathsf{e}}}

\newcommand{\biguoprtwo}{\mathop{\textstyle \biguopr_{\mathsf{b}}}}
\newcommand{\biguoprthree}{\mathop{\textstyle \biguopr_{\mathsf{c}}}}
\newcommand{\biguoprfour}{\mathop{\textstyle \biguopr_{\mathsf{d}}}}
\newcommand{\biguoprfive}{\mathop{\textstyle \biguopr_{\mathsf{e}}}}

\newcommand{\sems}{\mathcal{Z}}

\newcommand{\tones}[1]{#1^{\circledcirc}}

\newcommand{\acond}[2]{\alpha_{#1}(#2)}
\newcommand{\Acond}[2]{\alpha_{#1}\!\!\left(#2\right)}

\newcommand{\uoprju}{\uopr_{\mJU}}
\newcommand{\biguoprju}{\textstyle\biguopr_{\mJU}}

\newcommand{\uopras}{\uopr_{\mAS}}
\newcommand{\biguopras}{\textstyle\biguopr_{\mAS}}

\newcommand{\sblk}{S}
\newcommand{\sblka}{S}
\newcommand{\sblkb}{T}
\newcommand{\sblks}[2]{\beta_{#1}(#2)}
\newcommand{\Sblks}[2]{\beta_{#1}\!\left(#2\right)}

\newcommand{\uoprjulor}{\uoprju^{\raisebox{0pt}{\tiny $\lor$}}}
\newcommand{\biguoprjulor}{
	\textstyle\biguoprju^{\raisebox{0pt}{\tiny $\lor$}}
}

\newcommand{\uopraslor}{\uopras^{\raisebox{0pt}{\tiny $\lor$}}}
\newcommand{\biguopraslor}{\textstyle\biguopras^{\raisebox{0pt}{\tiny $\lor$}}}

\begin{document}

\title{Exception-Based Knowledge Updates\tnoteref{t1}}

\tnotetext[t1]{%
	This is a revised and extended version of the material presented in
	\cite{Slota2012,Slota2012a,Slota2013}.
}

\author{Martin Slota}
\ead{martin.slota@gmail.com}

\author{Jo\~{a}o Leite\corref{cor1}}
\ead{jleite@fct.unl.pt}

\cortext[cor1]{%
	Corresponding author.
}

\address{%
	NOVA LINCS \& Departamento de Inform{\'a}tica,
	Universidade Nova de Lisboa, \\
	Quinta da Torre,
	2829-516 Caparica,
	Portugal
}

\date{}

\pdfinfo{%
/Title (Exception-Based Knowledge Updates)
/Keywords (knowledge representation, answer set programming, belief update,
rule update, SE-models, RE-models, program equivalence)
/Author (Martin Slota and Joao Leite)}

\begin{abstract}
	Existing methods for dealing with \emph{knowledge updates} differ greatly
	depending on the underlying knowledge representation formalism. When
	Classical Logic is used, updates are typically performed by manipulating the
	knowledge base on the model-theoretic level. On the opposite side of the
	spectrum stand the semantics for updating Answer-Set Programs that need to
	rely on rule syntax. Yet, a \emph{unifying perspective} that could embrace
	both these branches of research is of great importance as it enables a
	deeper understanding of all involved methods and principles and creates room
	for their cross-fertilisation, ripening and further development.
	Furthermore, from a more pragmatic viewpoint, such a unification is a
	necessary step in addressing \emph{updates of hybrid knowledge bases}
	consisting of both a classical and a rule component.

	This paper bridges the seemingly irreconcilable approaches to updates. It
	introduces a novel monotonic characterisation of rules, dubbed
	\emph{\RE-models}, and shows it to be a more suitable semantic foundation
	for rule updates than \SE-models. Then it proposes a generic scheme for
	specifying semantic rule update operators, based on the idea of viewing a
	program as the \emph{set of sets of \RE-models} of its rules; updates are
	performed by introducing additional interpretations -- \emph{exceptions} --
	to the sets of \RE-models of rules in the original program. The introduced
	scheme is then used to define particular rule update operators that are
	closely related to both classical update principles and traditional
	approaches to rules updates, enjoying a range of plausible syntactic as well
	as semantic properties. In addition, these operators serve as a basis for a
	solution to the long-standing problem of \emph{state condensing} for two of
	the foundational rule update semantics, showing how they can be equivalently
	defined as binary operators on some class of logic programs.

	Finally, the essence of these ideas is extracted to define an abstract
	framework for exception-based update operators, viewing a knowledge base as
	the \emph{set of sets of models} of its elements. It is shown that the
	framework can capture a wide range of both model- and formula-based
	classical update operators, and thus serves as the first firm formal ground
	connecting classical and rule updates.
\end{abstract}

\maketitle

\section{Introduction}

Recent standardisation efforts gave rise to widely accepted knowledge
representation languages such as the Web Ontology Language
(OWL)\footnote{\url{http://www.w3.org/TR/owl-overview/}} and Rule Interchange
Format
(RIF),\footnote{\url{http://www.w3.org/2005/rules/wiki/RIF\_Working\_Group}}
based on Description Logics \cite{Baader2007} and Logic Programming
\cite{Colmerauer1973,Kowalski1974,Lloyd1987,Gelfond1988}, respectively. This
has fostered a large number of ontologies and rule bases with different levels
of complexity and scale. Whereas ontologies provide the logical underpinning
of intelligent access and information integration, rules are widely used to
represent business policies, regulations and declarative guidelines about
information. 

Since both ontologies and rules offer important features for knowledge
representation, considerable effort has been invested in identifying a unified
hybrid knowledge framework where expressivity of both formalisms could be
seamlessly combined. Over the years, work on hybrid knowledge bases has
matured significantly and fundamental semantic as well as computational
problems were addressed successfully
\cite{Hitzler2009,Bruijn2010,Motik2010,Bruijn2011,Knorr2011}. While such
formalisms make it possible to seamlessly combine rules and ontologies in a
single unified framework, they do not take into account the \emph{dynamic
character} of application areas where they are to be used. More particularly,
the essential support for keeping a hybrid knowledge base \emph{up to date},
by incorporating new and possibly conflicting information, is still missing.
Nonetheless, this topic has been extensively addressed in the context of both
Description Logics and Logic Programs, when taken separately.

\subsection*{Ontology Updates}

The area of research called \emph{ontology change} encompasses a number of
strongly related though distinguishable subareas, such as ontology matching,
ontology integration and merging, or ontology translation \cite{Flouris2008}.
The purest type of change, concerned with modifications to a single ontology,
is generally referred to as \emph{ontology evolution}. Approaches to ontology
evolution with a firm semantic underpinning, thus amenable to a formal
analysis of their behaviour and properties, are based on research in the area
of \emph{belief change}, initiated by the seminal work of Alchourr\'{o}n,
G\"{a}rdenfors and Makinson (AGM) \cite{Alchourron1985} who proposed a set of
desirable properties of change operators on monotonic logics, now referred to
as the \emph{AGM postulates}.

Subsequently, \emph{revision} and \emph{update} were distinguished as two very
related but ultimately different belief change operations
\cite{Keller1985,Winslett1990,Katsuno1991}. While revision deals with
incorporating new information about a \emph{static world} into a knowledge
base, update takes place when a knowledge base needs to be brought up to date
when the \emph{modelled world changes}. While AGM postulates were deemed
appropriate for describing revision, Katsuno and Mendelzon suggested a
different set of postulates for updates: the \emph{KM postulates for belief
update} \cite{Katsuno1991}.

Update operators respecting the KM postulates, usually referred to as
\emph{model-based}, are based on the idea that the models of a knowledge base
correspond to possible states of the represented world. When a change in the
world needs to be recorded, inertia is applied to each of these possible
states, making only the smallest necessary modifications to reflect the
change, and arriving at a new collection of possible states that represent the
world after the update. Since the updates are specified on the semantic level,
they are naturally \emph{syntax-independent}. These ideas, and particularly
Winslett's update operator \cite{Keller1985,Winslett1988}, were later used to
partially address ontology updates
\cite{Liu2006,Bong2007,Drescher2009,Giacomo2006,Giacomo2007,Giacomo2009,Kharlamov2013},
namely to update the part of the ontology with assertions about individuals
(the \emph{ABox}).

On the other hand, model-based operators are considered inappropriate for
updating ontological axioms that define the terminology (the \emph{TBox})
\cite{Calvanese2010,Slota2010a}. Their antipole, \emph{formula-based
operators}, which manipulate the knowledge base at a syntactic level and are
strongly related to \emph{base revision operators}, were adopted for
performing TBox updates instead \cite{Calvanese2010}, and they also inspired a
recent approach to ABox updates \cite{Lenzerini2011}.

\subsection*{Rule Updates}

When updates started to be investigated in the context of Logic Programming,
it was only natural to adapt belief update principles and operators to this
purpose \cite{Alferes1996,Marek1998}. However, such approaches proved
insufficiently expressive, principally because the model-based approach fails
to capture the essential relationships between literals encoded in rules
\cite{Leite1997}, and the formula-based approach is too crude as it does not
allow rules to be reactivated when reasons for their suppression disappear
\cite{Zhang2006}. Although state-of-the-art approaches to rule updates are
guided by the same basic intuitions and aspirations as belief updates, they
build upon fundamentally different principles and methods.

Many of them are based on the \emph{causal rejection principle}
\cite{Leite1997,Alferes2000,Eiter2002,Alferes2005,Osorio2007}, which states
that a rule is \emph{rejected} only if it is directly contradicted by a more
recent rule. This essentially means that inertia and minimal change is applied
to rules instead of to the state, as with model-based belief updates. Causal
rejection semantics are useful in a number of practical scenarios
\cite{Alferes2003,Saias2004,Siska2006,Ilic2008} and their behaviour is
intuitively predictable. Alternative approaches to rule updates employ
syntactic transformations and other methods, such as abduction
\cite{Sakama2003}, prioritisation and preferences
\cite{Zhang2006,Delgrande2007}, or dependencies on default assumptions
\cite{Sefranek2011,Krumpelmann2010}.

Despite the variety of techniques used in these approaches, certain properties
are common to all of them. First, the stable models assigned to a program
after one or more updates are always \emph{supported}: for each true atom
$\atm$ there exists a rule in either the original program or its updates that
has $\atm$ in the head and whose body is satisfied. Second, all mentioned rule
update semantics coincide when it comes to \emph{updating sets of facts} by
newer facts. We conjecture that any reasonable rule update semantics should
indeed be in line with the basic intuitions regarding \emph{support} and
\emph{fact update}. Another common characteristic of all these approaches is
that they need to refer to the \emph{syntactic structure} of a logic program:
the individual rules and, in most cases, also the literals in their heads and
bodies. This renders them seemingly irreconcilable with ontology updates since
ontology axioms simply have no heads and bodies.

\subsection*{Towards Hybrid Updates}

A unifying framework that could embrace both belief and rule updates is of
great importance as it enables a deeper understanding of all involved methods
and principles and creates room for their cross-fertilisation, ripening and
further development. It is also important for the development of update
semantics for hybrid knowledge bases -- in \cite{Slota2010a,Slota2011a} we
provided partial solutions to this problem but the inherent differences
between the distinct approaches to updates have prevented us from suggesting a
universal hybrid update semantics.

Moreover, we argue that syntax-independence, central to model-based belief
updates, is essential and should be pursued at large in order to encourage a
logical underpinning of all update operators and so facilitate analysis of
their semantic properties. When equivalence with respect to classical models
is inappropriate, as is the case with rules, syntax-independence should be
retained by finding an appropriate notion of equivalence, specific to the
underlying formalism and its use.

With these standpoints in mind, we proceed with our previous work addressing
the logical foundations of rule updates. In \cite{Slota2013a} we have shown
that \emph{strong equivalence} is not a suitable basis for syntax-independent
rule update operators because such operators cannot respect both support and
fact update. This can be demonstrated on programs $\prga = \set{\atma.,
\atmb.}$ and $\prgb = \set{\atma., \atmb \lpif \atma.}$ which are strongly
equivalent, so, due to syntax independence, an update asserting that $\atma$
is now false ought to lead in both cases to the same stable models. Due to
fact update, such an update on $\prga$ should naturally lead to a stable model
where $\atmb$ is true. But in case of $\prgb$ such a stable model would be
unsupported.

This led us to the study of stronger notions of program equivalence. In
\cite{Slota2011} we proposed to view a program as the \emph{set of sets of
models of its rules} in order to acknowledge rules as the atomic pieces of
knowledge and, at the same time, abstract away from unimportant differences
between their syntactic forms, focusing on their semantic content. In this
paper we develop these ideas further and arrive at a unifying perspective on
both classical and rule updates. More particularly, our main contributions are
as follows:

\begin{itemize}
	\item We introduce a novel monotonic characterisation of rules,
		\emph{\RE-models}, and show that they form a more suitable semantic
		foundation for rule updates than \SE-models;

	\item We propose a generic scheme for defining semantic rule update
		operators: a program, viewed as the \emph{set of sets of \RE-models} of
		its rules, is updated by introducing additional interpretations to those
		sets of \RE-models;

	\item We identify instances of the framework that bridge classical update
		principles with traditional rule update semantics: they combine
		\emph{syntax-independence} with \emph{support} and \emph{fact update} and
		have other desirable \emph{syntactic} as well as \emph{semantic}
		properties.

	\item We solve the enduring problem of \emph{state condensing} for two
		foundational rule update semantics by showing how they can be equivalently
		defined as binary operators on some class of logic programs;

	\item We define abstract exception-based operators for any knowledge
		representation formalism with a model-theoretic semantics;

	\item We show that exception-based operators capture a wide range of model-
		and formula-based belief update operators.

\end{itemize}

This paper is organised as follows: We introduce the necessary theoretical
background in Section~\ref{sec:background} and in Section~\ref{sec:robust
equivalence models} we define \RE-models and associated notions of
equivalence. Section~\ref{sec:exception-based-updates:rule updates} introduces
the generic scheme for specifying exception-based rule update operators as
well as its particular instances and analyses their theoretical properties.
Section~\ref{sec:state condensing} is devoted to the problem of state
condensing and introduces operators for condensing an update sequence into a
single program for foundational rule update semantics.
Sections~\ref{sec:exception-based-updates:abstract} and
\ref{sec:exception-based-updates:belief updates} introduce abstract
exception-based operators and show how they are able to characterise belief
updates. We discuss our findings in a broader context and conclude in
Section~\ref{sec:discussion}. Proofs of all formal results can be found in the
appendices.

\section{Background}

\label{sec:background}

\subsection{Propositional Logic}

%
We consider a propositional language over a finite set of propositional
variables $\atms$ and the usual set of propositional connectives to form
propositional formulas. A \emph{knowledge base} is a finite set of formulas.
A (two-valued) interpretation is any $\twi \subseteq \atms$ and the set of all
interpretations is denoted by $\twis$. The set of models of a knowledge base
$\kb$ is defined in the standard way and denoted by $\mod{\kb}$. Given an
interpretation $\twi$, we sometimes write $\twi \ent \kb$ or $\twi(\kb) = \tr$
if $\twi$ is a model of $\kb$, and $\twi \nent \kb$ or $\twi(\kb) = \fa$ if
$\twi$ is not a model of $\kb$. A knowledge base $\kb$ is \emph{consistent} if
$\mod{\kb}$ is non-empty; \emph{complete} if $\mod{\kb}$ is a singleton set.
Given two knowledge bases $\kba$, $\kbb$, we say that \emph{$\kba$ entails
$\kbb$}, denoted by $\kba \ent \kbb$, if $\mod{\kba} \subseteq \mod{\kbb}$,
and that \emph{$\kba$ is equivalent to $\kbb$}, denoted by $\kba \eq \kbb$, if
$\mod{\kba} = \mod{\kbb}$. The models, set of models, consistency,
completeness, entailment and equivalence are generalised to formulas by
treating every formula $\frm$ as the knowledge base $\set{\frm}$.

\subsection{Logic Programs}

%
The basic syntactic building blocks of rules are also propositional atoms from
$\atms$. A \emph{default literal} is an atom preceded by $\lpnot{}$ denoting
default negation. The set of all \emph{literals} $\lits$ consists of all
atoms and default literals. The \emph{complementary literal} to a literal
$\lit$ is denoted by $\lcmp{\lit}$ and defined as follows: for any atom
$\atm$, $\lcmp{\atm} = \lpnot \atm$ and $\lcmp{\lpnot \atm} = \atm$. Given a
set of literals $\slit$, we introduce the following notation: $\slit^+ =
\Set{\atm \in \atms | \atm \in \slit}$, $\slit^- = \Set{\atm \in \atms |
\lpnot \atm \in \slit}$, $\lcmp{\slit} = \Set{\lcmp{\lit} | \lit \in \slit}$.

A \emph{rule} is a pair of sets of literals $\rl = \tpl{\hrl, \brl}$. We say
that $\hrl$ is the \emph{head of $\rl$} and $\brl$ is the \emph{body of
$\rl$}. Usually, for convenience, we write $\rl$ as
$
	\Br{\hrl^+; \lcmp{\hrl^-} \lpif \brl^+, \lcmp{\brl^-}.}
$.
A rule is called \emph{non-disjunctive} if its head contains at most one
literal; \emph{a fact} if its head contains exactly one literal and its body
is empty. A \emph{program} is any set of rules. A program is
\emph{non-disjunctive} if all its rules are; \emph{acyclic} if it satisfies
the conditions set out in \cite{Apt1991}. We also introduce the following
non-standard notion that is needed throughout the rest of the paper:

\begin{definition}
	[Canonical Tautology]
	\label{def:canonical tautology}
	Let $\atm_\ctau \in \atms$ be fixed. The \emph{canonical tautology} $\ctau$
	is the rule $(\atm_\ctau \lpif \atm_\ctau.)$.
\end{definition}

Turning to the semantics, we need to define \emph{stable models}
\cite{Gelfond1988,Gelfond1991} and \emph{\SE-models}
\cite{Turner2003,Pearce1997} of a logic program.  We start by generalizing the
set of \emph{(classical) models} to literals, rules and programs. Given an
atom $\atm$, a set of literals $\slit$, a rule $\rl$ and a program $\prg$, we
define the following:
\begin{align*}
	\modc{\lpnot \atm}
	&=
	\twis \setminus \modc{\atm}
	\enspace,
	& \modc{\slit}
	&=
	\textstyle\bigcap_{\lit \in \slit} \modc{\lit}
	\enspace,
	\\
	\modc{\rl}
	&=
	\set{
		\twib \in \twis
		|
		\twib \notin \modc{\brl}
		\lor
		\exists \lit \in \hrl : \twib \in \modc{\lit}
	}
	\enspace,
	& \modc{\prg}
	&=
	\textstyle\bigcap_{\rl \in \prg} \modc{\rl}
	\enspace.
\end{align*}
For a set of literals $\slit$, we also write $\twib \ent \slit$ whenever
$\twib \in \modc{\slit}$. A program $\prg$ is \emph{consistent} if
$\modc{\prg} \neq \emptyset$.

The stable and \SE-models are defined in terms of reducts. Given a rule $\rl$
and an interpretation $\twib$, the \emph{reduct of $\rl$ w.r.t.\ $\twib$},
denoted by $\rl^\twib$, is the rule $(\hrl^+ \lpif \brl^+.)$ if $\twib \ent
\lcmp{\brl^-}$ and $\twib \ent \hrl^-$;
otherwise it is the canonical tautology $\ctau$.\footnote{%
	In other words, if $\twib$ satisfies all default literals in the body of the
	rule and none of the default literals in its head, then the reduct is
	obtained from the original rule by removing all default literals from its
	head and body. In other cases, $\twib$ is certainly a model of the rule, so
	the reduct is defined as the canonical tautology.
}
The \emph{reduct of a program $\prg$ w.r.t.\ $\twib$} is defined as
$\prg^\twib = \Set{\rl^\twib | \rl \in \prg}$.

An interpretation $\twib$ is a \emph{stable model} of a program $\prg$ if
$\twib$ is a subset-minimal model of $\prg^\twib$. The set of all stable
models of $\prg$ is denoted by $\modsm{\prg}$.

\SE-models are semantic structures that can be seen as three-valued
interpretations. In particular, we call a pair of interpretations $\tri =
\tpl{\twia, \twib}$ such that $\twia \subseteq \twib$ a \emph{three-valued
interpretation}. Each atom $\atm$ is assigned one of three truth values in
$\tri$: $\tri(\atm) = \tr$ if $\atm \in \twia$; $\tri(\atm) = \un$ if $\atm
\in \twib \setminus \twia$; $\tri(\atm) = \fa$ if $\atm \in \atms \setminus
\twib$. The set of all three-valued interpretations is denoted by $\tris$. A
three-valued interpretation $\tpl{\twia, \twib}$ is an \emph{\SE-model} of a
rule $\rl$ if $\twib$ is a model of $\rl$ and $\twia$ is a model of
$\rl^\twib$. The set of all \SE-models of a rule $\rl$ is denoted by
$\modse{\rl}$ and for any program $\prg$, $\modse{\prg} = \bigcap_{\rl \in
\prg} \modse{\rl}$. Note that $\twib$ is a stable model of $\prg$ if and only
if $\tpl{\twib, \twib} \in \modse{\prg}$ and for all $\twia \subsetneq \twib$,
$\tpl{\twia, \twib} \notin \modse{\prg}$. Also, $\twib \in \modc{\prg}$ if and
only if $\tpl{\twib, \twib} \in \modse{\prg}$. We say that a rule $\rl$ is
\textbf{(\SE-)tautological} if $\modse{\rl} = \tris$. Note that the canonical
tautology (c.f.\ Definition~\ref{def:canonical tautology}) is tautological.

\subsection{Belief Updates}

%
In this section we briefly introduce \emph{model-based} as well as
\emph{formula-based} belief update operators which form the basis of formal
approaches to ontology updates
\cite{Liu2006,Bong2007,Drescher2009,Giacomo2006,Giacomo2007,Giacomo2009,Calvanese2010,Lenzerini2011,Kharlamov2013}.

We liberally define a belief update operator as any function that takes the
original knowledge base and its update as inputs and returns the updated
knowledge base.

\begin{definition}
	[Update Operator]
	\label{def:bu:belief update operator}
	A \emph{belief update operator} is a binary function over the set of all
	knowledge bases. Any belief update operator $\uopb$ is inductively
	generalised to finite sequences of knowledge bases $\seq{\kb_\lia}_{\lia <
	\lng}$ as follows: $\biguopb \seq{} = \emptyset$ and $\biguopb
	\seq{\kb_\lia}_{\lia < \lng + 1} = (\biguopb \seq{\kb_\lia}_{\lia < \lng})
	\uopb \kb_\lng$.
\end{definition}

The fundamental idea underlying model-based update operators
\cite{Keller1985,Winslett1988,Winslett1990,Katsuno1991} is that models of the
original knowledge base are viewed as \emph{possible} or \emph{admissible
states} of the modelled domain, only one of which is the actual state. Given
this perspective, it is natural to perform an update with $\kbu$ by modifying
each of the possible states independently of the others, making it consistent
with $\kbu$, and thus obtaining a new set of interpretations -- the models of
the updated knowledge base. Formally this is captured by the equation
\begin{equation}
	\label{eq:belupd:operators}
	\mod{\kba \uopb \kbu}
	=
	\bigcup_{\twi \in \mod{\kba}}
	\incorp(\mod{\kbu}, \twi)
	\enspace,
\end{equation}
where $\incorp(\stwi, \twi)$ returns the members of $\stwi$ closer to $\twi$
so that the original information in $\twi$ is preserved as much as possible.
For instance, Winslett's operator, first introduced to deal with change in
action theories and relational databases with incomplete information
\cite{Keller1985,Winslett1988}, and used extensively to study ontology updates
\cite{Liu2006,Bong2007,Drescher2009,Giacomo2006,Giacomo2007,Giacomo2009,Kharlamov2013},
is characterised by equation \eqref{eq:belupd:operators} where $\incorp(\stwi,
\twi)$ is the set of models from $\stwi$ that interpret a minimal set of atoms
differently than $\twi$. From a more general perspective, Katsuno and
Mendelzon have shown \cite{Katsuno1991} that if each knowledge base is
represented by a single formula and $\incorp(\stwi, \twi)$ is bound to return
those members of $\stwi$ that are minimal w.r.t.\ a closeness relation
$\leq_\twi$ assigned to $\twi$,\footnote{%
	Technically, the closeness relation is simply a partial order on $\twis$
	such that $\twi$ is the least interpretation w.r.t.\ $\leq_\twi$.
}
then the class of operators which satisfy \eqref{eq:belupd:operators}
coincides with the class of operators which satisfy the declarative postulates
\bu{1} -- \bu{8} listed in Table~\ref{tbl:belief update postulates}. Note that
\bu{1} -- \bu{6} can be immediately generalised to deal with arbitrary
knowledge bases but \bu{7} and \bu{8} require disjunction of knowledge bases
to be defined.

\begin{table}[t]
	\centering
	\caption{Belief update postulates}
		\label{tbl:belief update postulates}
		\small
		\begin{tabular}{cm{20em}m{19.6em}}
			\toprule
			Postulate
				& Knowledge base as a single formula
				& Knowledge base as a set
				\\ \midrule[\heavyrulewidth]
			\labbu{1}
				& $\frma \uopb \frmu \ent \frmu$.
				& $\kba \uopb \kbu \ent \kbu$.
				\\ \midrule
			\labbu{2}
				& If $\frma \ent \frmu$, then $\frma \uopb \frmu \eq \frma$.
				& If $\kba \ent \kbu$, then $\kba \uopb \kbu \eq \kba$.
				\\ \midrule
			\labbutwotop
				& $\frma \uopb \top \eq \frma$.
				& $\kba \uopb \emptyset \eq \kba$.
				\\ \midrule
			\labbu{2.1}
				& $\frma \land \frmu \ent \frma \uopb \frmu$.
				& $\kba \land \kbu \ent \kba \uopb \kbu$.
				\\ \midrule
			\labbu{2.2}
				& $(\frma \land \frmu) \uopb \frmu \ent \frma$.
				& $(\kba \land \kbu) \uopb \kbu \ent \kba$.
				\\ \midrule
			\labbu{3}
				& If $\mod{\frma} \neq \emptyset$ and $\mod{\frmu} \neq \emptyset$,
					then $\mod{\frma \uopb \frmu} \neq \emptyset$.
				& If $\mod{\kba} \neq \emptyset$ and $\mod{\kbu} \neq \emptyset$,
					then $\mod{\kba \uopb \kbu} \neq \emptyset$.
				\\ \midrule
			\labbu{4}
				& If $\frma \eq \frmb$ and $\frmu \eq \frmv$, then $\frma \uopb \frmu
					\eq \frmb \uopb \frmv$.
				& If $\kba \eq \kbb$ and $\kbu \eq \kbv$, then $\kba \uopb \kbu
					\eq \kbb \uopb \kbv$.
				\\ \midrule
			\labbu{5}
				& $(\frma \uopb \frmu) \land \frmv \ent \frma \uopb (\frmu \land
					\frmv)$.
				& $(\kba \uopb \kbu) \cup \kbv \ent \kba \uopb (\kbu \cup \kbv)$.
				\\ \midrule
			\labbu{6}
				& If $\frma \uopb \frmu \ent \frmv$ and $\frma \uopb \frmv \ent
					\frmu$, then $\frma \uopb \frmu \eq \frma \uopb \frmv$.
				& If $\kba \uopb \kbu \ent \kbv$ and $\kba \uopb \kbv \ent
					\kbu$, then $\kba \uopb \kbu \eq \kba \uopb \kbv$.
				\\ \midrule
			\labbu{7}
				& If $\frma$ is complete, then $(\frma \uopb \frmu) \land (\frma \uopb
					\frmv) \ent \frma \uopb (\frmu \lor \frmv)$.
				& 
				\\ \midrule
			\labbu{8}
				& $(\frma \lor \frmb) \uopb \frmu \eq (\frma \uopb \frmu) \lor (\frmb
					\uopb \frmu)$.
				& 
				\\ \bottomrule
		\end{tabular}
\end{table}

Some of these postulates, and especially \bu{4} which guarantees
syntax-independence of an update operator, continue to be seen as fundamental
cornerstones of belief updates \cite{Herzig1999}. The desirability of others
has been questioned by many
\cite{Brewka1993,Boutilier1995,Doherty1998,Herzig1999}, and particularly
\bu{2}, \bu{5}, \bu{6} and \bu{7} are deemed controversial: the first three
sometimes lead to undesirable behaviour while the last one is hard to explain
intuitively and is satisfied only by a minority of existing update operators
\cite{Herzig1999}. Note also that \bu{2} entails the weaker principles
\butwotop, \bu{2.1} and \bu{2.2}. The first two are uncontroversial as they are
satisfied by all update operators. In addition, in the presence of \bu{4}, the
latter two are together powerful enough to entail \bu{2}, so the controversial
part of \bu{2} is \bu{2.2} \cite{Herzig1999}.

Earlier approaches to updates, dubbed \emph{formula-based}
\citep{Winslett1990}, operate on the syntax of a knowledge base and, as a
consequence, are not syntax-independent. Nevertheless, they have recently been
considered for performing ontology updates \cite{Calvanese2010,Lenzerini2011}.
Traditional formula-based update operators are Set-Of-Theories
\cite{Fagin1983}, WIDTIO (When In Doubt, Throw It Out)
\cite{Ginsberg1986,Ginsberg1988,Winslett1990} and Cross-Product
\cite{Ginsberg1986}. We define only the latter two as the Set-Of-Theories
operator produces a \emph{collection of knowledge bases} as its result instead
of a single knowledge base, and is equivalent to the Cross-Product operator
which compiles these knowledge bases into one.

The central notion in these operators is that of a \emph{possible remainder}
which is a maximal set of formulas from the original knowledge base that is
consistent with the update. Formally, given knowledge bases $\kba$ and $\kbu$,
the set of possible remainders $\remof{\kba, \kbu}$ is the set of maximal
subsets $\kba'$ of $\kba$ such that $\kba' \cup \kbu$ is consistent. The
distinct formula-based operators differ in how they deal with the case when
there is more than one possible remainder. While Cross-Product compiles the
different remainders into a single formula, WIDTIO takes the safer path -- it
keeps exactly those formulas that belong to the intersection of all remainders
and throws away the rest. Additionally, in \cite{Calvanese2010} the new
formula-based operator \emph{Bold} was suggested for performing TBox updates.
The Bold operator solves the problem of multiple remainders by using a
selection function to choose one and commit to it. Formally, a \emph{remainder
selection function} is a function $\sfn$ that assigns to every set of
remainders $\srem$ a remainder $\sfn(\srem) \in \srem$. Given such a selection
function $\sfn$, the Cross-Product operator $\uopcross$, WIDTIO operator
$\uopwidtio$ and Bold operator $\uopbold$ are defined for all knowledge bases
$\kba$, $\kbu$ as follows:
\begin{align*}
	\kba \uopcross \kbu &= \kbu \cup \Set{
		\textstyle \biglor_{\kba' \in \remof{\kba, \kbu}}
		\bigland \kba'
	} ,
	&
	\kba \uopwidtio \kbu &= \kbu \cup \bigcap \remof{\kba,
	\kbu} ,
	&
	\kba \uopbold \kbu &= \kbu \cup \sfn(\remof{\kba, \kbu}) .
\end{align*}

\subsection{Rule Updates}

%
Rule update semantics assign stable models to pairs or sequences of programs
where each component represents an update of the preceding ones. In the
following, we formalise some of the intuitions behind these semantics and
define two foundational rule update semantics.

We start with the basic concepts. A \emph{dynamic logic program} (DLP) is a
finite sequence of non-disjunctive programs. Given a DLP $\dprg$, we use
$\all{\dprg}$ to denote the set of all rules belonging to the programs in
$\dprg$.\footnote{%
	In order to avoid issues with rules that are repeated in multiple components
	of a DLP, we assume throughout this paper that every rule is uniquely
	identified in all set-theoretic operations. This can be formalised by
	assigning a unique name to each rule and performing operations on names
	instead of the rules themselves. However, for the sake of simplicity, we
	leave the technical realisation to the reader.
}
We say that $\dprg$ is \emph{acyclic} if $\all{\dprg}$ is acyclic. A rule
update semantics \S{} assigns a set of \emph{\S-stable models} $\mods{\dprg}$
to every DLP $\dprg$.

As indicated in the introduction, rule update semantics implicitly follow
certain basic intuitions. Particularly, they produce \emph{supported models}
and their behaviour coincides when it comes to \emph{updating sets of facts}
by newer facts. We formalise these two properties w.r.t.\ rule update
semantics for DLPs, calling them \emph{syntactic} because their formulation
requires that we refer to the syntax of the respective DLP.

In the static setting, support \cite{Apt1988,Dix1995a} is one of the basic
conditions that Logic Programming semantics are intuitively designed to
satisfy. Its generalisation to the dynamic case is straightforward.

\begin{property}
	[Support]
	Let $\prg$ be a program, $\atm$ an atom and $\twib$ an interpretation. We
	say that \emph{$\prg$ supports $\atm$ in $\twib$} if $\atm \in \hrl$ and
	$\twib \ent \brl$ for some rule $\rl \in \prg$.

	A rule update semantics \S{} \emph{respects support} if for every DLP
	$\dprg$ and every \S-stable model $\twib$ of $\dprg$ the following
	condition is satisfied: Every atom $\atm \in \twib$ is supported by
	$\all{\dprg}$ in $\twib$.
\end{property}

Thus, if a rule update semantics \S{} respects support, then there is at least
\emph{some} justification for every atom in an assigned \S-stable model.

The second syntactic property that is generally adhered to is the usual
expectation regarding how facts are to be updated by newer facts. It enforces
a limited notion of state inertia but only for the case when both the initial
program and its updates are consistent sets of facts.

\begin{property}
	[Fact Update]
	A rule update semantics \S{} \emph{respects fact update} if for every
	finite sequence of consistent sets of facts $\dprg = \seq{\prg_\lia}_{\lia <
	\lng}$, the unique \S-stable model of $\dprg$ is the interpretation
	\[
		\Set{
			\atm |
			\exists \lia < \lng : (\atm.) \in \prg_\lia
			\land
			(\forall \lib :
				\lia < \lib < \lng \mlthen (\lpnot \atm.) \notin \prg_\lib
			)
		} \enspace.
	\]
\end{property}

We also introduce two further syntactic properties that are more tightly bound
to approaches based on the causal rejection principle
\cite{Leite1997,Alferes2000,Eiter2002,Alferes2005,Osorio2007}. The first one
states the principle itself, under the assumption that a \emph{conflict}
between rules occurs if and only if the rules have complementary heads.

\begin{property}
	[Causal Rejection]
	A rule update semantics \S{} \emph{respects causal rejection} if for every
	DLP $\dprg = \seq{\prg_\lia}_{\lia < \lng}$, every \S-stable model $\twib$
	of $\dprg$, all $\lia < \lng$ and all rules $\rla \in \prg_\lia$, if $\twib$
	is not a model of $\rla$, then there exists a rule $\rlb \in \prg_\lib$ with
	$\lib > \lia$ such that $\hrlb = \lcmp{\hrla}$ and $\twib \ent \brlb$.
\end{property}

Intuitively, the principle requires that all \emph{rejected} rules, i.e.\
rules that are not satisfied in an \S-stable model $\twib$, must be in
conflict with a more recent rule whose body is satisfied in $\twib$. This rule
then provides a \emph{cause} for the rejection.

The final syntactic property stems from the fact that all rule update
semantics based on causal rejection coincide on acyclic DLPs
\cite{Homola2004,Alferes2005}. Thus, the behaviour of any rule update
semantics on acyclic DLPs can be used as a way to compare it to all these
semantics simultaneously. Before formalising the property, we define two
foundational rule update semantics based on causal rejection: the
\JU-semantics \cite{Leite1997} and the \AS-semantics \cite{Eiter2002}.

\begin{definition}
	[\JU-Semantics \cite{Leite1997} and \AS-Semantics \cite{Eiter2002}]
	\label{def:ru:ju}
	Let $\dprg = \seq{\prg_\lia}_{\lia < \lng}$ be a DLP and $\twib$ an
	interpretation. The sets of rejected rules $\rejju{\dprg, \twib}$ and
	$\rejas{\dprg, \twib}$ are defined as follows:
	\begin{align*}
		\rejju{\dprg, \twib}
		&=
		\Set{
			\rla \in \prg_\lia
			|
			\exists \lib
			\,
			\exists \rlb \in \prg_\lib
			:
			\lia < \lib < \lng
			\land
			\hrlb = \lcmp{\hrla}
			\land
			\twib \ent \brlb
		}
		\enspace,
		\\
		\rejas{\dprg, \twib}
		&=
		\Set{
			\rla \in \prg_\lia
			|
			\exists \lib
			\,
			\exists \rlb \in \prg_\lib \setminus \rejas{\dprg, \twib}
			:
			\lia < \lib < \lng
			\land
			\hrlb = \lcmp{\hrla}
			\land
			\twib \ent \brlb
		}
		\enspace.\footnote{%
			Note that although this definition is recursive, the defined set is
			unique. This is because we assume that every rule is uniquely identified
			and to determine whether a rule from $\prg_\lia$ is rejected, the
			recursion only refers to rejected rules from programs $\prg_\lib$ with
			$\lib$ strictly greater than $\lia$. One can thus first find the
			rejected rules in $\prg_{\lng - 1}$ (always $\emptyset$ by the
			definition), then those in $\prg_{\lng - 2}$ and so on until $\prg_0$.
		}
	\end{align*}
	The set $\modju{\dprg}$ of \emph{\JU-models of a DLP $\dprg$} consists of
	all interpretations $\twib$ such that $\twib$ is a stable model of the
	program
	$
		\all{\dprg} \setminus \rejju{\dprg, \twib}
	$.
	Similarly, the set $\modas{\dprg}$ of \emph{\AS-models of a DLP $\dprg$}
	consists of all interpretations $\twib$ such that $\twib$ is a stable model
	of the program
	$
		\all{\dprg} \setminus \rejas{\dprg, \twib}
	$.
\end{definition}

Under the \JU-semantics, a rule $\rla$ is rejected if a more recent rule
$\rlb$ is in conflict with $\rla$ and the body of $\rlb$ is satisfied in the
stable model candidate $\twib$. The only difference in the \AS-semantics
\cite{Eiter2002} is that rejected rules are prevented from rejecting other
rules. Perhaps surprisingly, this renders the \AS-semantics more sensitive to
tautological updates, which cannot indicate any change in the modelled world,
than the \JU-semantics. For example, the DLP 
\begin{align}
	\label{eq:ru:problem with as}
	\dprg_1 = \seq{\set{\atma.}, \set{\lpnot \atma.}, \set{\atma \lpif \atma.}}
\end{align}
has only one \JU-model, $\emptyset$, but the \AS-semantics admits the
additional undesired model $\set{\atma}$. Nevertheless, there are also
situations in which the \JU-semantics assigns additional models only due to
the presence of tautological rules. This is discussed in more detail in
Section~\ref{sec:discussion}.

The final syntactic property can now be stated as follows:

\begin{property}
	[Acyclic Justified Update]
	A rule update semantics \S{} \emph{respects acyclic justified update} if for
	every acyclic DLP $\dprg$, the set of \S-stable models of $\dprg$ is
	$\modju{\dprg}$.
\end{property}

\subsection{Program Equivalence}

%
While in propositional logic equivalence under classical models is \emph{the}
equivalence, there is no such single notion of program equivalence. When
considering Answer-Set Programs, the first choice is \emph{stable equivalence}
(or \emph{\SM-equivalence}) that compares programs based on their sets of
stable models.

In many cases, however, \SM-equivalence is not strong enough because programs
with the same stable models, when augmented with the same additional rules,
may end up having completely different stable models. This gives rise to the
notion of \emph{strong equivalence} \cite{Lifschitz2001} which requires that
stable models stay the same even in the presence of additional rules. It is a
well-known fact that programs are strongly equivalent if and only if they have
the same set of \SE-models \cite{Turner2003}. Thus, we refer to strong
equivalence as \emph{\SE-equivalence}.

But even \SE-equivalence is not satisfactory when used as a basis for
syntax-independent rule update operators because such operators cannot respect
both support and fact update \cite{Slota2013a}.
So in order to arrive at syntax-independent rule update operators that satisfy
the basic intuitions underlying rule updates, we need to search for a notion
of program equivalence that is stronger than \SE-equivalence. One candidate is
the \emph{strong update equivalence} (or \emph{\SU-equivalence})
\cite{Inoue2004}, which requires that under both additions and removals of
rules, stable models of the two programs in question remain the same. It has
been shown in \cite{Inoue2004} that this notion of equivalence is very strong
-- programs are \SU-equivalent only if they contain exactly the same
non-tautological rules, and in addition, each of them may contain some
tautological ones. Thus, this notion of program equivalence seems perhaps
\emph{too strong} as it is not difficult to find rules such as $(\lpnot \atm
\lpif \atm.)$ and $(\lpif \atm.)$ that are syntactically different but carry
the same meaning.

This observation resulted in the definition of \emph{strong rule equivalence}
(or \emph{\SR-equivalence}) and \emph{strong minimal rule equivalence} (or
\emph{\SMR-equivalence}) in \cite{Slota2011} that, in terms of strength,
fall between \SE-equivalence and \SU-equivalence. It is based on the idea of
viewing a program $\prg$ as the set of sets of \SE-models of its rules
$\modser{\prg} = \Set{\modse{\rl} | \rl \in \prg}$.

The five mentioned notions of program equivalence are defined as follows:

\begin{definition}
	[Program Equivalence]
	Let $\prga, \prgb$ be programs, $\prga^\ctau = \prga \cup \set{\ctau}$,
	$\prgb^\ctau = \prgb \cup \set{\ctau}$, let $\min \sstri$ denote the
	subset-minimal elements of any set of sets $\sstri$ and $\div$ denote
	set-theoretic symmetric difference. We write
	\begin{align*}
		\prga &\eqSM \prgb
			&& \text{whenever}
			&& \modsm{\prga} = \modsm{\prgb}
			\enspace;
		&
		\prga &\eqSMR \prgb
			&& \text{whenever}
			&& \!\min \modser{\prga^\ctau} = \min \modser{\prgb^\ctau}
			\enspace;
		\\
		\prga &\eqSE \prgb
			&& \text{whenever}
			&& \modse{\prga} = \modse{\prgb}
			\enspace;
		&
		\prga &\eqSR \prgb
			&& \text{whenever}
			&& \modser{\prga^\ctau} = \modser{\prgb^\ctau}
			\enspace;
		\\
		\prga &\eqSU \prgb
			&& \text{whenever}
			&& \modse{\prga \div \prgb} = \tris
			\enspace.
	\end{align*}
	We say that \emph{$\prga$ is \X-equivalent to $\prgb$} if $\prga \eqX \prgb$.
\end{definition}

So two programs are \SR-equivalent if they contain the same rules, modulo
\SE-models; $\ctau$ is added to both programs so that presence or absence of
tautological rules in a program does not influence program equivalence. In the
case of \SMR-equivalence, only the subset-minimal sets of \SE-models are
compared, the motivation being that programs such as $\prga = \set{\atma \lpif
\atmb.}$ and $\prgb = \set{\atma., \atma \lpif \atmb.}$, when updated, should
behave the same way since the extra rule in $\prgb$ is just a weakened version
of the rule in $\prga$. Though $\prga$ and $\prgb$ are not \SR-equivalent,
they are \SMR-equivalent.

To formally capture the comparison of strength between these notions of
program equivalence, we write $\eqX \strle \eqY$ if $\prga \eqY \prgb$ implies
$\prga \eqX \prgb$ and $\eqX \strl \eqY$ if $\eqX \strle \eqY$ but not $\eqY
\strle \eqX$. Then:

\begin{proposition}
	[\cite{Slota2011}]
	\label{prop:se:equivalence comparison}
	$\eqSM \,\prec\, \eqSE \,\prec\, \eqSMR \,\prec\, \eqSR \,\prec\, \eqSU$.
\end{proposition}

\section{Robust Equivalence Models}

\label{sec:robust equivalence models}

In \cite{Slota2011} we studied the expressivity of \SE-models with respect to
a single rule. On the one hand, \SE-models turned out to be a useful means of
stripping away irrelevant syntactic details. On the other hand, a rule with a
default literal in its head is indistinguishable from an integrity constraint
\cite{Inoue1998,Janhunen2001,Cabalar2007a}. For example, the rules
\begin{align}
	\label{eq:strongly equivalent rules}
	&\lpif \atma, \atmb. \quad
	& \lpnot \atma &\lpif \atmb. \quad
	& \lpnot \atmb &\lpif \atma. \quad
\end{align}
have the same set of \SE-models. In a static setting, these rules indeed carry
essentially the same meaning: ``it must not be the case that $\atma$ and
$\atmb$ are both true''. But in a dynamic context, the latter two rules may,
in addition, express that the truth of one atom gives a \emph{reason} for the
other atom to \emph{cease being true}. For example, an update of the program
$\set{\atma., \atmb.}$ by $\set{\lpnot \atma \lpif \atmb.}$ leads to the
stable model $\set{\atmb}$ while an update by $\set{\lpnot \atmb \lpif
\atma.}$ to the stable model $\set{\atma}$. This convention is adopted by
causal rejection-based rule update semantics
\cite{Leite1997,Alferes2000,Eiter2002,Alferes2005,Osorio2007} which constitute
one of the most mature approaches to rule updates.

In order to be able to semantically characterise causal rejection-based rule
update semantics, we need to distinguish between constraints and rules with
default literals in their heads. These classes can be formally captured as
follows:

\begin{definition}
	[Constraint and Abolishing Rule]
	A rule $\rl$ is a \emph{constraint} if $\hrl = \emptyset$ and $\brl^+$ is
	disjoint with $\brl^-$.\footnote{%
		The latter condition guarantees that a constraint is not tautological.
	}
	A rule $\rl$ is \emph{abolishing} if $\hrl^+ = \emptyset$, $\hrl^- \neq
	\emptyset$ and the sets $\hrl^-$, $\brl^+$ and $\brl^-$ are pairwise
	disjoint.
\end{definition}

\noindent So what we are looking for is a semantic characterisation of rules that
\begin{textenum}[1)]
	\item can distinguish constraints from related abolishing rules;
		\label{RE requirement:1}

	\item discards irrelevant syntactic details (akin to \SE-models);
		\label{RE requirement:2}

	\item has a clear link to stable models (akin to \SE-models).
		\label{RE requirement:3}
\end{textenum}
In the following we introduce a novel monotonic semantics that exactly meets
these criteria. We show that it possesses the desired properties and use it to
introduce a notion of program equivalence that is strong enough as a basis for
syntax-independent rule update operators.

Without further ado, \emph{robust equivalence models}, or \emph{\RE-models}
for short, are defined as follows:

\begin{definition}
	[\RE-Model]
	\label{def:re model}
	Let $\rl$ be a rule. A three-valued interpretation $\twiab \in \tris$ is an
	\emph{\RE-model of $\rl$} if $\twia \ent \rl^\twib$. The set of all
	\RE-models of a rule $\rl$ is denoted by $\modre{\rl}$ and for any program
	$\prg$, $\modre{\prg} = \bigcap_{\rl \in \prg} \modre{\rl}$.

	We say that a rule $\rl$ is \emph{\RE-tautological} if $\modre{\rl} =
	\tris$. Rules $\rla$, $\rlb$ are \emph{\RE-equivalent} if $\modre{\rla} =
	\modre{\rlb}$.
\end{definition}

Thus, unlike with \SE-models, it is not required that $\twib \ent \rl$ in
order for $\twiab$ to be an \RE-model of $\rl$. As a consequence, \RE-models
can distinguish between rules in \eqref{eq:strongly equivalent rules}: while
both $\tpl{\set{\atmb}, \set{\atma, \atmb}}$ and $\tpl{\set{\atma},
\set{\atma, \atmb}}$ are \RE-models of the constraint, the former is not an
\RE-model of the first abolishing rule and the latter is not an \RE-model of
the second abolishing rule. This property holds in general, establishing
requirement \ref{RE requirement:1}):

\begin{proposition}
	\label{prop:re:abolishing rules not equivalent}
	If $\rla$, $\rlb$ are two different abolishing rules or an abolishing rule
	and a constraint, then $\rla$, $\rlb$ are not \RE-equivalent.
\end{proposition}
\begin{proof}
	See \ref{app:program-equivalence}, page~\pageref{proof:re:abolishing rules
	not equivalent}.
\end{proof}

As for requirement \ref{RE requirement:2}), we first note that \RE-equivalence
is a refinement of \SE-equivalence -- there are no rules that are
\RE-equivalent but not \SE-equivalent. The following result also shows that it
is \emph{only} the ability to distinguish between constraints and abolishing
rules that is introduced by \RE-models -- rules that are not \RE-equivalent to
abolishing rules are distinguishable by \RE-models if and only if they are
distinguishable by \SE-models. Furthermore, the classes of \SE-tautological
and \RE-tautological rules coincide, so we can simply use the word
\emph{tautological} without ambiguity.

\begin{proposition}
	[\RE-Equivalence vs.\ \SE-Equivalence]
	\label{prop:re:se-equivalence}
	~
	\begin{textitem}
		\item If two rules are \RE-equivalent, then they are \SE-equivalent.

		\item If two rules, neither of which is \RE-equivalent to an abolishing
			rule, are \SE-equivalent, then they are \RE-equivalent.

		\item A rule is \RE-tautological if and only if it is \SE-tautological.
	\end{textitem}
\end{proposition}
\begin{proof}
	See \ref{app:program-equivalence}, page~\pageref{proof:re:se-equivalence}.
\end{proof}

The affinity between \SE-models and stable models is fully retained by
\RE-models, which establishes requirement \ref{RE requirement:3}).

\begin{proposition}
	[\RE-Models vs.\ Stable Models]
	\label{prop:re:sm}
	An interpretation $\twib$ is a stable model of a program $\prg$ if and only
	if $\tpl{\twib, \twib} \in \modre{\prg}$ and for all $\twia \subsetneq
	\twib$, $\twiab \notin \modre{\prg}$.
\end{proposition}
\begin{proof}
	See \ref{app:program-equivalence}, page~\pageref{proof:re:sm}.
\end{proof}

Also worth noting is that any set of three-valued interpretations can be
expressed by a program using \RE-models. This is not the case with
\SE-models since only well-defined sets of three-valued interpretations have
corresponding programs.

\begin{proposition}
	\label{prop:re:expressibility}
	For any $\stri \subseteq \tris$ there exists a program $\prg$ such that
	$\modre{\prg} = \stri$.
\end{proposition}
\begin{proof}
	See \ref{app:program-equivalence}, page~\pageref{proof:re:expressibility}.
\end{proof}

Further properties of \RE-models, analogous to those established in
\cite{Slota2011} for \SE-models, can be found in \ref{app:program-equivalence}
starting on page~\pageref{app:program-equivalence}.

Since \RE-models are able to distinguish constraints from abolishing rules
while keeping the essential properties of \SE-models, we henceforth adopt them
as the basis for defining syntax-independent rule update operators. We denote
the set of sets of \RE-models of rules inside a program $\prg$ by
$\modrer{\prg} = \Set{ \modre{\rl} | \rl \in \prg }$. We also introduce three
additional notions of program equivalence: \RE-, \RR- and \RMR-equivalence
that are analogous to \SE-, \SR- and \SMR-equivalence. 

\begin{definition}
	[Program Equivalence Using \RE-Models]
	Let $\prga, \prgb$ be programs, $\prga^\ctau = \prga \cup \set{\ctau}$,
	$\prgb^\ctau = \prgb \cup \set{\ctau}$. We write
	\begin{align*}
		\prga &\eqRE \prgb
			&& \text{whenever}
			&& \modre{\prga} = \modre{\prgb}
			\enspace;
		&
		\prga &\eqRMR \prgb
			&& \text{whenever}
			&& \min \modrer{\prga^\ctau} = \min \modrer{\prgb^\ctau}
			\enspace;
		\\
		&&&&&&
		\prga &\eqRR \prgb
			&& \text{whenever}
			&& \modrer{\prga^\ctau} = \modrer{\prgb^\ctau}
			\enspace.
	\end{align*}
\end{definition}

In order to consider belief update principles in the context of rule updates,
we also need to establish notions of \emph{program entailment} which are in
line with the above defined program equivalence relations. This task is
troublesome in case of \SM-equivalence because the usage of entailment in
belief update postulates is clearly a monotonic one while stable models are
non-monotonic. For instance, a reformulation of \bu{1} would require that
$\prga \uopr \prgu \ent \prgu$, though there is no reason for $\prga \uopr
\prgu$ to have \emph{less} stable models than (or the same as) $\prgu$. Due to
these issues, we refrain from defining \SM-entailment. The remaining
entailment relations are defined as follows:

\begin{definition}
	[Program Entailment]
	Let $\prga, \prgb$ be programs, $\prga^\ctau = \prga \cup \set{\ctau}$ and
	$\prgb^\ctau = \prgb \cup \set{\ctau}$. We write
	\begin{align*}
		\prga &\entSE \prgb
			&& \text{whenever}
			&& \modse{\prga} \subseteq \modse{\prgb}
			\enspace;
		&
		\prga &\entSMR \prgb
			&& \text{whenever}
			&& \forall \rlb \in \prgb^\ctau \, \exists \rla \in \prga^\ctau :
				\modse{\rla} \subseteq \modse{\rlb}
			\enspace;
		\\
		\prga &\entRE \prgb
			&& \text{whenever}
			&& \modre{\prga} \subseteq \modre{\prgb}
			\enspace;
		&
		\prga &\entRMR \prgb
			&& \text{whenever}
			&& \forall \rlb \in \prgb^\ctau \, \exists \rla \in \prga^\ctau :
			\modre{\rla} \subseteq \modre{\rlb}
			\enspace;
		\\
		\prga &\entSU \prgb
			&& \text{whenever}
			&& \modse{\prgb \setminus \prga} = \tris
			\enspace;
		&
		\prga &\entSR \prgb
			&& \text{whenever}
			&& \forall \rlb \in \prgb^\ctau \, \exists \rla \in \prga^\ctau :
			\modse{\rla} = \modse{\rlb}
			\enspace;
		\\
		&&&&&&
		\prga &\entRR \prgb
			&& \text{whenever}
			&& \forall \rlb \in \prgb^\ctau \, \exists \rla \in \prga^\ctau :
			\modre{\rla} = \modre{\rlb}
			\enspace.
	\end{align*}
	We say that \emph{$\prga$ \X-entails $\prgb$} if $\prga \entX \prgb$.
\end{definition}

As the following proposition shows, the defined entailment relations are fully
in line with the corresponding equivalence relations.

\begin{proposition}
	\label{prop:ent vs eq}
	Let \X{} be one of \SE{}, \RE{}, \SMR{}, \RMR{}, \SR{}, \RR{}, \SU{} and
	$\prga$, $\prgb$ be programs. Then, $\prga \eqX \prgb$ if and only if $\prga
	\entX \prgb$ and $\prgb \entX \prga$.
\end{proposition}
\begin{proof}
	See \ref{app:program-equivalence}, page~\pageref{proof:prop:ent vs
	eq}.
\end{proof}

Note that it follows directly from the previous considerations that
\RR-equivalence is stronger than \SR-equivalence, \RMR-equivalence is stronger
than \SMR-equivalence and \RE-equivalence is stronger than \SE-equivalence.
Figure~\ref{fig:equivalences} illustrates the strength comparison of all eight
notions of program equivalence which are formally stated in
Proposition~\ref{prop:se:equivalence comparison} and in the following result:

\begin{figure}[t]
	\centering
	\begin{tikzpicture} 
		\tikzstyle{every node}=[circle, fill=black, minimum size=4pt, inner sep=0pt]
		\draw (0,0) node (sm) [label=below:\SM{}] {}
			-- +(1,0) node (se) [label=below:\SE{}] {};
		\draw (se) -- +(.707,.707) node (smr) [label=left:\SMR{}] {};
		\draw (smr) -- +(.707,.707) node (sr) [label=left:\SR{}] {};
		\draw (sr) -- +(.707,-.707) node (mr) [label=above:\RR{}] {};
		\draw (mr) -- +(1,0) node (su) [label=above:\SU{}] {};
		\draw (se) -- +(.707,-.707) node (me) [label=right:\RE{}] {};
		\draw (me) -- +(.707,.707) node (mmr) [label=right:\RMR{}] {};
		\draw (mmr) -- (mr);
		\draw (smr) -- (mmr);
	\end{tikzpicture}
	\caption[Notions of program equivalence and their strength]{%
		Notions of program equivalence and entailment from the weakest on the left
		side to the strongest on the right side. A missing link between
		\X[\scriptsize X]{} and \Y[\scriptsize Y]{} indicates that $\eqX$ is
		incomparable with $\eqY$ and $\entX$ is incomparable with $\entY$, i.e.\
		\emph{none} of the following four relations holds: $\eqX \strle \eqY$,
		$\eqY \strle \eqX$, $\entX \strle \entY$, $\entY \strle \entX$.
	}
	\label{fig:equivalences}
\end{figure}
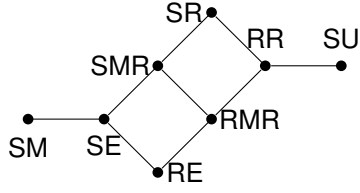

\begin{proposition}
	\label{prop:eq and ent comparison}
	The following holds:
	\begin{enumerate}[(1)]
		\setlength{\itemsep}{0pt}
		\newcommand{\bxl}[1]{\makebox[15em][l]{#1}}
		\newcommand{\bxc}[1]{\makebox[6em][c]{#1}}

		\item
			\bxl{$\eqSM \strl \eqSE \strl \eqRE \strl \eqRMR \strl \eqRR \strl \eqSU$}
			\bxc{and}
			$\entSE \strl \entRE \strl \entRMR \strl \entRR \strl \entSU$
			\enspace;

		\item
			\bxl{$\eqSE \strl \eqSMR \strl \eqSR \strl \eqRR$}
			\bxc{and}
			$\entSE \strl \entSMR \strl \entSR \strl \entRR$
			\enspace;

		\item
			\bxl{$\eqSMR \strl \eqRMR$}
			\bxc{and}
			$\entSMR \strl \entRMR$
			\enspace;

		\item
			\bxl{$\eqRE \nstrle \eqSMR$ and $\eqSMR \nstrle \eqRE$}
			\bxc{and}
			$\entRE \nstrle \entSMR$ and $\entSMR \nstrle \entRE$
			\enspace;

		\item
			\bxl{$\eqRE \nstrle \eqSR$ and $\eqSR \nstrle \eqRE$}
			\bxc{and}
			$\entRE \nstrle \entSR$ and $\entSR \nstrle \entRE$
			\enspace;

		\item
			\bxl{$\eqRMR \nstrle \eqSR$ and $\eqSR \nstrle \eqRMR$}
			\bxc{and}
			$\entRMR \nstrle \entSR$ and $\entSR \nstrle \entRMR$
			\enspace.
	\end{enumerate}
\end{proposition}
\begin{proof}
	See \ref{app:program-equivalence}, page~\pageref{proof:prop:eq and ent
	comparison}.
\end{proof}

\section{Exception-Based Rule Update Operators}

\label{sec:exception-based-updates:rule updates}

In this section we propose a generic scheme for specifying semantic rule
update operators. We define instances of the scheme and show that they enjoy a
number of plausible properties, ranging from the respect for support and fact
update to syntax-independence and other semantic properties.

As suggested already, a program is semantically characterised by the set of
sets of \RE-models of its rules. Our update framework is based on a simple yet
novel idea of introducing additional interpretations -- \emph{exceptions} --
to the sets of \RE-models of rules in the original program. The formalisation
of this idea is straightforward: an exception-based update operator is
characterised by an \emph{exception function} \te{} that takes three inputs:
the set of \RE-models $\modre{\rl}$ of a rule $\rl \in \prga$ and the semantic
characterisations $\modrer{\prga}$, $\modrer{\prgu}$ of the original and
updating programs. Then it returns the interpretations that are to be
introduced as exceptions to $\rl$, so the characterisation of the updated
program contains the augmented set of \RE-models
\begin{equation}
	\label{eq:re models of rule after update}
	\modre{\rl} \cup \e \Br{
		\modre{\rl}, \modrer{\prga}, \modrer{\prgu}
	}
	\enspace.
\end{equation}
Hence, the semantic characterisation of $\prga$ updated by $\prgu$ is
\begin{equation}
	\label{eq:characterisation after update}
	\Set{
		\modre{\rl} \cup \e \Br{
			\modre{\rl}, \modrer{\prga}, \modrer{\prgu}
		}
		|
		\rl \in \prga
	}
	\cup \modrer{\prgu}
	\enspace.
\end{equation}
In other words, rules from $\prga$ are augmented with the respective
exceptions and rules in $\prgu$ remain untouched.

From the syntactic viewpoint, we want a rule update operator $\uopr$ to return
a program $\prga \uopr \prgu$ with the semantic characterisation
\eqref{eq:characterisation after update}. This brings us to the following
issue: What if no rule exists whose set of \RE-models is equal to \eqref{eq:re
models of rule after update}? In that case, no rule corresponds to the
augmented set of \RE-models of a rule $\rl \in \prga$, so the program $\prga
\uopr \prgu$ cannot be constructed. Moreover, such situations may occur quite
frequently since a single rule has very limited expressivity. For instance,
updating the fact $(\atma.)$ by the rule $(\lpnot \atma \lpif \atmb, \atmc.)$
may easily result in a set of \RE-models expressible by the program
$\set{\atma \lpif \lpnot \atmb., \atma \lpif \lpnot \atmc.}$ but not
expressible by any single rule. To keep a firm link to operations on
syntactic objects, we henceforth deal with this problem by allowing the inputs
and output of rule update operators to be sets of rules \emph{and programs},
which we dub \emph{rule bases}.\footnote{%
	We allow for individual rules in a rule base out of convenience only. A
	single rule $\rl$ in a rule base $\rb$ is treated exactly the same way as if
	$\rb$ contained the singleton program $\set{\rl}$.
}
In other words, the result of updating a rule, i.e.\ introducing exceptions to
it, may be a set of rules, so the result of updating a program may be a rule
base. Technically, a rule base can capture any possible result of an
exception-based update due to Proposition~\ref{prop:re:expressibility}.

Formally, a \emph{rule base} is any set of rules and programs. Given a rule
base $\rb$, an interpretation $\twib$ is a \emph{model of $\rb$}, denoted by
$\twib \ent \rb$, if $\twib \ent \gprg$ for all $\gprg \in \rb$; $\rb^\twib =
\Set{\gprg^\twib | \gprg \in \rb}$; the set of \emph{stable models of $\rb$}
consists of all interpretations $\twib$ such that $\twib$ is a subset-minimal
model of $\rb^\twib$; $\modser{\rb} = \Set{\modse{\gprg} | \gprg \in \rb}$ and
$\modse{\rb} = \bigcap \modser{\rb}$; $\modrer{\rb} = \Set{\modre{\gprg} |
\gprg \in \rb}$ and $\modre{\rb} = \bigcap \modrer{\rb}$. All notions of
program equivalence and entailment are naturally extended to rule bases and any
rule update operator $\uopr$ is generalised to sequences of rule bases
$\seq{\rb_\lia}_{\lia < \lng}$ as follows: $\biguopr \seq{} = \emptyset$ and
$\biguopr \seq{\rb_\lia}_{\lia < \lng + 1} = (\biguopr \seq{\rb_\lia}_{\lia <
\lng}) \uopr \rb_\lng$.

Note that a program is a special case of a rule base. Each element $\gprg$ of
a rule base, be it a rule or a program, represents an \emph{atomic} piece of
information and exception-based update operators view and manipulate $\gprg$
only through its set of \RE-models $\modre{\gprg}$. Due to this, we refer to
all such elements $\gprg$ as \emph{rules}, even if formally they may actually
be programs.

Having resolved this issue, we can proceed to the definition of an
exception-based rule update operator.

\begin{definition}
	[Exception-Based Rule Update Operator]
	\label{def:exception-based rule update operator}
	Given an exception function \te{}, a rule update operator $\uopr$ is
	\emph{\te-based} if for all rule bases $\rba$, $\rbu$, $\modrer{\rba \uopr
	\rbu}$ is equal to \eqref{eq:characterisation after update}. Also, $\uopr$
	is \emph{exception-based} if it is \te-based for some exception function
	\te{}.
\end{definition}

Note that for each exception function \te{} there is a whole class of
\te-based rule update operators that differ in the syntactic representations
of the sets of \RE-models in \eqref{eq:characterisation after update}. For
instance, when working over the set of propositional symbols $\atms =
\set{\atma, \atmb}$ and considering some exception-based operator $\uopr$, the
exception function may specify that for some programs $\prga$, $\prgu$, the
program $\prga \uopr \prgu$ contains a rule or program representing the set of
\RE-models
$
	\stri = \set{
		\tpl{\emptyset, \emptyset},
		\tpl{\emptyset, \atma}, \tpl{\atma, \atma},
		\tpl{\emptyset, \atmb},
		\tpl{\emptyset, \atma\atmb}, \tpl{\atma, \atma\atmb},
		\tpl{\atma\atmb, \atma\atmb}
	}
$.\footnote{%
	Sometimes we omit the usual set notation when we write interpretations,
	e.g.\ instead of $\set{\atma, \atmb}$ we may simply write $\atma\atmb$.
}
This set can be represented by the rule $\rla = (\atma \lpif \atmb.)$ or,
alternatively, by the rule $\rlb = (\atma; \lpnot \atmb \lpif \atmb.)$, or
even by the program $\prgb = \Set{(\atma \lpif \atmb.), (\atma \lpif \atma,
\lpnot \atmb.)}$, and the exception function does not specify which syntactic
representation of the set should be used in $\prga \uopr \prgu$.

\subsection{Simple Exception Functions and Their Syntactic Properties}

General exception functions, as introduced above, have as inputs the entire
semantic characterisations of the original as well as updating programs
($\modrer{\prga}$ and $\modrer{\prgu}$, respectively) when determining
exceptions to any single rule. As it turns out, all this information is not
strictly necessary in order to capture rule update operators enjoying a range
of plausible syntactic as well as semantic properties.

Thus, further study of general exception functions is left for future research
and in this paper we concentrate on a simpler, constrained class of exception
functions that is nevertheless powerful enough to serve as a basis for
well-behaved semantic rule update operators. Not only does this lead to
simpler definitions, but the study of restricted classes of exception
functions is essential in order to understand their expressivity, i.e.\ the
types of update operators they are able to capture.

More particularly, we study exception functions that produce exceptions based
on conflicts between pairs of rules, one from the original and one from the
updating program, while ignoring the context in which these rules are
situated. Formally, an exception function \te{} is \emph{simple} if for all
$\stria \subseteq \tris$ and $\sstria, \sstrib \subseteq \pws{\tris}$,
\[
	\e(\stria, \sstria, \sstrib) = \textstyle \bigcup_{\strib \in \sstrib} \er(\stria, \strib)
\]
where $\er : \pws{\tris} \times \pws{\tris} \rightarrow \pws{\tris}$ is a
\emph{local exception function}. If $\uopr$ is an \te-based rule update
operator, then we also say that $\uopr$ is \emph{\ter-based} and that $\uopr$
is \emph{simple}.

As we shall see, in spite of their local nature, particular simple exception
functions generate rule update operators that satisfy the syntactic properties
laid out in Section~\ref{sec:background} and are closely related to the \JU-
and \AS-semantics for DLPs.

The inspiration for defining concrete local exception functions \ter{} comes
from rule update semantics based on causal rejection. But since the relevant
concepts, such as that of a \emph{conflict} or \emph{rule rejection}, rely on
rule syntax to which an exception function has no direct access, our first
objective is to find similar concepts on the semantic level. In particular, we
need to define conditions under which two sets of \RE-models are in conflict.
But first we introduce two preparatory concepts.

We define a \emph{truth value substitution} as follows: Given an
interpretation $\twib$, an atom $\atma$ and a truth value $\val \in \set{\tr,
\un, \fa}$, by $\twib[\val/\atma]$ we denote the three-valued interpretation
$\tri$ such that $\tri(\atma) = \val$ and $\tri(\atmb) = \twib(\atmb)$ for all
atoms $\atmb \neq \atma$.

This enables us to introduce the main concept needed for defining a conflict
between two sets of three-valued interpretations. Given a set of three-valued
interpretations $\stri$, an atom $\atm$, a truth value $\val_0$ and a
two-valued interpretation $\twib$, we say that \emph{$\stri$ forces $\atm$ to
have the truth value $\val_0$ w.r.t.\ $\twib$}, denoted by $\stri^\twib(\atm)
= \val_0$, if
\[
	\twib[\val/\atm] \in \stri \text{ if and only if } \val = \val_0 \enspace.
\]
In other words, the three-valued interpretation $\twib[\val_0/\atm]$ must be
the unique member of $\stri$ that either coincides with $\twib$ or differs
from it only in the truth value of $\atm$. Note that $\stri^\twib(\atm)$ stays
undefined in case no $\val_0$ with the above property exists.

Two sets of three-valued interpretations $\stria$, $\strib$ are \emph{in
conflict on atom $\atm$ w.r.t.\ $\twib$}, denoted by $\stria
\confl{\twib}{\atm} \strib$, if both $\stria^\twib(\atm)$ and
$\strib^\twib(\atm)$ are defined and $\stria^\twib(\atm) \neq
\strib^\twib(\atm)$. The following example illustrates all these concepts.

\begin{example}
	\label{ex:semantic conflict}
	Consider rules $\rla = (\atma.)$, $\rlb = (\lpnot \atma \lpif \lpnot
	\atmb.)$ with the respective sets of \RE-models
	\begin{align*}
		\stria &= \Set{
				\tpl{\atma, \atma}, \tpl{\atma, \atma\atmb},
				\tpl{\atma\atmb, \atma\atmb}
			}
		&& \text{and}
		&
		\strib &= \Set{
			\tpl{\emptyset, \emptyset}, \tpl{\emptyset, \atmb}, \tpl{\atmb, \atmb},
			\tpl{\emptyset, \atma\atmb}, \tpl{\atma, \atma\atmb},
			\tpl{\atmb, \atma\atmb}, \tpl{\atma\atmb, \atma\atmb}
		}
		\enspace.
	\end{align*}
	Intuitively, $\stria$ forces $\atma$ to $\tr$ w.r.t.\ all interpretations
	and $\strib$ forces $\atma$ to $\fa$ w.r.t.\ interpretations in which $\atmb$
	is false. Formally it follows that $\stria^\emptyset(\atma) = \tr$ because
	$\tpl{\atma, \atma}$ belongs to $\stria$ and neither $\tpl{\emptyset,
	\atma}$ nor $\tpl{\emptyset, \emptyset}$ belongs to $\stria$. Similarly, it
	follows that $\strib^\emptyset(\atma) = \fa$. Hence $\stria
	\confl{\emptyset}{\atma} \strib$. Using similar arguments we can conclude
	that $\stria \confl{\atma}{\atma} \strib$. However, it does not hold that
	$\stria \confl{\atma\atmb}{\atma} \strib$ because
	$\strib^{\atma\atmb}(\atma)$ is undefined.
\end{example}

In order to define a particular local exception function based on causal
rejection, it only remains to decide which three-valued interpretations become
exceptions when a conflict w.r.t.\ an interpretation $\twib$ occurs. One
intuition we can draw on is the relationship between \RE-models and stable
models (c.f.\ Proposition~\ref{prop:re:sm}): $\twib$ is a stable model of a
program if and only if $\tpl{\twib, \twib}$ is the unique \RE-model of the
program whose second component is $\twib$. So given an original rule $\rla$
with $\stria = \modre{\rla}$ and an updating rule $\rlb$ with $\strib =
\modre{\rlb}$ where a conflict occurs between $\stria$ and $\strib$ w.r.t.\
$\twib$, $\rla$ needs to be weakened so that it cannot influence whether
$\twib$ becomes a stable model of the updated program or not. In other words,
we need to introduce all three-valued interpretations whose second component
is $\twib$ as exceptions to that rule. Formally:

\begin{definition}
	[Exception Function \terone{}]
	The local exception function \terone{} is for all $\stria, \strib \subseteq
	\tris$ defined by:
	\[
		\erone(\stria, \strib) = \Set{
			\tpl{\twia, \twib} \in \tris
			|
			\exists \atm : \stria \confl{\twib}{\atm} \strib
		} \enspace.
	\]
\end{definition}

\begin{theorem}
	[Syntactic Properties of \terone]
	\label{thm:syntactic properties:erone}
	Every \terone-based rule update operator respects support and fact update
	and it also respects causal rejection and acyclic justified update
	w.r.t.\ DLPs of length at most two.
\end{theorem}
\begin{proof}
	See \ref{app:exception-based rule update}, page
	\pageref{proof:thm:re:syntactic properties:erone}.
\end{proof}

This means that \terone-based rule update operators enjoy a combination of
desirable syntactic properties that operators based on \SE-models cannot
\cite{Slota2013a}. However, these operators diverge from causal rejection,
even on acyclic DLPs, when more than one update is performed.

\begin{example}
	\label{ex:terone iterated}
	Consider again the rules $\rla$, $\rlb$ and their sets of \RE-models
	$\stria$, $\strib$ from Example~\ref{ex:semantic conflict} and some
	\terone-based rule update operator $\uopr$. Then $\modrer{\set{\rla} \uopr
	\set{\rlb}}$ will contain two elements, $\stria'$ and $\strib$, where
	$
		\stria'
		= \stria \cup \erone(\stria, \strib)
		= \stria \cup \set{
			\tpl{\emptyset, \emptyset},
			\tpl{\emptyset, \atma},
			\tpl{\atma, \atm}
		}
	$.
	An additional update by the fact $\set{\atmb.}$ then leads to the
	characterisation
	$
		\modrer{\biguopr \seq{\set{\rla}, \set{\rlb}, \set{\atmb.}}}
	$
	which contains three elements: $\stria''$, $\strib$ and $\modre{\atmb.}$ where
	$
		\stria'' = \stria'
			\cup \set{\tpl{\emptyset, \atmb}, \tpl{\atmb, \atmb}}
	$.
	Furthermore, due to Proposition~\ref{prop:re:sm}, $\twib = \set{\atmb}$ is a
	stable model of the program $\biguopr \seq{\set{\rla}, \set{\rlb},
	\set{\atmb.}}$ because $\tpl{\atmb, \atmb}$ belongs to $\modre{\biguopr
	\seq{\set{\rla}, \set{\rlb}, \set{\atmb.}}}$ and $\tpl{\emptyset, \atmb}$
	does not. However, $\twib$ does not respect causal rejection and it is
	neither a \JU- nor a \AS-model of $\seq{\set{\rla}, \set{\rlb}, \set{\atmb.}}$.
\end{example}

This shortcoming of \terone{} largely stems from the asymmetry of its
definition -- when a conflict occurs w.r.t.\ $\twib$, atoms that are true in
$\twib$ may become undefined in an exception $\twiab$ but atoms that are false
must remain false. Eliminating this asymmetry, by allowing for exceptions
$\tpl{\twia, \twic}$ with $\twia \subseteq \twib \subseteq \twic$, and
choosing the introduced exceptions more carefully, leads to more well-behaved
exception functions, defined as follows:

\begin{definition}
	[Exception Functions \tertwo{}, \terthree{}, \terfour{} and \terfive{}]
	The local exception functions \tertwo{}, \terthree{}, \terfour{} and
	\terfive{} are for all $\stria, \strib \subseteq \tris$ defined by:
	\begin{align*}
		\ertwo(\stria, \strib) &= \Set{
			\tpl{\twia, \twic} \in \tris \mid
				\exists \twib \, \exists \atm : \stria \confl{\twib}{\atm} \strib
				\land \twia \subseteq \twib \subseteq \twic
				\land (\atm \in \twic \setminus \twia \mlthen \twic = \twib)
			} \enspace, \\
		\erthree(\stria, \strib) &= \Set{
			\tpl{\twia, \twic} \in \tris \mid
				\exists \twib \, \exists \atm : \stria \confl{\twib}{\atm} \strib
				\land \twia \subseteq \twib \subseteq \twic
				\land (
					\atm \in \twic \setminus \twia
					\mlthen
					\twic = \twib
					\land
					\tpl{\twib, \twib} \notin \stria
				)
			} \enspace, \\
		\erfour(\stria, \strib) &= \begin{cases}
			\tris & \text{if } \stria = \strib \enspace; \\
			\ertwo(\stria, \strib) & \text{otherwise} \enspace,
		\end{cases}
		\qquad \qquad
		\erfive(\stria, \strib) = \begin{cases}
			\tris & \text{if } \stria = \strib \enspace; \\
			\erthree(\stria, \strib) & \text{otherwise} \enspace.
		\end{cases}
	\end{align*}
\end{definition}

The functions \tertwo{} and \terfour{} introduce more exceptions than
\terone{} while \terthree{} and \terfive{} eliminate some of those returned by
\terone{} and add some additional ones. The difference between \tertwo{} and
\terfour{}, and similarly also between \terthree{} and \terfive{}, is in that
\terfour{} additionally ``wipes out'' rules from the original program that are
repeated in the update by introducing all interpretations as exceptions to
them, rendering them tautological. This will turn out to be significant later
when we examine semantic properties of simple exception-based rule update
operators. In all four functions, a conflict on $\atm$ w.r.t.\ $\twib$ leads
to the introduction of interpretations in which atoms either maintain the
truth value they had in $\twib$, or they become undefined. Additionally, extra
conditions are imposed on the case when $\atm$ becomes undefined.
Interestingly, this leads to operators that satisfy all syntactic properties.

\begin{theorem}
	[Syntactic Properties of \tertwo{}, \terthree{}, \terfour{} and \terfive{}]
	\label{thm:syntactic properties:ertwo and erfour}
	Let $\uopr$ be a \tertwo-, \terthree-, \terfour- or \terfive-based rule
	update operator. Then $\uopr$ respects support, fact update, causal
	rejection and acyclic justified update.
\end{theorem}
\begin{proof}
	See \ref{app:exception-based rule update}, page
	\pageref{proof:thm:re:syntactic properties:ertwo and erfour}.
\end{proof}

Furthermore, it is worth noting that even on programs with cycles, \tertwo-
and \terfour-based operators are very closely related to the \JU-semantics
and, similarly, \terthree- and \terfive-based operators to the \AS-semantics.
They diverge from the syntax-based semantics only on rules with an appearance
of the same atom in both the head and body. Formally, we say that a rule is a
\emph{local cycle} if $(\hrlp \cup \hrln) \cap (\brlp \cup \brln) \neq
\emptyset$.

\begin{theorem}
	[Characterisation of \JU- and \AS-Semantics Using Exception Functions]
	\label{thm:er vs ju and as}
	Let $\dprg$ be a DLP, $\twib$ an interpretation, $\uopr_1$ a \tertwo- or
	\terfour-based rule update operator and $\uopr_2$ a \terthree- or
	\terfive-based rule update operator. Then,
	\begin{itemize}
		\item $\modsm{\biguopr_1 \dprg} \subseteq \modju{\dprg}$ and
			$\modsm{\biguopr_2 \dprg} \subseteq \modas{\dprg}$;

		\item if $\all{\dprg}$ contains no local cycles, then $\modju{\dprg}
			\subseteq \modsm{\biguopr_1 \dprg}$ and $\modas{\dprg} \subseteq
			\modsm{\biguopr_2 \dprg}$.
	\end{itemize}
\end{theorem}
\begin{proof}
	See \ref{app:exception-based rule update}, page
	\pageref{proof:thm:er vs ju and as}.
\end{proof}

This means that up to the marginal case of local cycles, \tertwo{} and
\terfour{} can be seen as semantic characterisations of the \JU-semantics and,
similarly, \terthree{} and \terfive{} characterise the \AS-semantics. The
exceptional cases when \emph{less} stable models are found than using the
traditional semantics occur when the DLP contains tautological rules or rules
with the negation of their head in the body. The former can be seen as a
strict improvement as it introduces immunity to tautologies. The latter is an
unavoidable consequence of the fact that exception-based operators only
manipulate \RE-models which are unable to distinguish between a constraint
$(\lpif \atm.)$ and a rule $(\lpnot \atm \lpif \atm.)$ The traditional
approaches do distinguish between them in that the former rule can never be
rejected while the latter can.

This tight relationship also sheds new light on the problem of \emph{state
condensing}, discussed in more detail in Section~\ref{sec:state condensing}.

\subsection{Semantic Properties}

We proceed by examining further properties of rule update operators -- of
those based on simple exception functions in general, and of the \terone-,
\tertwo-, \terthree-, \terfour- and \terfive-based ones in particular. The
properties we consider in this section are \emph{semantic} in that they put
conditions on the \emph{models} of a result of an update and do not need to
refer to the syntax of the original and updating programs. Our results are
summarised in Table~\ref{tbl:semantic properties} and in the following we
explain and discuss them. The interested reader may find all the proofs in
\ref{app:rules:semantic properties} starting on page
\pageref{app:rules:semantic properties}. 

\subsubsection*{Traditional Semantic Properties of Rule Updates}

The properties in the upper part of Table~\ref{tbl:semantic properties} were
introduced in \citep{Eiter2002, Alferes2005, Delgrande2007}. We formalise all
of them for rule bases $\rba$, $\rbb$, $\rbu$, $\rbv$ and a rule update
operator $\uopr$ and each can actually be seen as a \emph{meta-property} that
is
\noindent instantiated once we adopt a particular notion of program
equivalence. Therefore, each row of Table~\ref{tbl:semantic properties} has
eight cells that stand for particular instantiations of the property. This
provides a more complete picture of how simple rule update operators,
properties and program equivalence are interrelated.

Unless stated otherwise (in a footnote), each tick (\tick) signifies that the
property in question holds for \emph{all} simple rule update operators. A
missing tick signifies that the property does not generally hold for simple
rule update operators, and in particular there are \terone-, \tertwo-,
\terthree-, \terfour- and \terfive-based operators for which it is violated. A
tick is smaller if it is a direct consequence of a preceding larger tick in
the same row and of the interrelations between the program equivalence and
entailment relations (c.f.\ Figure~\ref{fig:equivalences}).

At a first glance, it is obvious that none of the semantic properties is
satisfied under \SU-equivalence. This is because the conditions placed on a
rule update operator by an exception function are at the semantic level, while
\SU-equivalence effectively compares programs syntactically. For instance, an
exception-based operator $\uopr$, for any exception function \te{}, may behave
as follows: $\emptyset \uopr \set{ \lpnot \atma \lpif \atma. } = \set{\lpif
\atma.}$. This is because the rules before and after update are
\RE-equivalent. However, due to the fact that the programs $\set{\lpnot \atma
\lpif \atma.}$ and $\set{\lpif \atma.}$ are considered different under
\SU-equivalence, $\uopr$ cannot satisfy \pup{Initialisation} w.r.t.\
\SU-equivalence. The situation with all other properties is analogous.

As for the other notions of equivalence, we separately discuss each group of
related properties:

\begin{description}
	\setlength{\itemsep}{0pt}
	\item[\pup{Initialisation} and \pup{Disjointness}:]
		These properties are satisfied ``by construction'', regardless of which
		simple rule update operator we consider and of which notion of equivalence
		we pick. 	
\end{description}

\begin{landscape}
	\begin{table}
		\newcommand{\tickn}{\tick}
		\newcommand{\tickt}{\tiny\tick}
		\newcommand{\tnotelabela}{\textsf{a}}
		\newcommand{\tnotelabelb}{\textsf{b}}
		\newcommand{\tnotelabelc}{\textsf{c}}
		\newcommand{\tnotelabeld}{\textsf{d}}
		\newcommand{\tnotelabele}{\textsf{e}}
		\newcommand{\tnotelabelf}{\textsf{$\dagger$}}
		\newcommand{\tnotelabelg}{\textsf{$\star$}}
		\newcommand{\lbabcde}{%
				\tnote{\tnotelabela\tnotelabelb\tnotelabelc\tnotelabeld\tnotelabele}
		}
		\newcommand{\lbbcde}{\tnote{\tnotelabelb\tnotelabelc\tnotelabeld\tnotelabele}}
		\newcommand{\lbde}{\tnote{\tnotelabeld\tnotelabele}}
		\newcommand{\lbf}{\tnote{\tnotelabelf}}
		\newcommand{\lbg}{\tnote{\tnotelabelg}}
		\centering
		\begin{threeparttable}
			\caption{Semantic properties of simple rule update operators}
			\label{tbl:semantic properties}
			\small
			\begin{tabular}{>{\centering}m{8.5em}m{30em}cccccccc}
				\toprule
					&& \multicolumn{8}{c}{\normalsize Type of $\equiv$, $\ent$ and $\mod{\cdot}$}
					\\ \cmidrule{3-10}
				\centering \normalsize Property
					& \centering \normalsize Formalisation
					& \SU{}
					& \RR{}
					& \SR{}
					& \RMR{}
					& \SMR{}
					& \RE{}
					& \SE{}
					& \SM{} \\ \midrule[\heavyrulewidth]
				\labpup{Initialisation}
					& $\emptyset \uopr \rbu \equiv \rbu$.
					&
					& \tickn
					& \tickt
					& \tickt
					& \tickt
					& \tickt
					& \tickt
					& \tickt
					\\ \midrule
				\labpup{Disjointness}
					& If $\rba$, $\rbb$ are over disjoint alphabets, 
						then $(\rba \cup \rbb) \uopr \rbu
							\equiv (\rba \uopr \rbu) \cup (\rbb \uopr \rbu)$.
					&
					& \tickn
					& \tickt
					& \tickt
					& \tickt
					& \tickt
					& \tickt
					& \tickt
					\\ \midrule
				\labpup{Non-interference}\lbf
					& If $\rbu$, $\rbv$ are over disjoint alphabets, 
						then $(\rba \uopr \rbu) \uopr \rbv
							\equiv (\rba \uopr \rbv) \uopr \rbu$.
					&
					& \tickn \lbabcde
					& \tickt \lbabcde
					& \tickt \lbabcde
					& \tickt \lbabcde
					& \tickt \lbabcde
					& \tickt \lbabcde
					& \tickt \lbabcde
					\\ \midrule
				\labpup{Tautology}
					& If $\rbu$ is tautological, then $\rba \uopr \rbu \equiv \rba$.
					&
					& \tickn \lbg
					& \tickt \lbg
					& \tickt \lbg
					& \tickt \lbg
					& \tickt \lbg
					& \tickt \lbg
					& \tickt \lbg
					\\ \midrule
				\labpup{Immunity to Tautologies}
					& If $\rbb$ and $\rbv$ are tautological, 
						then $(\rba \cup \rbb) \uopr (\rbu \cup \rbv) \equiv \rba \uopr \rbu$.
					&
					& \tickn \lbg
					& \tickt \lbg
					& \tickt \lbg
					& \tickt \lbg
					& \tickt \lbg
					& \tickt \lbg
					& \tickt \lbg
					\\ \midrule
				\labpup{Idempotence}
					& $\rba \uopr \rba \equiv \rba$.
					&
					& \tickn \lbde
					& \tickt \lbde
					& \tickn
					& \tickt
					& \tickt
					& \tickt
					& \tickt
					\\ \midrule
				\labpup{Absorption}
					& $(\rba \uopr \rbu) \uopr \rbu \equiv \rba \uopr \rbu$.
					&
					& \tickn \lbde
					& \tickt \lbde
					& \tickn \lbbcde
					& \tickt \lbbcde
					& \tickt \lbbcde
					& \tickt \lbbcde
					& \tickt \lbbcde
					\\ \midrule
				\labpup{Augmentation}\lbf
					& If $\rbu \subseteq \rbv$, then
						$(\rba \uopr \rbu) \uopr \rbv \equiv \rba \uopr \rbv$.
					&
					& \tickn \lbde
					& \tickt \lbde
					& \tickn \lbbcde
					& \tickt \lbbcde
					& \tickt \lbbcde
					& \tickt \lbbcde
					& \tickt \lbbcde
					\\ \midrule
				\labpup{Associativity}
					& $\rba \uopr (\rbu \uopr \rbv) \equiv (\rba \uopr \rbu) \uopr \rbv$.
					&
					&
					&
					&
					&
					&
					&
					&
					\\ \midrule[\heavyrulewidth]
				\labpu{1}
					& $\rba \uopr \rbu \ent \rbu$
					&
					& \tickn
					& \tickt
					& \tickt
					& \tickt
					& \tickt
					& \tickt
					& \na
					\\ \midrule
				\labputwotop
					& $\rba \uopr \emptyset \equiv \rba$
					&
					& \tickn
					& \tickt
					& \tickt
					& \tickt
					& \tickt
					& \tickt
					& \tickt
					\\ \midrule
				\labpu{2.1}
					& $\rba \cup \rbu \ent \rba \uopr \rbu$
					& 
					&
					&
					& \tickn
					& \tickt
					& \tickt
					& \tickt
					& \na
					\\ \midrule
				\labpu{2.2}
					& $(\rba \cup \rbu) \uopr \rbu \ent \rba$
					&
					&
					&
					&
					&
					&
					&
					& \na
					\\ \midrule
				\labpu{3}
					& If $\mod{\rba} \neq \emptyset$ and $\mod{\rbu} \neq \emptyset$, then
						$\mod{\rba \uopr \rbu} \neq \emptyset$
					& \na
					& \na
					& \na
					& \na
					& \na
					&
					&
					&
					\\ \midrule
				\labpu{4}
					& If $\rba \equiv \rbb$ and $\rbu \equiv \rbv$, then
						$\rba \uopr \rbu \equiv \rbb \uopr \rbv$
					&
					& \tickn \lbg
					&
					&
					&
					&
					&
					&
					\\ \midrule
				\labpu{5}
					& $(\rba \uopr \rbu) \cup \rbv \ent \rba \uopr (\rbu \cup \rbv)$
					&
					&
					&
					& \tickn
					& \tickt
					& \tickt
					& \tickt
					& \na
					\\ \midrule
				\labpu{6}
					& If $\rba \uopr \rbu \ent \rbv$ and $\rba \uopr \rbv \ent \rbu$, then
						$\rba \uopr \rbu \equiv \rba \uopr \rbv$
					&
					&
					&
					&
					&
					&
					&
					& \na
					\\ \bottomrule
			\end{tabular}
			\begin{tablenotes}
				\item[\tnotelabela]
					Holds if $\uopr$ is \terone-based.
				\item[\tnotelabelb]
					Holds if $\uopr$ is \tertwo-based.
				\item[\tnotelabelc]
					Holds if $\uopr$ is \terthree-based.
				\item[\tnotelabeld]
					Holds if $\uopr$ is \terfour-based.
				\item[\tnotelabele]
					Holds if $\uopr$ is \terfive-based.
				\item[\tnotelabelg]
					Holds if $\uopr$ is based on a local exception function \ter{} such
					that $\er(\stri, \tris) \subseteq \stri$ for all $\stri \subseteq
					\tris$. This is satisfied by \terone{}, \tertwo{}, \terthree{},
					\terfour{} and \terfive{}.
				\item[\tnotelabelf]
					Results on this line hold only if $\rba, \rbu, \rbv$ are
					non-disjunctive programs.
			\end{tablenotes}
		\end{threeparttable}
	\end{table}
\end{landscape}

\begin{description}
	\setlength{\itemsep}{0pt}
	\item[\pup{Tautology} and \pup{Immunity to Tautologies}:] These are
		naturally satisfied by all simple update operators that do not introduce
		exceptions merely due to the presence of a tautological rule in the
		updating program. In particular, both properties are satisfied by
		\terone-, \tertwo-, \terthree-, \terfour- and \terfive-based operators.
		Note that these properties are generally acknowledged as very desirable
		although most existing rule update semantics fail to comply with them
		\cite{Leite1997,Alferes2000,Eiter2002,Sakama2003,Zhang2006,Osorio2007,Delgrande2007}.

	\item[\pup{Non-interference}:]
		This property is satisfied by \terone-, \tertwo-, \terthree-, \terfour-
		and \terfive-based rule update operators. However, this is only the case
		when non-disjunctive programs are considered, pointing towards one of the
		important open problems faced by state-of-the-art research on rule
		updates: examples, desirable properties and methods for updating
		\emph{disjunctive} programs. Insights in this direction should shed light
		on whether \pup{Non-interference} is desirable in the disjunctive case.

	\item[\pup{Idempotence}, \pup{Absorption} and \pup{Augmentation}:] These are
		the only properties that reveal differences amongst the different
		exception functions. They are not satisfied by \terone-, \tertwo- and
		\terthree-based operators under \SR- and \RR-equivalence because a program
		updated by its subset may contain weakened versions of the original rules.
		Since such rules are not part of the original program, the programs before
		and after update are considered to be different under \SR- and
		\RR-equivalence. This problem is dodged in \terfour{} and \terfive{} by
		completely obliterating original rules that appear in the update.

		This also indicates that \SR- and \RR-equivalence are slightly \emph{too
		strong} for characterising updates because programs such as $\set{\atma.}$
		and $\set{\atma., \atma \lpif \atmb.}$ are not considered equivalent even
		though we expect the same behaviour from them when they are updated. We
		speculated in Section~\ref{sec:robust equivalence models} that it might be
		possible to address this issue by adopting the weaker \SMR- or
		\RMR-equivalence. However, it turns out that due to the monotonic nature
		of \SE- and \RE-models, these equivalence relations are \emph{too weak}:
		programs such as $\set{\atma.}$ and $\set{\atma., \atmb \lpif \lpnot
		\atma.}$ are \SMR- and \RMR-equivalent but, when updated by $\set{\lpnot
		\atma.}$, any rule update operator respecting fact update and causal
		rejection must provide the stable model $\emptyset$ in the former case,
		which violates causal rejection in the latter case.
		
		Moreover, \terone-based operators fail to satisfy \pup{Absorption} and
		\pup{Augmentation}. Along with Example~\ref{ex:terone iterated}, this
		indicates that \terone{} does not correctly handle iterated updates.

	\item[\pup{Associativity}:]
		This is one of the few properties that is not satisfied by any of the
		defined classes of operators. It is closely related to the question of
		whether \emph{rejected rules are allowed to reject}. \pup{Associativity}
		can be seen as postulating that an update operator must behave the same
		way regardless of whether rejected rules are allowed to reject or not. As
		witnessed by the \JU- and \AS-semantics (c.f.\ equation~\ref{eq:ru:problem
		with as}), rule update semantics tend to generate unwanted models when
		rejected rules are not allowed to reject.
\end{description}

\subsubsection*{Belief Update Postulates Reformulated}

The lower part of Table~\ref{tbl:semantic properties} contains a reformulation
of the belief update postulates \bu{1} -- \bu{6} for rule bases. We omit the
last two postulates as they require program disjunction and it is not clear
how to obtain it appropriately. Note also that \bu{7} has been heavily
criticised in the literature as being mainly a means to achieve formal results
instead of an intuitive principle \citep{Herzig1999} and though \bu{8}
reflects the basic intuition behind belief update -- that of updating each
model independently of the others -- such a point of view is hardly
transferable to knowledge represented using rules because a single model, be
it a classical, stable, \SE- or \RE-model, fails to encode the
interdependencies between literals expressed in rules that are necessary for
properties such as support.

Since we did not define \SM-entailment, postulates that refer to it have the
\SM{} column marked as ``\na''. Now we turn to the individual postulates.

\begin{description}
	\setlength{\itemsep}{0pt}
	\item[\pu{1} and \putwotop:]
		Similarly as \pup{Initialisation} and \pup{Disjointness}, these postulates
		are satisfied by any simple rule update operator and under all notions of
		equivalence.

	\item[\pu{2.1} and \pu{5}:]
		Postulate \pu{2.1} is not satisfied under \SR- and \RR-equivalence for the same
		reasons, described above, that prevent \terone- and \tertwo-based
		operators from satisfying \pup{Idempotence}. The situation with \pu{5} is
		the same since it implies \pu{2.1} in the presence of \putwotop{}
		\cite{Herzig1999}.

	\item[\pu{2.2} and \pu{6}:]
		Postulate \pu{2.2} requires that $\set{\atma., \lpnot \atma.} \uopr
		\set{\lpnot \atma.} \ent \atma$ which, in the presence of \pu{1}, amounts
		to postulating that one can never recover from an inconsistent state,
		contrary to most existing rule update semantics which do allow for
		recovery from such states. The case of \pu{6} is the same since it implies
		\pu{2.2} in the presence of \pu{1} and \putwotop{} \cite{Herzig1999}.

	\item[\pu{3}:]
		This postulate relies on a function that returns the set of models of a
		rule base. Thus, $\modsm{\cdot}$, $\modse{\cdot}$ and $\modre{\cdot}$ can
		be used for this purpose and the other columns in the corresponding row in
		Table~\ref{tbl:semantic properties} make little sense, so they are marked
		as ``\na''. Furthermore, this postulate is not satisfied by any of the
		defined classes of exception-based operators. It is also one of the
		principles that most existing approaches to rule update chronically fail
		to satisfy. In order to satisfy it, a context-aware exception function
		would have to be used because conflicts may arise in a set of more than
		two rules that are pairwise consistent. For instance, when updating
		$\set{\atma.}$ by $\set{\atmb \lpif \atma., \lpnot \atmb \lpif \atma.}$,
		one would somehow need to detect and resolve the joint conflict between
		these three rules. This is however impossible with a simple exception
		function because it only considers conflicts between pairs of rules, one
		from the original program and one from the update.

	\item[\pu{4}:]
		\label{pstl:pua}
		This postulate requires update operators to be syntax-independent. In this
		context it is useful to consider the following weaker principles:
		\begin{align}
			&\labpu{4.1}
				\quad
				\text{%
					If $\prga \eq \prgb$, then $\prga \uopb \prgu \eq \prgb \uopb \prgu$.
				}
			&&&&&& \labpu{4.2}
				\quad
				\text{%
					If $\prgu \eq \prgv$, then $\prga \uopb \prgu \eq \prga \uopb \prgv$.
				}
				\notag
		\end{align}
		The failure to satisfy \pua{4.1} under \SM-, \SE-, \RE-, \SMR- and
		\RMR-equivalence is inevitable if properties such as support, fact update
		and causal rejection are to be respected (c.f.\ \cite{Slota2013a} and the
		above discussion of \pup{Idempotence}, \pup{Absorption} and
		\pup{Augmentation}). Furthermore, \pua{4.1} is also violated under
		\SR-equivalence due to the fact that a constraint such as $(\lpif \atma.)$
		cannot be weakened by the introduced exception functions while the fact
		$(\lpnot \atma.)$ can, although it is strongly equivalent to the
		constraint.

		Similar arguments can be used to show that the principle \pua{4.2} is not
		satisfied under \SM-, \SE-, \RE, \SMR- and \RMR-equivalence. We only need
		to observe that any \terone-, \tertwo-, \terthree-, \terfour- or
		\terfive-based operator $\uope$ satisfies $\emptyset \uope \prga \eqRR
		\prga$, so each example used to show that \pua{4.1} is not desirable,
		involving two updates $\prga_1 \uope \prgu$ and $\prga_2 \uope \prgu$, can
		be reused to show the same for \pua{4.2} by instead considering the
		updates $\emptyset \uope \prga_1 \uope \prgu$ and $\emptyset \uope \prga_1
		\uope \prgu$. Additionally, \pua{4.2} is not satisfied under
		\SR-equivalence due to the fact that updates such as $\set{\lpnot \atma
		\lpif \atmb.}$, $\set{\lpnot \atmb \lpif \atma.}$ and $\set{\lpif \atma,
		\atmb.}$ have different effects on the program $\set{\atma., \atmb.}$.

		These observations indicate that \pua{4.1} and \pua{4.2},
		and thus also \pu{4}, are too strong under \SM-, \SE-, \RE-, \SMR- and
		\RMR-equivalence. Under \SR-equivalence, they are incompatible with
		operators that solve conflicts based on the heads of rules. On the other
		hand, due to the semantic underpinning of simple rule update operators,
		\pu{4} is satisfied by most of them, including all \terone-, \tertwo-,
		\terthree-, \terfour- and \terfive-based ones, under \RR-equivalence.
\end{description}

\subsubsection*{Summary}

The results in Table~\ref{tbl:semantic properties} indicate that simple
rule update operators satisfy a number of essential semantic properties by
design. This is especially important for properties such as \pup{Tautology},
\pup{Immunity to Tautologies} and \putwotop, generally acknowledged as very
desirable although most existing rule update semantics fail to comply with
them
\cite{Leite1997,Alferes2000,Eiter2002,Sakama2003,Zhang2006,Osorio2007,Delgrande2007}.

Some of the semantic properties, namely \pup{Idempotence}, \pup{Absorption}
and \pup{Augmentation}, outline the main differences between the
particular local exception functions that we have defined. Although \terone is
the function with the simplest definition, it does not satisfy any of these
properties, indicating that it cannot deal well with iterated updates.

Our results also show that a number of KM postulates are not satisfied by the
defined rule update operators. Coincidentally, with the exception of \bu{3},
the violated postulates are the ones that are deemed controversial in the
belief update community
\cite{Brewka1993,Boutilier1995,Doherty1998,Herzig1999}.

Perhaps most importantly, Table~\ref{tbl:semantic properties} introduces a
novel way of viewing the various semantic properties, where each property is
parametrised by a particular notion of program equivalence and/or entailment.
This richer view should prove useful in further research on rule update
operators and their semantic properties.

\section{State Condensing}

\label{sec:state condensing}

Our results about semantic rule update operators bring along a new point of
view on traditional approaches to rule updates. Particularly interesting is
the fact that a semantic rule update operator, following the belief update
tradition, is a binary function on the set of programs over the same alphabet.
This guarantees that after an update, the original program is \emph{replaced}
by a new program which continues to be used in its place, and the process can
be easily iterated if the need arises. Even though this perspective is very
natural, even fundamental, it has been largely neglected in the ample body of
work on rule updates. Existing semantics typically proceed by characterising
the models of the update and, at most, either describe a set of answer-set
programs that could represent the update, instead of only one, or produce an
answer-set program written in a language extended with a considerable amount
of new atoms, making it difficult to understand and to further update.

For instance, semantics based on causal rejection
\cite{Leite1997,Alferes2000,Eiter2002,Alferes2005,Osorio2007}, such as the
\JU- and \AS-semantics, assign models to sequences of non-disjunctive programs
and sometimes admit non-minimal models which no non-disjunctive program can
capture. Consequently, they must resort to the introduction of additional
meta-level atoms in order to construct a single program whose stable models
correspond to the models assigned to the sequence, leading to difficulties
with iterating the update process. Producing a program from a more expressive
class of logic programs would result in the inability to perform another
update as well because the update semantics are not defined for such programs.

A different approach in \cite{Sakama2003} deals with program updates by
borrowing ideas from literature on \emph{belief revision} and utilising an
abductive framework to accomplish such updates. In this case, multiple
alternative programs can be the result of an update and no mechanism is
provided to choose among them.

A somewhat similar situation occurs with the approach of \cite{Zhang2006}
where intricate syntactic transformations are combined with a semantics for
prioritised logic programs that ultimately leads to a set of logic programs.
Since all of these programs together represent the result of the update, it is
once again unclear how to construct a single program that combines all of
them.

The rule update semantics suggested in \cite{Delgrande2007} are also based on
syntactic transformations into a logic program with preferences among rules,
but in contrast with \cite{Zhang2006}, the semantics of such programs is
defined by directly specifying their preferred stable models and not by
translation into an ordinary program (or a set thereof). Thus, although an
actual syntactic object is constructed that represents the update, it needs to
be interpreted in a richer formalism to take into account preferences among
rules.

Finally, frameworks that specify program updates by manipulating dependencies
on default assumptions induced by rules
\cite{Sefranek2006,Sefranek2011,Krumpelmann2012} are mainly concerned with
identifying the effects of irrelevant updates and other theoretical properties
of the stable models assigned to a pair or sequence of programs. They do not
consider representing the result of an update by a single program.

In this section, we unravel the true potential of specifying updates as binary
operators on some class of programs. Despite the fact that existing program
update semantics do not seem compatible with this point of view, we show that
at least the foundational \JU- and \AS-semantics \emph{can} be viewed in this
manner.

Theorem~\ref{thm:er vs ju and as} already provides part of the solution as it
shows that any DLP can be \emph{condensed} into a \emph{single rule base} over
the same alphabet that behaves just as the original DLP when further updates
are performed on it. Formally, this can be stated as follows:

\begin{corollary}
	[State Condensing into a Rule Base]
	\label{cor:state condensing}
	There exist rule update operators $\uopr_1$ and $\uopr_2$ such that for
	every DLP $\dprg = \seq{\prg_\lia}_{\lia < \lng}$ without local cycles and
	all $\lib < \lng$ there exist rule bases $\rb_1$ and $\rb_2$ with
	\begin{align*}
		\modju{\dprg}
			&=
			\modsm{
				\textstyle \biguopr_1
				\seq{\rb_1, \prg_{\lib + 1}, \dotsc, \prg_{\lng - 1}}
			}
		&& \text{and}
		& \modas{\dprg}
			=
			\modsm{
				\textstyle \biguopr_2
				\seq{\rb_2, \prg_{\lib + 1}, \dotsc, \prg_{\lng - 1}}
			}
		\enspace.
	\end{align*}
\end{corollary}
\begin{proof}
	See \ref{app:exception-based rule update}, page \pageref{proof:cor:state
	condensing}.
\end{proof}

However, this result is not completely satisfactory since it relies on rule
bases instead of standard classes of logic programs, does not provide explicit
definitions of operators $\uopr_1$ and $\uopr_2$, and does not apply to DLPs
with local cycles.

In the following, we address all these issues by defining specific rule update
operators that faithfully characterise the \JU- and \AS-semantics -- when
applied to an arbitrary DLP $\dprg$, they produce a program whose stable
models coincide with \JU- and \AS-models of $\dprg$, respectively. In this
way, the new operators provide a way to condense any DLP into a single program
that includes all relevant information about the DLP, not only to identify its
stable models, but also for the purpose of performing further updates.
Thereby, we solve the long-standing problem known as \emph{state condensing}
from the literature on causal rejection semantics for rule updates.

To achieve this, our operators must deal with a more general class of programs
than non-disjunctive ones. First we define simple and elegant operators on
subclasses of \emph{nested logic programs} \cite{Turner2003} with the required
property. Subsequently, we show that the full expressivity of nested programs
is not necessary for this purpose by defining an additional pair of operators
that produce \emph{disjunctive logic programs} and still maintain the same
properties w.r.t.\ \JU- and \AS-semantics.

Throughout this section we assume that all programs are finite. We also assume
that all non-disjunctive rules $\rl$, originating in some DLP, have exactly
one literal $\lit$ in their heads. This latter assumption does not remove any
significant expressivity from DLPs under the \JU- and \AS-semantics since each
constraint $(\lpif \brl.)$ can be equivalently encoded as the rule $(\atm_\bot
\lpif \lpnot \atm_\bot, \brl.)$ where $\atm_\bot$ is a fresh atom designated
for this purpose.

\subsection{Programs with Nested Expressions}

We adopt the syntax and stable model semantics of logic programs with nested
expressions \cite{Turner2003}. A \emph{nested rule} is an expression $\rl$ of
the form $(\hrl \lpif \brl.)$ where $\hrl$ and $\brl$, called the \emph{head}
and \emph{body} of $\rl$, are expressions built inductively from the
propositional atoms in $\atms$ and the 0-place connectives $\bot$ and $\top$
using the unary connective $\lpnot$ (default negation) and the binary
connectives $\land$ and $\lor$.\footnote{%
	Within this section, we refer to such expressions simply as \emph{formulas}
	even though they are different from the propositional formulas introduced in
	Section~\ref{sec:background}.
}
A nested rule of the form $(\hrl \lpif \top.)$ is usually identified with the
formula $\hrl$. A \emph{nested program} is a finite set of nested rules. Each
disjunctive program $\prg$ and DLP $\dprg = \seq{\prg_\lia}_{\lia < \lng}$, as
defined in Section~\ref{sec:background}, is translated, respectively, to a
nested program and a sequence thereof as follows:
\begin{align*}
	\tones{\prg}
		&=
		\Set{ \biglor \hrl \lpif \bigland \brl. | \rl \in \prg}
	&& \text{and}
	& \tones{\dprg}
		&=
		\seq{\tones{\prg_\lia}}_{\lia < \lng}
		\enspace.
\end{align*}

Satisfaction of a formula $\frma$ in an interpretation $\twib$, denoted by
$\twib \ent \frma$, is defined in the usual way. Furthermore, $\twib$
satisfies a nested rule $\rl$, denoted by $\twib \ent \rl$, if $\twib \ent
\brl$ implies $\twib \ent \hrl$, and $\twib$ satisfies a nested program
$\prg$, denoted by $\twib \ent \prg$, if $\twib \ent \rl$ for all $\rl \in
\prg$. The \emph{reduct} of a formula $\frm$ relative to $\twib$, denoted by
$\frm^\twib$, is obtained by replacing, in $\frm$, every maximal occurrence of
a formula of the form $\lpnot \frmb$ with $\bot$ if $\twib \ent \frmb$ and
with $\top$ otherwise. The reducts of a nested rule $\rl$ and of a nested
program $\prg$ are, respectively, $\rl^\twib = \br{\hrl^\twib \lpif
\brl^\twib.}$ and $\prg^\twib = \set{\rl^\twib | \rl \in \prg}$. Finally, an
interpretation $\twib$ is a \emph{stable model} of a nested program $\prg$ if
it is subset-minimal among the interpretations that satisfy $\prg^\twib$.

\subsection{Condensing into a Nested Program}

\label{sec:state condensing:nested}

Now we can proceed with the definition of condensing operators $\uoprju$ and
$\uopras$ for the \JU- and \AS-semantics. The property that these operators
must fulfill is that for any DLP $\dprg$, the stable models of the nested
program resulting from applying the operators to $\dprg$ are exactly the \JU-
and \AS-models of $\dprg$, respectively.

Both $\uoprju$ and $\uopras$ are defined by utilising the concept of an
\emph{activation formula} which captures the condition under which a literal
$\lit$ is derived by some rule in a nested program $\prgu$. Formally, the
\emph{activation formula for $\lit$ in $\prgu$} is defined as follows:
\[
	\acond{\prgu}{\lit}
	=
	\biglor \Set{
		\bigland \brl
		|
		\rl \in \prgu \land \hrl = \lit
	}
	\enspace.\footnote{%
		Note that $\biglor \emptyset$ is simply $\bot$.
	}
\]
The operator $\uoprju$ is based on the following simple idea: When updating a
program $\prga$ by a program $\prgu$, each rule $\rl$ from $\expv{\prga}$ with
literal $\lit$ in its head must be disabled when some rule from $\expv{\prgu}$
for the literal complementary to $\lit$ is activated. This can be achieved by
augmenting the body of $\rl$ with the additional condition $\lpnot
\Acond{\expv{\prgu}}{\lcmp{\lit}}$. Formally:

\begin{definition}
	[Condensing Operator $\uoprju$]
	A \emph{\JU-rule} is a nested rule with a single literal in its head and a
	\emph{\JU-program} is a nested program that consists of \JU-rules.

	The binary operator $\uoprju$ on the set of all \JU-programs is defined as
	follows: Given two \JU-programs $\prga$ and $\prgu$, the \JU-program $\prga
	\uoprju \prgu$ consists of the following rules:
	\begin{textenum}[1.]
		\item for all $\rla \in \expv{\prga}$, the nested rule
			$
				\Br{
					\hrla
					\lpif
					\brla \land \lpnot \Acond{\expv{\prgu}}{\lcmp{\hrla}}.
				}
			$;

		\item all nested rules from $\expv{\prgu}$.
	\end{textenum}
\end{definition}

In case of the \AS-semantics, additional choice rules are needed.
Particularly, for every nested rule $\rl \in \tones{\all{\dprg}}$ whose head
is an atom, the update operator needs to include the nested rule $(\hrl \lor
\lcmp{\hrl} \lpif \brl.)$ in the resulting program. Intuitively, these
additional rules account for the differences in the definition of
$\rejju{\dprg, \twib}$ and $\rejas{\dprg, \twib}$ by making sure that no rule
is ever completely eliminated, but stays partially in effect by generating
alternative worlds for the atom in its head: one where it is satisfied and one
where it is not. Essentially, this means that whenever the original body of
the rule is satisfied, its head atom cannot be assumed false by default and is
interpreted ``classically'' instead.

\begin{definition}
	[Condensing Operator $\uopras$]
	A \emph{\AS-rule} is a nested rule with either a single literal or a
	disjunction of two literals $\lit$ and $\lcmp{\lit}$ in its head, and a
	\emph{\AS-program} is a nested program that consists of \AS-rules.

	The binary operator $\uopras$ on the set of all \AS-programs is defined as
	follows: Given two \AS-programs $\prga$ and $\prgu$, the \AS-program $\prga
	\uopras \prgu$ consists of the following rules:
	\begin{textenum}[1.]
		\item for all $\rla \in \expv{\prga}$ with $\hrla \in \lits$, the nested rule
			$
				\Br{
					\hrla
					\lpif
					\brla \land \lpnot \Acond{\expv{\prgu}}{\lcmp{\hrla}}.
				}
			$;

		\item all $\rla \in \prga$ such that $\hrla$ is of the form $\lit \lor
			\lcmp{\lit}$;

		\item for all $\rla \in \prgu$ with $\hrla \in \atms$, the nested rule
			$
				\Br{
					\hrla \lor \lcmp{\hrla}
					\lpif
					\brla.
				}
			$;

		\item all nested rules from $\expv{\prgu}$.
	\end{textenum}
\end{definition}

The following example illustrates the relationship between the \JU- and
\AS-semantics and the condensing operators $\uoprju$ and $\uopras$, while the
subsequent theorem establishes it in general.

\begin{example}
	\label{ex:uoprju and uopras}
	Suppose that programs $\prga$, $\prgu$ and $\prgv$, written as nested
	programs, are as follows:
	\begin{align*}
		\tones{\prga}: \quad
			\atma &\lpif \top.
		&
		\tones{\prgu}: \quad
			\lpnot \atma &\lpif \lpnot \atmb \land \lpnot \atmc.
		&
		\tones{\prgv}: \quad
			\atma &\lpif \atmd.
		\\
			\atmb &\lpif \atma.
			& \lpnot \atma &\lpif \atmd.
			& \atmc &\lpif \atmc.
		\\
			\atmc& \lpif \top.
			& \lpnot \atmc &\lpif \top.
			& \atmd &\lpif \top.
	\end{align*}
	In addition to the rules from $\expv{\tones{\prgu}}$, the program
	$\biguoprju \seq{\tones{\prga}, \tones{\prgu}}$ contains the following three
	nested rules:
	\begin{align}
		&\atma \lpif \top \land \lpnot ( (\lpnot \atmb \land \lpnot \atmc) \lor \atmd ).
		&
		&\atmb \lpif \atma \land \lpnot \! \bot.
		&
		&\atmc \lpif \top \land \lpnot \! \top.
		\label{eq:prga updated by prgu}
		\intertext{%
			Also, the program $\biguopras \seq{\tones{\prga}, \tones{\prgu}}$
			consists of the rules in $\biguoprju \seq{\tones{\prga}, \tones{\prgu}}$
			and of three additional choice rules. Note that these rules are not part
			of the program $\tones{\prga} \uopras \tones{\prgu}$. They belong to the
			program $\emptyset \uopras \tones{\prga}$, and, consequently, also to
			$(\emptyset \uopras \tones{\prga}) \uopras \tones{\prgu} = \biguopras
			\seq{\tones{\prga}, \tones{\prgu}}$. They are as follows:
		}
		&\atma \lor \lpnot \atma \lpif \top.
		& 
		&\atmb \lor \lpnot \atmb \lpif \atma.
		& 
		&\atmc \lor \lpnot \atmc \lpif \top.
		\label{eq:prga extra uopras rules}
	\end{align}
	Furthermore, both $\biguoprju \seq{\tones{\prga}, \tones{\prgu}}$ and
	$\biguopras \seq{\tones{\prga}, \tones{\prgu}}$ have the stable models
	$\emptyset$ and $\set{\atma, \atmb}$ which coincide with the \JU- and
	\AS-models of the DLP $\seq{\prga, \prgu}$.

	The situation is more interesting after $\tones{\prgv}$ is added to the
	update sequence. In addition to the rules from $\expv{\tones{\prgv}}$, the
	program $\biguoprju \seq{\tones{\prga}, \tones{\prgu}, \tones{\prgv}}$
	contains the following rules:
	\begin{align*}
		\atma &\lpif
			\top
			\land \lpnot ( (\lpnot \atmb \land \lpnot \atmc) \lor \atmd )
			\land \lpnot \! \bot.
		&
		\atmb &\lpif
			\atma
			\land \lpnot \! \bot
			\land \lpnot \! \bot.
		&
		\atmc &\lpif
			\top \land
			\lpnot \! \top
			\land \lpnot \! \bot.
		\\
		\lpnot \atma &\lpif
			\lpnot \atmb \land
			\lpnot \atmc \land
			\lpnot \atmd.
		&
		\lpnot \atma &\lpif
			\atmd
			\land \lpnot \atmd.
		&
		\lpnot \atmc &\lpif
			\top
			\land \lpnot \atmc.
	\end{align*}
	Also, the program $\biguopras \seq{\tones{\prga}, \tones{\prgu},
	\tones{\prgv}}$ consists of the rules in $\biguoprju \seq{\tones{\prga},
	\tones{\prgu}, \tones{\prgv}}$, the three choice rules listed in
	\eqref{eq:prga extra uopras rules} and, additionally, the following three
	choice rules rules originating in $\tones{\prgv}$:
	\begin{align*}
		\atma \lor \lpnot \atma &\lpif \atmd.
		&
		\atmc \lor \lpnot \atmc &\lpif \atmc.
		&
		\atmd \lor \lpnot \atmd &\lpif \top.
	\end{align*}

	Note that some body conjuncts, such as $\lpnot \! \bot$, and even whole
	rules, for instance $(\lpnot \atma \lpif \atmd \land \lpnot \atmd.)$, can
	can be eliminated from the resulting programs without affecting their stable
	models or the stable models resulting from further updates. Slightly less
	elegant definitions of $\uoprju$ and $\uopras$ could be used to perform such
	simplifications automatically. For illustration, Figure~\ref{fig:ju and as}
	lists the simplified versions of the nested programs $\biguoprju
	\seq{\tones{\prga}, \tones{\prgu}, \tones{\prgv}}$ and $\biguopras
	\seq{\tones{\prga}, \tones{\prgu}, \tones{\prgv}}$. It is also not
	difficult to verify that the unique stable model of the former program is
	$\set{\atma, \atmb, \atmd}$, which coincides with the unique \JU-model of
	$\seq{\prga, \prgu, \prgv}$. On the other hand, due to the rule $(\atmc \lor
	\lpnot \atmc \lpif \top.)$, the latter program has the additional stable
	model $\set{\atma, \atmb, \atmc, \atmd}$, which coincides with the
	additional \AS-model of the DLP $\seq{\prga, \prgu, \prgv}$.
\end{example}

\begin{figure}[t]
	\small
	\begin{minipage}[t]{0.475\textwidth}
		\hrule
		\vspace{-1em}
		\begin{align*}
			\atma &\lpif
				\lpnot ( (\lpnot \atmb \land \lpnot \atmc) \lor \atmd ).
			&
			\atmb &\lpif
				\atma.
			\\
			\lpnot \atma &\lpif
				\lpnot \atmb \land
				\lpnot \atmc \land
				\lpnot \atmd.
			\\
				\atma &\lpif \atmd.
				& \atmd &\lpif \top.
			\\
			\\
		\end{align*}
		\vspace{-1.5em}
		\hrule
	\end{minipage}
	\hspace{0.05\textwidth}
	\begin{minipage}[t]{0.475\textwidth}
		\hrule
		\vspace{-1em}
		\begin{align*}
			\atma &\lpif
				\lpnot ( (\lpnot \atmb \land \lpnot \atmc) \lor \atmd ).
			&
			\atmb &\lpif
				\atma.
			\\
			\lpnot \atma &\lpif
				\lpnot \atmb \land
				\lpnot \atmc \land
				\lpnot \atmd.
			\\
				\atma &\lpif \atmd.
				& \atmd &\lpif \top.
			\\
			\atma \lor \lpnot \atma &\lpif \top.
			& 
			\atmb \lor \lpnot \atmb &\lpif \atma.
			\\
			\atmc \lor \lpnot \atmc &\lpif \top.
			&
			\atmd \lor \lpnot \atmd &\lpif \top.
		\end{align*}
		\vspace{-1.5em}
		\hrule
	\end{minipage}
	\caption{%
		Nested programs $\biguoprju \seq{\tones{\prga}, \tones{\prgu},
		\tones{\prgv}}$ (left) and $\biguopras \seq{\tones{\prga}, \tones{\prgu},
		\tones{\prgv}}$ (right) without irrelevant rules and body conjuncts.
	}
	\label{fig:ju and as}
\end{figure}

\begin{theorem}
	[State Condensing Using $\uoprju$ and $\uopras$]
	\label{thm:ju binary operator}
	Let $\dprg$ be a DLP. An interpretation $\twib$ is a
	\begin{textenum}[(i)]
		\item \JU-model of $\dprg$ if and only if it is a stable model of
			$\biguoprju \tones{\dprg}$;

		\item \AS-model of $\dprg$ if and only if it is a stable model of
			$\biguopras \tones{\dprg}$.
	\end{textenum}
\end{theorem}
\begin{proof}
	See \ref{app:proofs}, page \pageref{proof:thm:ju binary operator}.
\end{proof}

The significance of this theorem is in that it demonstrates that the operators
$\uoprju$ and $\uopras$ provide a full characterisation of the \JU- and
\AS-semantics: They condense any DLP $\dprg$ into a single nested program
$\prg$ such that the stable models of $\prg$ coincide with the \JU- and
\AS-models of $\dprg$, respectively, and any further updates of $\dprg$ can be
equivalently performed directly on $\prg$ using the operators.

Note that since the operators manipulate rules on the syntactic level, they
are completely faithful to \JU- and \AS-semantics, even when the DLP contains
local cycles (see also Theorem~\ref{thm:er vs ju and as} where local cycles
form a special case due to the fact that exception functions manipulate rules
on the semantic level).

Interestingly, an operator very similar to $\uoprju$ has been studied by
Osorio and Cuevas \cite{Osorio2007}. They have proven that it characterises
the \AS-semantics for the case of a single update but did not consider the
possibility of performing iterated updates. Their result is also a consequence
of Theorem~\ref{thm:ju binary operator} and of the fact that the \JU- and
\AS-semantics provide the same result when only a single update is performed.

\subsection{Condensing into a Disjunctive Program}

\label{sec:state condensing:disjunctive}

The condensing operators defined in the previous section can be further
refined in order to produce a program that meets additional requirements. In
the present section we show that nested expressions can be completely
eliminated from the resulting program while still preserving the same tight
relationship with the original rule update semantics. Thus, we introduce an
additional pair of operators, $\uoprjulor$ and $\uopraslor$, that operate on
disjunctive programs with default negation in heads of rules. Note that due to
the non-minimality of \JU- and \AS-models for certain DLPs, disjunctive
programs \emph{without} default negation in heads of rules would already be
insufficient for this purpose.

The ideas underlying the new operators are fairly straightforward.
Essentially, nested expressions are introduced into the resulting programs
only by the negations of activation formulas in their bodies, so these are the
parts of rules that need to be translated into conjunctions in bodies and
disjunctions in heads of rules. In particular, by utilising De Morgan's law
and distributivity of conjunction over disjunction, we can obtain a new
formula, a disjunction of conjunctions of default literals and double-negated
atoms, that is strongly equivalent to the original formula. For instance, in
case of the first rule in \eqref{eq:prga updated by prgu}, we can equivalently
write the condition $\lpnot((\lpnot \atmb \land \lpnot \atmc) \lor \atmd)$ as
$
	(\lpnot \lpnot \atmb \land \lpnot \atmd)
	\lor
	(\lpnot \lpnot \atmc \land \lpnot \atmd)
$.
Then it suffices to break up the resulting rule into multiple rules, each with
one of the disjuncts of this formula in the body, and remove one of the
negations from each double-negated literals and ``move'' it into the head of
the newly constructed rule. In case of the first rule in \eqref{eq:prga
updated by prgu}, the result would be the disjunctive rules
\begin{align}
	\label{eq:disj rules}
	\atma \lor \lpnot \atmb &\lpif \lpnot \atmd.
	&& \text{and}
	& \atma \lor \lpnot \atmc &\lpif \lpnot \atmd.
\end{align}

We call each set of literals, without double negation, within each of the
disjuncts described above a \emph{blocking set}. Formally, if the activation
formula $\acond{\tones{\prgu}}{\lit}$ contains $\top$ as one of its disjuncts,
then there is no blocking set for $\lit$ in $\prgu$. Otherwise, suppose that
$\acond{\tones{\prgu}}{\lit} = (\lit_1^1 \land \dotsb \land \lit_{\lic_1}^1)
\lor \dotsb \lor (\lit_1^\lng \land \dotsb \land \lit_{\lic_\lng}^\lng)$. A
\emph{blocking set for $\lit$ in $\prgu$} is any set of literals
$\Set{\lcmp{\lit_{\lia_1}^1}, \dotsc, \lcmp{\lit_{\lia_\lng}^\lng}}$ where $1
\leq \lia_\lib \leq \lic_\lib$ for every $\lib$ with $1 \leq \lib \leq \lng$.
We denote the set of all blocking sets for $\lit$ in $\prgu$ by
$\sblks{\prgu}{\lit}$.

Each nested rule $\Br{\hrl \lpif \brl \land \lpnot
\Acond{\tones{\expv{\prgu}}}{\lcmp{\hrl}}.}$ can thus be replaced by the set
of disjunctive rules
\[
	\Set{
		\hrl; \lcmp{\sblk^+} \lpif \brl, \lcmp{\sblk^-}.
		|
		\sblk \in \Sblks{\expv{\prgu}}{\lcmp{\hrl}}
	} \enspace.
\]
Furthermore, when $\hrl$ is a default literal, it is more convenient to move
the new default literals from the head into the body since this operation
preserves stable models \cite{Inoue1998} and makes it easier to pinpoint the
original head literal in the rule. This leads us to the following definition
of $\uoprjulor$:

\begin{definition}
	[Condensing Operator $\uoprjulor$]
	Given a disjunctive rule $\rl$, a literal $\lit$ is the \emph{original head
	of $\rl$} if either $\lit \in \atms$ and $\hrl \cap \atms = \set{\lit}$, or
	$\lit \notin \atms$ and $\hrl = \set{\lit}$.

	The binary operator $\uoprjulor$ on the set of all disjunctive programs is
	defined as follows: Given two disjunctive programs $\prga$ and $\prgu$, the
	disjunctive program $\prga \uoprjulor \prgu$ consists of the following
	rules:
	\begin{textenum}[1.]
		\item for all $\rla \in \expv{\prga}$ with original head $\lit \in \atms$
			and all $\sblk \in \sblks{\expv{\prgu}}{\lcmp{\lit}}$, the rule
			$\Br{\hrla; \lcmp{\sblk^+} \lpif \brla, \lcmp{\sblk^-}.}$;

		\item for all $\rla \in \expv{\prga}$ with original head $\lit \notin
			\atms$ and all $\sblk \in \sblks{\expv{\prgu}}{\lcmp{\lit}}$, the rule
			$\Br{\hrla \lpif \brla, \sblk^+, \lcmp{\sblk^-}.}$;

		\item all rules from $\expv{\prgu}$.
	\end{textenum}
\end{definition}

As for the \AS-semantics, similar modifications can be applied in the
definition of $\uopras$ to obtain an operator that produces a disjunctive
program. Furthermore, due to the additional choice rules included in the
result, the rules can be further simplified, when compared to the rules
produced by $\uoprjulor$. In particular, the first group of rules can be
treated the same way as the second, leading to the following definition of
$\uopraslor$:

\begin{definition}
	[Condensing Operator $\uopraslor$]
	The binary operator $\uopraslor$ on the set of all disjunctive programs is
	defined as follows: Given two disjunctive programs $\prga$ and $\prgu$, the
	disjunctive program $\prga \uopraslor \prgu$ consists of the following
	rules:
	\begin{textenum}[1.]
		\item for all $\rla \in \expv{\prga}$ with $\hrla = \set{\lit}$ and all
			$\sblk \in \Sblks{\expv{\prgu}}{\lcmp{\lit}}$, the rule $\Br{\hrla \lpif
			\brla, \sblk^+, \lcmp{\sblk^-}.}$;

		\item all $\rla \in \prga$ such that $\hrla$ is of the form $\set{\atm,
			\lpnot \atm}$ for some $\atm \in \atms$;

		\item for all $\rla \in \prgu$ with $\hrla = \set{\atm}$ for some $\atm
			\in \atms$, the rule $\Br{\atm; \lpnot \atm \lpif \brla.}$;

		\item all rules from $\expv{\prgu}$.
	\end{textenum}
\end{definition}

If we consider the programs $\prga$, $\prgu$ and $\prgv$ from
Example~\ref{ex:uoprju and uopras}, the two main differences between the
disjunctive program $\biguoprjulor \seq{\prga, \prgu, \prgv}$ and its nested
counterpart $\biguoprju \seq{\prga, \prgu, \prgv}$, besides their different
syntax, are as follows: 1) the nested rule is turned into the two disjunctive
rules listed in \eqref{eq:disj rules}, and 2) the rule $\atmc \lpif \top \land
\lpnot \top \land \lpnot \bot$ has no counterpart in the disjunctive program
as there is no blocking set for $\lpnot \atmc$ in $\prgu$. The differences
between $\biguopraslor \seq{\prga, \prgu, \prgv}$ and $\biguopras \seq{\prga,
\prgu, \prgv}$ are analogical and in Figure~\ref{fig:julor and aslor} we list
the disjunctive programs simplified in the same way as their nested counterparts
in Figure~\ref{fig:ju and as}.

\begin{figure}[t]
	\small
	\begin{minipage}[t]{0.475\textwidth}
		\hrule
		\vspace{-1em}
		\begin{align*}
			\atma; \lpnot \atmb &\lpif
				\lpnot \atmd.
			&
			\atmb &\lpif
				\atma.
			\\
			\atma; \lpnot \atmc &\lpif
				\lpnot \atmd.
			\\
			\lpnot \atma &\lpif
				\lpnot \atmb,
				\lpnot \atmc,
				\lpnot \atmd.
			\\
				\atma &\lpif \atmd.
				& \atmd &.
			\\
			\\
		\end{align*}
		\vspace{-1.5em}
		\hrule
	\end{minipage}
	\hspace{0.05\textwidth}
	\begin{minipage}[t]{0.475\textwidth}
		\hrule
		\vspace{-1em}
		\begin{align*}
			\atma &\lpif
				\atmb, \lpnot \atmd.
			&
			\atmb &\lpif
				\atma.
			\\
			\atma &\lpif
				\atmc, \lpnot \atmd.
			\\
			\lpnot \atma &\lpif
				\lpnot \atmb,
				\lpnot \atmc,
				\lpnot \atmd.
			\\
				\atma &\lpif \atmd.
				& \atmd &.
			\\
			\atma; \lpnot \atma &.
			& 
			\atmb; \lpnot \atmb &\lpif \atma.
			\\
			\atmc; \lpnot \atmc &.
			&
			\atmd; \lpnot \atmd &.
		\end{align*}
		\vspace{-1.5em}
		\hrule
	\end{minipage}
	\caption{%
		Disjunctive programs $\biguoprju^{\lor} \seq{\prga, \prgu, \prgv}$ (left)
		and $\biguopras^{\lor} \seq{\prga, \prgu, \prgv}$ (right) without
		irrelevant rules.
	}
	\label{fig:julor and aslor}
\end{figure}

From a general perspective, the operators $\uoprjulor$ and $\uopraslor$
preserve the main property of $\uoprju$ and $\uopras$.

\begin{theorem}
	[State Condensing Using $\uoprjulor$ and $\uopraslor$]
	\label{thm:ju disjunctive binary operator}
	Let $\dprg$ be a DLP. An interpretation $\twib$ is a
	\begin{textenum}[(i)]
		\item \JU-model of $\dprg$ if and only if it is a stable model of
			$\biguoprjulor \dprg$;

		\item \AS-model of $\dprg$ if and only if it is a stable model of
			$\biguopraslor \dprg$.
	\end{textenum}
\end{theorem}
\begin{proof}
	See \ref{app:proofs}, page \pageref{proof:thm:ju disjunctive binary
	operator}.
\end{proof}

Although operators $\uoprjulor$ and $\uopraslor$ eliminate the necessity for
using nested rules to condense a DLP into a single program, this comes at a
cost. Namely, the size of the nested program resulting from applying operators
$\uoprju$ and $\uopras$ is always linear in size of the argument programs,
while in case of $\uoprjulor$ and $\uopraslor$, the resulting program can be
exponentially larger. Furthermore, Figures~\ref{fig:ju and as} and
\ref{fig:julor and aslor} suggest that the representations produced by
$\uoprjulor$ and $\uopraslor$ will be less faithful to the form of the rules
in the original programs, and thus less readable. This indicates that the
nested program is more suitable as a way to store the condensed program, both
in terms of space and readability. Additionally, in order to find its stable
models, a more efficient translation can be used that utilises additional
meta-level atoms to prevent the exponential explosion (see
\cite{Leite1997,Eiter2002} for further details). However, such a translation
will no longer be equivalent to the original program sequence w.r.t.\
performing further updates.

\section{Abstract Exception-Based Operators}

\label{sec:exception-based-updates:abstract}

In this section we generalise exception-based operators to arbitrary knowledge
representation formalisms with monotonic model-theoretic semantics.

Thus, we assume to be using some knowledge representation formalism in which a
\emph{knowledge base} is a subset of the set of all \emph{knowledge atoms}
$\gfrms$ and $\sems$ denotes the set of all \emph{semantic structures} among
which the \emph{models} of knowledge atoms are chosen. The set of models of a
knowledge atom $\gfrm$ is denoted by $\mod{\gfrm}$. The \emph{semantic
characterisation} of a knowledge base $\gkb$ is the \emph{set of sets of
models} of its knowledge atoms: $\modr{\gkb} = \Set{\mod{\gfrm} | \gfrm \in
\gkb}$. The models of $\gkb$ are the models of all its elements, i.e.\
$\mod{\gkb} = \bigcap \modr{\gkb}$.

An abstract exception-based update operator is characterised by an
\emph{exception function} that, given the set of models of a knowledge atom
$\gfrm$ and the semantic characterisations of the original and updating
knowledge base, returns the semantic structures that are to be introduced as
exceptions to $\gfrm$.

\begin{definition}
	[Exception Function]
	\label{def:exception function}
	An \emph{exception function} is any function
	$
		\e : \pws{\sems} \times \pws{\pws{\sems}} \times \pws{\pws{\sems}}
		\rightarrow \pws{\sems}
	$.
\end{definition}

Given such an exception function and knowledge bases $\gkba$ and $\gkbu$, it
naturally follows that the semantic characterisation resulting from updating
$\gkba$ by $\gkbu$ should consist of sets of models of each knowledge atom
$\gfrm$ from $\gkba$, each augmented with the respective exceptions, and also
the unmodified sets of models of knowledge atoms from $\gkbu$. In other words,
we obtain the set of sets of models
\begin{equation}
	\label{eq:exception-based update}
	\Set{
		\mod{\gfrm} \cup \e(\mod{\gfrm}, \modr{\gkba}, \modr{\gkbu})
		|
		\gfrm \in \gkba
	}
	\cup \modr{\gkbu}
	\enspace.
\end{equation}

Turning to the syntactic side, an \emph{update operator} is binary function
over $\pws{\gfrms}$ that takes the original knowledge base and its update as
inputs and returns the updated knowledge base. An \emph{exception-based
update operator} is then formalised as follows:

\begin{definition}
	[Abstract Exception-Based Update Operator]
	\label{def:exception-based operator}
	Given an exception function \te{}, an update operator $\uope$ is
	\emph{\te-based} if for all $\gkba$, $\gkbu \subseteq \gfrms$,
	$\modrer{\gkba \uope \gkbu}$ is equal to \eqref{eq:exception-based update}.
	Also, $\uope$ is \emph{exception-based} if it is \te-based for some
	exception function \te{}.
\end{definition}

\section{Belief Updates Using Exception-Based Operators}

\label{sec:exception-based-updates:belief updates}

Concrete exception-based operators for propositional knowledge bases are
obtained from the abstract framework presented in the previous section by
identifying the set of knowledge atoms $\gfrms$ with the set of all
propositional formulas and the set of semantic structures $\sems$ with
propositional interpretations.

This setup allows us to immediately prove that most conceivable model-based
operators can be faithfully modelled as exception-based ones. Particularly,
this is possible for any update operator satisfying the uncontroversial
postulates \bu{1}, \bu{2.1} and \bu{4} \cite{Herzig1999}.

\begin{theorem}
	[Model-Based Updates Using Exception-Based Operators]
	\label{thm:belupd:model-based by exception-based}
	If $\uopb$ is a belief update operator that satisfies \bu{1}, \bu{2.1} and
	\bu{4}, then there exists an exception function \te{} such that for every
	\te-based update operator $\uope$ and all finite sequences of knowledge
	bases $\dkb$, $\mod{\biguopb \dkb} = \mod{\biguope \dkb}$.
\end{theorem}
\begin{proof}
	See \ref{app:belupd}, page~\pageref{proof:belupd:model-based by
	exception-based}.
\end{proof}

An interesting point regarding this result is that the principles \bu{1},
\bu{2.1} and \bu{4} are not specific to update operators, they are also
satisfied by \emph{AGM revision} operators. These operators are developed for
the case of revising a \emph{belief set} which is a set of formulas closed
w.r.t.\ a logical consequence operator $\cn$. A revision operator~$\uoprev$
takes an original belief set $\thr$ and a formula $\frmu$ representing its
revision and produces the revised belief set $\thr \uoprev \frmu$. The typical
properties satisfied by AGM revision operators include \emph{success},
\emph{inclusion} and \emph{extensionality} \cite{Hansson1993}, formalised,
respectively, as
\begin{align*}
	& \frmu \in \thr \uoprev \frmu \enspace,
	&& \thr \uoprev \frmu \subseteq \cn(\thr \cup \set{\frmu}) \enspace,
	&& \text{%
		If $\frmu \eq \frmv$, then $\thr \uoprev \frmu = \thr \uoprev \frmv$.
	}
\end{align*}
These three properties directly imply that \bu{1}, \bu{2.1} and \bu{4} are
satisfied by AGM revision operators. Hence,
Theorem~\ref{thm:belupd:model-based by exception-based} directly applies to
AGM revision operators as well. Note also that the operator adopted for ABox
updates in \cite{Lenzerini2011}, inspired by WIDTIO, performs a deductive
closure of the ABox before updating it, so it corresponds to the standard
\emph{full meet AGM revision operator}.

Turning to formula-based belief update operators, we can achieve similar
results by introducing the following principles satisfied by many
formula-based operators. Here, for any knowledge base $\kb$,
$\modr{\kb}^\twis$ denotes the set $\modr{\kb} \cup \set{\twis}$. The
principles are as follows:
\begin{textenum}[1.]
 	\setlength{\labelwidth}{.9cm}
	\setlength{\itemindent}{.5cm}
	\renewcommand{\theenumi}{\labfu{\arabic{enumi}}}
	\renewcommand{\labelenumi}{\theenumi\hfill}

	\item $\modr{\kba \uopf \kbu} \supseteq \modr{\kbu}$.
		\label{pstl:fu:1}

	\setcounter{enumi}{0}
	\renewcommand{\theenumi}{\labfu{2.\arabic{enumi}}}
	\item $\modr{\kba \cup \kbu} \supseteq \modr{\kba \uopf \kbu}$.
		\label{pstl:fu:2.1}

	\setcounter{enumi}{3}
	\renewcommand{\theenumi}{\labfu{\arabic{enumi}}}
	\item If $\modr{\kba}^\twis = \modr{\kbb}^\twis$ and $\modr{\kbu}^\twis =
		\modr{\kbv}^\twis$, then $\modr{\kba \uopf \kbu}^\twis = \modr{\kbb \uopf
		\kbv}^\twis$.
		\label{pstl:fu:4}
\end{textenum}
We can see that \fu{1} and \fu{2.1} are \emph{stronger} versions of \bu{1},
and \bu{2.1}, respectively. While \fu{1} requires that the sets of models of
formulas in $\kbu$ be retained in the semantic characterisation of $\kba \uopf
\kbu$, \fu{2.1} states that every formula in $\kba \uopf \kbu$ be equivalent
to some formula in $\kba \cup \kbu$. Intuitively, this means that $\kba \uopf
\kbu$ is obtained from $\kba \cup \kbu$ by deleting some of its elements,
modulo equivalence. Finally, \fu{4} is a reformulation of \bu{4} that is
satisfied by formula-based operators -- it can be seen as syntax-independence
w.r.t.\ the set of sets of models of a knowledge base, modulo the presence of
tautologies, instead of the overall set of models as in \bu{4}. In some ways
it is \emph{weaker} than \bu{4} as its antecedent is much stronger.

The WIDTIO operator \cite{Ginsberg1986,Ginsberg1988,Winslett1990} satisfies
all of these principles, and so does the Bold operator \cite{Calvanese2010} if
it is based on a remainder selection function that selects remainders with the
same semantic characterisation when given sets of remainders with the same
sets of semantic characterisations. More formally:

\begin{definition}
	[Regular Bold Operator]
	For any set of remainders $\srem$ we define $\modrr{\srem}^\twis =
	\set{\modr{\kba}^\twis | \kba \in \srem}$. We say that the Bold operator
	$\uopbold$ is \emph{regular} if for all sets of remainders $\srem_1$,
	$\srem_2$ such that $\modrr{\srem_1}^\twis = \modrr{\srem_2}^\twis$ it holds
	that $\modr{\sfn(\srem_1)}^\twis = \modr{\sfn(\srem_2)}^\twis$.
\end{definition}

The regularity condition guarantees a certain degree of independence of
syntax, e.g.\ given the sets of remainders $\srem_1 = \set{\set{\atma},
\set{\atmb}}$ and $\srem_2 = \set{\set{\atma \land \atma}, \set{\atmb \lor
\atmb}}$, a regular Bold operator either selects $\set{\atma}$ from $\srem_1$
and $\set{\atma \land \atma}$ from $\srem_2$, or it selects $\set{\atmb}$ from
$\srem_1$ and $\set{\atmb \lor \atmb}$ from $\srem_2$. A non-regular one might
select, say, $\set{\atma}$ from $\srem_1$ and $\set{\atmb \lor \atmb}$ from
$\srem_2$. Thus the regularity condition ensures that the operator is
independent of the syntax of individual formulas in the knowledge base.

The Cross-Product operator \cite{Ginsberg1986} satisfies \fu{1} (thus also
\bu{1}), \bu{2.1} and \fu{4}, but not \fu{2.1}.

\begin{proposition}
	[Properties of Formula-Based Updates]
	\label{prop:belupd:formula-based properties}
	The WIDTIO and regular Bold operators satisfy \fu{1}, \fu{2.1} and \fu{4}.
	The Cross-Product operator satisfies \fu{1}, \bu{2.1} and \fu{4} but does
	not satisfy $\fu{2.1}$.
\end{proposition}
\begin{proof}
	See \ref{app:belupd}, page~\pageref{proof:belupd:formula-based
	properties}.
\end{proof}

The following result establishes that formula-based operators such as WIDTIO
and regular Bold can be fully captured by exception-based operators. In
addition, operators such as Cross-Product can be captured for the case of a
single update.

\begin{theorem}
	[Formula-Based Updates Using Exception-Based Operators]
	\label{thm:belupd:formula-based by exception-based}
	If $\uopf$ is an update operator that satisfies \fu{1}, \fu{2.1} and \fu{4},
	then there exists an exception function \te{} such that for every \te-based
	update operator $\uope$ and all finite sequences of knowledge bases $\dkb$,
	$\mod{\biguopf \dkb} = \mod{\biguope \dkb}$.

	If $\uopf$ is an update operator that satisfies \bu{1}, \bu{2.1} and
	\fu{4}, then there exists an exception function \te{} such that for every
	\te-based update operator $\uope$ and all knowledge bases $\kba$, $\kbu$,
	$\mod{\kba \uopf \kbu} = \mod{\kba \uope \kbu}$.
\end{theorem}
\begin{proof}
	See \ref{app:belupd}, page~\pageref{proof:belupd:formula-based by
	exception-based}.
\end{proof}

Similarly as with principles of model-based update operators, principles
\fu{1}, \fu{2.1} and \fu{4} are closely related with the properties of
\emph{base revision operators} \cite{Gardenfors1992,Hansson1993}. In
particular, two types of base revision are identified in \cite{Hansson1993},
the \emph{internal} and \emph{external base revision}. Both of them satisfy
base revision counterparts of \emph{success} and \emph{inclusion} and, in
addition, internal revision operators satisfy a property called
\emph{uniformity}. These three principles together entail that internal
revision operators satisfy \fu{1}, \fu{2.1} and one half of \fu{4}; the other
half can be achieved by putting additional constraints on the two-place
selection function that generates the revision operator, similar to the
\emph{regularity} condition we imposed on the Bold operator above. Such
regular internal revision operators are thus directly subject to
Theorem~\ref{thm:belupd:formula-based by exception-based}. The same however
does not hold for regular external revision operators as they need not satisfy
\emph{uniformity}. Note also that the WIDTIO and Bold operators coincide with
\emph{internal full meet base revision} and \emph{internal maxichoice base
revision} operators, respectively. 

\section{Conclusion}

\label{sec:discussion}

Throughout this paper we developed a novel perspective on knowledge updates
and demonstrated that it offers the first unifying ground for characterising
rule update semantics as well as both model- and formula-based classical
update operators.

More particularly, we defined a new monotonic characterisation of rules, the
\RE-models, and introduced a generic method for specifying semantic rule
update operators in which a logic program is viewed as the \emph{set of sets
of \RE-models of its rules} and updates are performed by introducing
additional interpretations to the sets of \RE-models of rules in the original
program. This framework allowed us to define concrete semantic rule update
operators that enjoy an interesting combination of syntactic as well as
semantic properties that had never been reconciled before. Furthermore, it
turned out that these operators can faithfully characterise the foundational
\JU- and \AS-semantics for rule updates.

These insights also allowed us to address the long enduring problem of
\emph{state condensing}, i.e.\ transforming a sequence of answer-set programs
-- interpreted as updates -- into a single answer-set program written in the
same alphabet. Partly, this problem emerges because some rule update semantics
employ complex mechanisms such as preferences or other minimality criteria
that make it impossible to encode the result in a single answer-set program.
Others have model-theoretic characterisations that assign non-minimal models
to certain update sequences, and it is well known that stable models of
non-disjunctive answer-set programs are minimal. By resorting to more
expressive classes of answer-set programs, namely nested and disjunctive, we
solved this problem for both \JU- and \AS-semantics. In all four cases, two
for each semantics using both classes of answer-set programs, the resulting
program is written with the same alphabet and is ready to be further updated.
We have illustrated with some examples that the resulting programs written
using nested answer-set programming are perhaps more readable than those
written using disjunctive answer-set programs, in the sense that they more
closely match the intuitions underlying the semantics for updates that we
consider.

Furthermore, we defined \emph{abstract exception-based operators} that can be
used in any knowledge representation formalism with a monotonic
model-theoretic semantics. Then we showed that exception-based operators for
propositional knowledge bases can fully capture update operators that form the
basis of ontology updates, such as the model-based Winslett's operator, or the
formula-based WIDTIO and Bold operators
\cite{Liu2006,Giacomo2009,Calvanese2010,Lenzerini2011}. The Cross-Product
operator can be captured when a single update is performed and the same can be
said about the Set-Of-Theories operator since for a single update it is
equivalent to the Cross-Product operator \cite{Winslett1990}. Nevertheless,
neither of these two operators offers a viable alternative for updating
ontologies -- Cross-Product requires that disjunctions of ontology axioms be
performed, which is typically not supported in DLs, and Set-Of-Theories
produces a disjunctive ontology which is impractical and deviates from
mainstream DL research.

Overall, exception functions and exception-based operators can capture both
traditional syntax-based approaches to rule updates as well as a wide range of
model- and formula-based belief update operators. Thus, they offer a uniform
framework that bridges two very distinct approaches to updates, previously
considered irreconcilable. These findings are essential to better understand
their interrelations. In addition, they open up new possibilities for
addressing updates of Hybrid Knowledge Bases consisting of both an ontology
and a rule component since the different methods used for dealing with ABox,
TBox and rule updates can be viewed uniformly by looking at their associated
exception functions. When coupled with a counterpart of \SE- or \RE-models in
the context of hybrid knowledge bases, this can lead to universal hybrid
update semantics which, in turn, can further improve our understanding of the
distinct update paradigms.

Our investigation also directly points to challenges that need to be tackled
next. First, semantic characterisations of additional rule update semantics
need to be investigated. This poses a number of challenges due to the need to
detect non-tautological irrelevant updates
\cite{Alferes2005,Sefranek2006,Sefranek2011}. For instance, the simple
functions examined in this paper, as well as the original \JU- and
\AS-semantics, cannot distinguish an update of $\set{\atma.}$ by $\prgu =
\set{\lpnot \atma \lpif \lpnot \atmb., \lpnot \atmb \lpif \lpnot \atma.}$,
where it is plausible to introduce the exception $\tpl{\emptyset, \emptyset}$
(and the stable model $\emptyset$ along with it), from an update of
$\set{\atma., \atmb.}$ by $\prgu$, where such an exception should not be
introduced due to the cyclic dependency of justifications to reject $(\atma.)$
and $(\atmb.)$. In such situations, context-aware functions need to be used.
Such functions have the potential of satisfying properties such as \pu{3} and
\pup{Associativity}. They would facilitate the search for condensing operators
for other rule update semantics and perhaps shed some light on the problem of
\emph{updating disjunctive programs} which has received very little attention
up until now.

Another challenge is to find additional logical characterisations of rules,
namely a notion of program equivalence that is weaker than \RR-equivalence but
stronger than \RE-equivalence so that both \pu{4} and properties such as
\pu{2.1} can be achieved under a single notion of program equivalence. In
this context, the close relationship between \RE-models and \emph{T-models}
\cite{Wong2007}, used in the context of \emph{forgetting} in logic programs,
asks for more attention as well.

Computational properties of different classes of exception-based update
operators should also be investigated and it might be interesting to look for
constrained classes of exception functions that satisfy syntax-independence
w.r.t.\ \SR-equivalence. Such functions, however, will not be able to respect
causal rejection because \SE-models cannot distinguish abolishing rules.

Our discussion of the expressivity of exception-based operators w.r.t.\
\emph{revision} operators, on both belief sets and belief bases, can be used
to tackle and unify approaches to \emph{ontology revision}
\cite{Qi2008,Halaschek-Wiener2006,Ribeiro2007}. This seems relevant even in
the context of ontology updates since it has been argued in the literature
that the strict distinction between revision and update is not suitable in the
context of ontologies \cite{Calvanese2010}.

\section*{Acknowledgements}\label{sec-ackn}
J.\ Leite was partially supported by Funda{\c c}\~ao para a Ci\^encia e a Tecnologia (FCT) under project NOVA LINCS ({UID}/{CEC}/{04516}/{2013}).

\appendix

\section{Proofs: Robust Equivalence Models}

\label{app:program-equivalence}

In this section we prove formal properties of \RE-models (c.f.\
Definition~\ref{def:re model}). We begin by defining a set of representatives
of rule equivalence classes induced by \SE-models, as it was introduced
in \cite{Slota2011} since it is needed in our proofs.

\begin{definition}
	[Transformation into an \SE-Canonical Rule \cite{Slota2011}]
	\label{def:se:transformation into canonical rule}
	Given a rule $\rl$, we define the \SE-canonical rule $\canse{\rl}$ as
	follows:
	\begin{textenum}[(i)]
		\item
			\label{case:def:se-canonical rule:1}
			If any of the sets $\hrl^+ \cap \brl^+$, $\hrl^- \cap \brl^-$ and
			$\brl^+ \cap \brl^-$ is non-empty, then $\canse{\rl} = \ctau$.

		\item
			\label{case:def:se-canonical rule:2}
			If \eqref{case:def:se-canonical rule:1} does not apply and $\hrl^+
			\setminus \brl^- \neq \emptyset$, then $\canse{\rl} = \Br{(\hrl^+ \setminus
			\brl^-); \lcmp{(\hrl^- \setminus \brl^+)} \lpif \brl^+, \lcmp{\brl^-}.}$.

		\item
			\label{case:def:se-canonical rule:3}
			If \eqref{case:def:se-canonical rule:1} does not apply and $\hrl^+
			\setminus \brl^- = \emptyset$, then $\canse{\rl}$ is the constraint
			$\Br{\lpif (\brl^+ \cup \hrl^-), \lcmp{\brl^-}.}$.
	\end{textenum}
\end{definition}

\begin{proposition}
	[\cite{Slota2011}]
	\label{prop:se:canonical equivalence}
	Every rule $\rl$ is \SE-equivalent to the \SE-canonical rule $\canse{\rl}$.
\end{proposition}

\begin{corollary}
	[\cite{Slota2011}]
	\label{cor:se:canonical not equivalent}
	No two different \SE-canonical rules are \SE-equivalent.
\end{corollary}

The following presentation follows a similar pattern as the one used in
\cite{Slota2011} for \SE-models. In particular, we introduce a set of
representatives of rule equivalence classes induced by \RE-models and show how
to reconstruct a representative from the set of its \RE-models. Then we prove
the properties of \RE-models that have been formulated in
Section~\ref{sec:robust equivalence models}.

\begin{remark}
	We use the following additional notation: For any rule $\rl$, $\rl^+$
	denotes the rule $(\hrlp \lpif \brlp.)$ and $\rl^-$ denotes the rule
	$\Br{\lcmp{\hrln} \lpif \lcmp{\brln}.}$. Note that the definition of an
	\RE-model (c.f.\ Definition~\ref{def:re model}) implies that $\twiab \in
	\modre{\rl}$ if and only if $\twia \ent \rl^+$ or $\twib \ent \rl^-$. This
	fact is used implicitly in the following proofs.
\end{remark}


\begin{lemma}
	\label{lemma:re:tautology}
	Rules of the following forms are \RE-tautological:
	\begin{align*}
		\atm; \slith &\lpif \atm, \slitb. &
		\slith; \lpnot \atm &\lpif \slitb, \lpnot \atm. &
		\slith \lpif \slitb, \atm, \lpnot \atm.
	\end{align*}
\end{lemma}
\begin{proof}
	\label{proof:lemma:re:tautology}
	First assume that a rule $\rl$ is of the first form and take some $\twiab
	\in \tris$. We need to show that $\twiab$ is an \RE-model of $\rl$, so it
	suffices to show that $\twia \ent \rl^+$. This follows from the fact that
	$\atm$ belongs to both $\hrl^+$ and $\brl^+$.

	Now suppose that $\rl$ is of the second form. Given some $\twiab \in \tris$,
	we see that the atom $\atm$ belongs to both $\hrl^-$ and $\brl^-$, so $\twib
	\ent \rl^-$. Hence, $\twiab$ is an \RE-model of $\rl$.
	
	Finally, suppose that $\rl$ takes the third form and take some $\twiab \in
	\tris$. If $\twib \ent \rl^-$, then $\twiab$ is an \RE-model of $\rl$. On
	the other hand, if $\twib \nent \rl^-$, then $\twib \ent \lpnot \atm$ and,
	consequently, $\twia \nent \atm$ because $\twia$ is a subset of $\twib$.
	This implies that $\twia \ent \rl^+$, so we can once again conclude that
	$\twiab$ is an \RE-model of $\rl$.
\end{proof}


\begin{lemma}
	\label{lemma:re:head repetition}
	The following pairs of rules are \RE-equivalent:
	\begin{textenum}[(1)]
		\item
			\label{lemma:re:head repetition:positive}
			$(\atm; \slith \lpif \slitb, \lpnot \atm.)$ and $(\slith \lpif
			\slitb, \lpnot \atm.)$;

		\item
			\label{lemma:re:head repetition:negative}
			$(\slith; \lpnot \atm \lpif \atm, \slitb.)$ and $(\slith \lpif \atm,
			\slitb.)$.
	\end{textenum}
\end{lemma}
\begin{proof}
	~
	\begin{textenum}[(1)]
		\item Let the first rule be denoted by $\rl_1$, the second by $\rl_2$ and
			take some $\twiab \in \tris$. We need to show that $\twiab$ is an
			\RE-model of $\rl_1$ if and only if it is an \RE-model of $\rl_2$. Thus,
			it suffices to prove the following:
			\begin{equation}
				\label{eq:lemma:re:positive head repetition:1}
				\twia \ent \rl_1^+ \lor \twib \ent \rl_1^-
				\qquad \text{if and only if} \qquad
				\twia \ent \rl_2^+ \lor \twib \ent \rl_2^-
				\enspace.
			\end{equation}
			First note that $\rl_1^- = \rl_2^-$, so $\twib \ent \rl_1^-$ holds if
			and only if $\twib \ent \rl_2^-$. So if $\twib \ent \rl_1^-$, then
			$\twib \ent \rl_2^-$ and we can conclude that
			\eqref{eq:lemma:re:positive head repetition:1} holds. On the other hand,
			if $\twib \nent \rl_1^-$, then $\twib \nent \rl_2^-$ and
			\eqref{eq:lemma:re:positive head repetition:1} reduces to proving that
			$\twia \ent \rl_1^+$ holds if and only if $\twia \ent \rl_2^+$. Now it
			suffices to observe that $\twib \nent \rl_1^-$ implies $\twib \ent
			\lpnot \atm$ and since $\twia$ is a subset of $\twib$, we can conclude
			that $\twia \nent \atm$. Since $\rl_1^+$ differs from $\rl_2^+$ only in
			the head atom $\atm$, our claim follows.

		\item Let the first rule be denoted by $\rl_1$, the second by $\rl_2$ and
			take some $\twiab \in \tris$. We need to show that $\twiab$ is an
			\RE-model of $\rl_1$ if and only if it is an \RE-model of $\rl_2$. Thus,
			it suffices to prove the following:
			\begin{equation}
				\label{eq:lemma:re:negative head repetition:1}
				\twia \ent \rl_1^+ \lor \twib \ent \rl_1^-
				\qquad \text{if and only if} \qquad
				\twia \ent \rl_2^+ \lor \twib \ent \rl_2^-
				\enspace.
			\end{equation}
			First note that $\rl_1^+ = \rl_2^+$, so $\twia \ent \rl_1^+$ holds if
			and only if $\twia \ent \rl_2^+$. So if $\twia \ent \rl_1^+$, then
			$\twia \ent \rl_2^+$ and we can conclude that
			\eqref{eq:lemma:re:negative head repetition:1} holds. On the other hand,
			if $\twia \nent \rl_1^+$, then $\twia \nent \rl_2^+$ and
			\eqref{eq:lemma:re:negative head repetition:1} reduces to proving that
			$\twib \ent \rl_1^-$ holds if and only if $\twib \ent \rl_2^-$. Now it
			suffices to observe that $\twia \nent \rl_1^+$ implies $\twia \ent \atm$
			and since $\twia$ is a subset of $\twib$, we can conclude that $\twib
			\nent \lpnot \atm$. Since $\rl_1^-$ differs from $\rl_2^-$ only in the
			head literal $\lpnot \atm$, our claim follows. \qedhere
	\end{textenum}
\end{proof}

\begin{definition}
	[\RE-Canonical Rule]
	\label{def:re:canonical rule}
	We say that a rule $\rl$ is \emph{\RE-canonical} if either it is $\ctau$, or
	the sets $\hrl^+ \cup \hrl^-$, $\brl^+$ and $\brl^-$ are pairwise disjoint.
\end{definition}


\begin{definition}
	[Transformation into an \RE-Canonical Rule]
	\label{def:re:transformation into canonical rule}
	Given a rule $\rl$, we define the \RE-canonical rule $\canre{\rl}$ as follows:
	\begin{textenum}[(i)]
		\item
			\label{case:def:re-canonical rule:1}
			If any of the sets $\hrl^+ \cap \brl^+$, $\hrl^- \cap \brl^-$ and
			$\brl^+ \cap \brl^-$ is non-empty, then $\canre{\rl} = \ctau$.
		\item
			\label{case:def:re-canonical rule:2}
			If \eqref{case:def:re-canonical rule:1} does not apply, then
			$\canre{\rl}$ is the rule
			$\Br{
				(\hrl^+ \setminus \brl^-); \lcmp{(\hrl^- \setminus \brl^+)}
				\lpif
				\brl^+, \lcmp{\brl^-}.
			}$.
	\end{textenum}
\end{definition}


\begin{proposition}
	\label{prop:re:canonical equivalence}
	For every rule $\rl$, $\modre{\rl} = \modre{\canre{\rl}}$.
\end{proposition}
\begin{proof}
	This can be shown by a careful iterative application of
	Lemmas~\ref{lemma:re:tautology} and \ref{lemma:re:head repetition}. First
	observe that if $\canre{\rl} = \ctau$, then Lemma~\ref{lemma:re:tautology}
	implies that $\rl$ is \RE-tautological, thus indeed \RE-equivalent to
	$\ctau$.

	In the principal case we can use Lemma~\ref{lemma:re:head
	repetition}\eqref{lemma:re:head repetition:positive} on all atoms from
	$\hrl^+ \cap \brl^-$ and remove them one by one from $\hrl^+$ while
	preserving \RE-models. A similar situation occurs with atoms from $\hrl^-
	\cap \brl^+$ which can be, according to Lemma \ref{lemma:re:head
	repetition}\eqref{lemma:re:head repetition:negative}, removed from $\hrl^-$
	without affecting \RE-models. After these steps are performed, the resulting
	rule coincides with $\canre{\rl}$.
\end{proof}





\begin{lemma}
	\label{lemma:re:model conditions}
	For any rule $\rl$ and $\twiab \in \tris$, $\twiab \notin \modre{\rl}$ if
	and only if $\brl^+ \subseteq \twia \subseteq \atms \setminus \hrl^+$ and
	$\hrl^- \subseteq \twib \subseteq \atms \setminus \brl^-$.
\end{lemma}
\begin{proof}
	Note that $\twiab \notin \modre{\rl}$ if and only if both $\twia \nent
	\rl^+$ and $\twib \nent \rl^-$. It can be easily verified that the former is
	equivalent to $\brl^+ \subseteq \twia \subseteq \atms \setminus \hrl^+$ and
	the latter to $\hrl^- \subseteq \twib \subseteq \atms \setminus \brl^-$.
\end{proof}


\begin{corollary}
	\label{cor:re:model conditions}
	Let $\rl$ be an \RE-canonical rule different from $\ctau$, put $\twia =
	\brl^+$, $\twib = \hrl^- \cup \brl^+$ and $\twib' = \atms \setminus \brl^-$,
	and let $\atm$ be an atom. Then the following holds:
	\begin{textenum}[(1)]
		\item $\tpl{\twia, \twib \cup \set{\atm}} \in \modre{\rl}$ if and only if
			$\atm \in \brl^-$.
			\label{part:cor:re:model conditions:2}

		\item $\tpl{\twia \cup \set{\atm}, \twib \cup \set{\atm}} \in \modre{\rl}$
			if and only if $\atm \in \hrl^+ \cup \brl^-$.
			\label{part:cor:re:model conditions:3}


		\item $\tpl{\twia \setminus \set{\atm}, \twib'} \in \modre{\rl}$ if and
			only if $\atm \in \brl^+$.
			\label{part:cor:re:model conditions:5}

		\item $\twiab \notin \modre{\rl}$.
			\label{part:cor:re:model conditions:1}
	\end{textenum}
\end{corollary}
\begin{proof}
	Follows from Lemma~\ref{lemma:re:model conditions} and the disjointness
	properties satisfied by \RE-canonical rules.
\end{proof}



\begin{lemma}
	\label{lemma:re:negative body}
	Let $\rl$ be an \RE-canonical rule different from $\ctau$ and $\atm$ an
	atom. Then the following holds:
	\begin{textenum}[(1)]
		\item $\atm \in \brl^-$ if and only if for all $\twiab \in \tris$, $\atm
			\in \twib$ implies $\twiab \in \modre{\rl}$;

		\item $\atm \in \hrl^+$ if and only if $\atm \notin \brl^-$ and for all
			$\twiab \in \tris$, $\atm \in \twia$ implies $\twiab \in \modre{\rl}$;

		\item $\atm \in \brl^+$ if and only if for all $\twiab \in \tris$, $\atm
			\notin \twia$ implies $\twiab \in \modre{\rl}$;

		\item $\atm \in \hrl^-$ if and only if $\atm \notin \brl^+$ and for all
			$\twiab \in \tris$, $\atm \notin \twib$ implies $\twiab \in
			\modre{\rl}$.
	\end{textenum}
\end{lemma}
\begin{proof}
	~
	\begin{textenum}[(1)]
		\item Suppose that $\atm \in \brl^-$ and take some $\twiab \in \tris$ with
			$\atm \in \twib$. Then $\twib \nent \lpnot \atm$, so it follows that
			$\twib \ent \rl^-$. Consequently $\twiab \in \modre{\rl}$.

			To prove the converse implication, let $\twia = \brl^+$ and $\twib =
			\hrl^- \cup \brl^+$. It follows that $\tpl{\twia, \twib \cup \set{\atm}}
			\in \modre{\rl}$, so by Corollary~\ref{cor:re:model
			conditions}\eqref{part:cor:re:model conditions:2} we conclude that $\atm
			\in \brl^-$.

		\item Suppose that $\atm \in \hrl^+$ and take some $\twiab \in \tris$ with
			$\atm \in \twia$. Then $\twia \ent \atm$, so it follows that $\twia \ent
			\rl^+$. Consequently $\twiab \in \modre{\rl}$.
			
			To prove the converse implication, let $\twia = \brl^+$ and $\twib =
			\hrl^- \cup \brl^+$. It follows that $\tpl{\twia \cup \set{\atm}, \twib
			\cup \set{\atm}} \in \modre{\rl}$, so by Corollary~\ref{cor:re:model
			conditions}\eqref{part:cor:re:model conditions:3} we conclude that
			$\atm$ belongs to $\hrl^+ \cup \brl^-$. Moreover, by the assumption we
			know that $\atm \notin \brl^-$, so $\atm \in \hrl^+$.

		\item Suppose that $\atm \in \brl^+$ and take some $\twiab \in \tris$ with
			$\atm \notin \twia$. Then $\twia \nent \atm$, so it follows that $\twia
			\ent \rl^+$. Consequently, $\twiab \in \modre{\rl}$.

			To prove the converse implication, let $\twia = \brl^+$ and $\twib' =
			\atms \setminus \brl^-$. It follows that $\tpl{\twia \setminus
			\set{\atm}, \twib'} \in \modre{\rl}$, so by
			Corollary~\ref{cor:re:model conditions}\eqref{part:cor:re:model
			conditions:5} we conclude that $\atm \in \brl^+$.

		\item Suppose that $\atm \in \hrl^-$ and take some $\twiab \in \tris$ with
			$\atm \notin \twib$. Then $\twib \ent \lpnot \atm$, so it follows that
			$\twib \ent \rl^-$. Consequently, $\twiab \in \modre{\rl}$.
	
			To prove the converse implication, let $\twia = \brl^+$ and $\twib =
			\hrl^- \cup \brl^+$. Corollary~\ref{cor:re:model
			conditions}\eqref{part:cor:re:model conditions:1} guarantees that
			$\tpl{\twia, \twib} \notin \modre{\rl}$. Furthermore, by the assumption
			it follows that $\tpl{\twia \setminus \set{\atm}, \twib \setminus
			\set{\atm}} \in \modre{\rl}$. Consequently, $\twib$ must differ from
			$\twib \setminus \set{\atm}$, which implies that $\atm \in \twib$.
			Furthermore, since $\twib = \hrl^- \cup \brl^+$ and $\atm \notin \brl^+$
			by assumption, we conclude that $\atm \in \hrl^-$. \qedhere
	\end{textenum}
\end{proof}

%
%
%


\begin{definition}
	[Rule \RE-Induced by a Set of Interpretations]
	\label{def:re:models to rule}
	Let $\stri \subseteq \tris$. The \emph{rule \RE-induced by $\stri$}, denoted
	by $\rulere{\stri}$, is defined as follows: If $\stri = \tris$, then
	$\rulere{\stri} = \ctau$; otherwise, $\rulere{\stri}$ is of the form
	$(\hstrire^+; \lcmp{\hstrire^-} \lpif \bstrire^+, \lcmp{\bstrire^-}.)$ where
	\begin{align*}
		\bstrire^- &= \Set{
			\atm \in \atms
			|
			\forall \twiab \in \tris :
			\atm \in \twib
			\mlthen \twiab \in \stri
		} \enspace, \\
		\hstrire^+ &= \Set{
			\atm \in \atms
			|
			\forall \twiab \in \tris :
			\atm \in \twia
			\mlthen \twiab \in \stri
		} \setminus \bstrire^- \enspace, \\
		\bstrire^+ &= \Set{
			\atm \in \atms
			|
			\forall \twiab \in \tris :
			\atm \notin \twia
			\mlthen \twiab \in \stri
		} \enspace, \\ 
		\hstrire^- &= \Set{
			\atm \in \atms
			|
			\forall \twiab \in \tris :
			\atm \notin \twib
			\mlthen \twiab \in \stri
		} \setminus \bstrire^+ \enspace.
	\end{align*}
\end{definition}


\begin{proposition}
	\label{prop:re:canonical from models}
	For every \RE-canonical rule $\rl$, $\rulere{\modre{\rl}} = \rl$.
\end{proposition}
\begin{proof}
	If $\rl = \ctau$, then $\modre{\rl} = \tris$ and, by
	Definition~\ref{def:re:models to rule}, $\rulere{\tris} = \ctau = \rl$, so
	the identity is satisfied. In the principal case, $\rl$ is an \RE-canonical
	rule different from $\ctau$. Let $\stri = \modre{\rl}$. It follows from
	Definition~\ref{def:re:models to rule} and Lemma~\ref{lemma:re:negative
	body} that $\rl = \rulere{\stri}$.
\end{proof}

%


\begin{corollary}
	\label{cor:re:canonical not equivalent}
	No two different \RE-canonical rules are \RE-equivalent.
\end{corollary}
\begin{proof}
	Follows directly from Proposition~\ref{prop:re:canonical from models}.
\end{proof}

%



\begin{proof}
	[\textbf{Proof of Proposition~\ref{prop:re:abolishing rules not equivalent}}]
	\label{proof:re:abolishing rules not equivalent}
	Follows from Corollary~\ref{cor:re:canonical not equivalent} since every
	abolishing rule and every constraint is \RE-canonical.
\end{proof}

\begin{lemma}
	\label{lemma:re:canse equals canse canme}
	For every rule $\rl$, $\canse{\rl} = \canse{\canre{\rl}}$.
\end{lemma}
\begin{proof}
	Follows directly from Definitions~\ref{def:se:transformation into canonical
	rule} and \ref{def:re:transformation into canonical rule}.
\end{proof}

\begin{lemma}
	\label{lemma:re:non-abolishing have equal se- and me-canonical}
	If $\rl$ is not \RE-equivalent to any abolishing rule, then $\canre{\rl} =
	\canse{\rl}$.
\end{lemma}
\begin{proof}
	Follows directly from Definitions~\ref{def:se:transformation into canonical
	rule} and \ref{def:re:transformation into canonical rule}.
\end{proof}

%
%
\begin{proof}
	[\textbf{Proof of Proposition~\ref{prop:re:se-equivalence}}]
	\label{proof:re:se-equivalence}

	Suppose that $\rla$ and $\rlb$ are \RE-equivalent. Then $\canre{\rla} =
	\canre{\rlb}$ by Proposition~\ref{prop:re:canonical equivalence} and
	Corollary~\ref{cor:re:canonical not equivalent}. By
	Lemma~\ref{lemma:re:canse equals canse canme} it follows that $\canse{\rla}
	= \canse{\canre{\rla}} = \canse{\canre{\rlb}} = \canse{\rlb}$, so by
	Proposition~\ref{prop:se:canonical equivalence} we can conclude that $\rla$,
	$\rlb$ are \SE-equivalent.

	Now suppose that neither $\rla$ nor $\rlb$ is \RE-equivalent to an
	abolishing rule and $\rla$ is \RE-equivalent to $\rlb$. Then, by
	Proposition~\ref{prop:re:canonical equivalence}, $\canre{\rla}$ is
	\RE-equivalent to $\canre{\rlb}$ and by Corollary~\ref{cor:re:canonical not
	equivalent}, $\canre{\rla} = \canre{\rlb}$. It follows from
	Lemma~\ref{lemma:re:non-abolishing have equal se- and me-canonical} that
	$\canse{\rla} = \canse{\rlb}$ and by Proposition~\ref{prop:se:canonical
	equivalence}, $\rla$ is \SE-equivalent to $\rlb$.

	Finally, by Corollary~\ref{cor:re:canonical not equivalent}, a rule $\rl$ is
	\RE-tautological if and only if $\canre{\rl} = \ctau$ which holds if and
	only if $\canse{\rl} = \ctau$
	which, by Corollary~\ref{cor:se:canonical not equivalent}, holds if and only
	if $\rl$ is \SE-tautological.
\end{proof}

\begin{proof}
	[\textbf{Proof of Proposition~\ref{prop:re:sm}}]
	\label{proof:re:sm}
	Suppose that $\twib$ is a stable model of $\prg$. Then $\twib$ is a
	subset-minimal model of $\prg^\twib$. Thus, $\tpl{\twib, \twib}$ is an
	\RE-model of $\prg$. Now suppose that $\twiab$ is an \RE-model of $\prg$ for
	some $\twia \subseteq \twib$. Then $\twia \ent \prg^\twib$ and by the
	minimality of $\twib$ we obtain $\twia = \twib$.
	
	Suppose that $\tpl{\twib, \twib}$ is an \RE-model of $\prg$ and for all
	$\twia \subsetneq \twib$, $\twiab$ is not an \RE-model of $\prg$. Then
	$\twib \ent \prg^\twib$ and $\twib$ is also a subset-minimal model of
	$\prg^\twib$. Consequently, $\twib$ is a stable model of $\prg$.
\end{proof}

\begin{proof}
	[\textbf{Proof of Proposition~\ref{prop:re:expressibility}}]
	\label{proof:re:expressibility}
	Let $\prg$ contain the rule $\rl_{\twiab} = \Br{(\atms \setminus \twia);
	\lcmp{\twib} \lpif \twia, \lcmp{(\atms \setminus \twib)}.}$ for each
	$\twiab \in \tris \setminus \stri$. It is an immediate consequence of
	Lemma~\ref{lemma:re:model conditions} that $\modre{\rl_{\twiab}} = \tris
	\setminus \set{\twiab}$. Thus,
	\[
		\modre{\prg}
		= \bigcap_{\twiab \in \tris \setminus \stri}
			\tris \setminus \set{\twiab}
		= \tris \setminus
			\bigcup_{\twiab \in \tris \setminus \stri}
			\set{\twiab}
		= \tris \setminus (\tris \setminus \stri)
		= \stri
		\enspace. \qedhere
	\]
\end{proof}


\begin{definition}
	A \emph{program entailment relation} is a preorder on the set of all
	programs. A \emph{program equivalence relation} is an equivalence relation
	on the set of all programs. 

	Given a program entailment relation $\ent$ and a program equivalence
	relation $\eq$, we say that $\ent$ \emph{is associated with $\eq$} if for
	all programs $\prga$, $\prgb$, $\prga \eq \prgb$ holds if and only if
	$\prga \ent \prgb$ and $\prgb \ent \prga$.
\end{definition}

\begin{proof}
	[\textbf{Proof of Proposition~\ref{prop:ent vs eq}}]
	\label{proof:prop:ent vs eq}

	If \X{} is \SE{}, \RE{} or \SU{}, then the property follows immediately from
	the definitions of $\entX$ and $\eqX$.

	If \X{} is either \SR{} or \RR{}, then it follows from the definition of
	$\entX$ that $\prga \entX \prgb$ is equivalent to $\modxr{\prga^\ctau}
	\supseteq \modxr{\prgb^\ctau}$. Thus, $\prga \entX \prgb$ together with
	$\prgb \entX \prga$ is equivalent to $\modxr{\prga^\ctau} =
	\modxr{\prgb^\ctau}$, which is the definition of $\prga \eqX \prgb$.

	It remains to consider the case when \X{} is \SMR{} or \RMR{}. Let \Y{} be
	\SE{} or \RE{}, respectively. First suppose that $\prga \eqX \prgb$. By the
	definition of $\eqX$ we obtain that $\min \modyr{\prga^\ctau} = \min
	\modyr{\prgb^\ctau}$. Our goal is to prove that $\prga \entX \prgb$ and
	$\prgb \entX \prga$. We only show the former; the proof of the latter is
	analogous. Take some $\rlb \in \prgb$. Our goal is find some $\rla \in
	\prga^\ctau$ such that $\mody{\rla} \subseteq \mody{\rlb}$. Take some
	subset-minimal $\stri \in \modyr{\prgb^\ctau}$ such that $\stri \subseteq
	\mody{\rlb}$. It follows from our assumption that $\stri$ belongs to
	$\modyr{\prga^\ctau}$. In other words, there exists some $\rla \in
	\prga^\ctau$ such that $\mody{\rla} = \stri \subseteq \mody{\rlb}$.

	Now suppose that both $\prga \entX \prgb$ and $\prgb \entX \prga$. We need
	to prove that $\prga \eqX \prgb$, i.e.\ that $\min \modyr{\prga^\ctau} =
	\min \modyr{\prgb^\ctau}$. We only show that $\min \modyr{\prga^\ctau}
	\subseteq \min \modyr{\prgb^\ctau}$; the proof of the other inclusion is
	analogical. Take some $\rla \in \prga^\ctau$ such that
	\begin{equation}
		\label{eq:proof:prop:ent vs eq:2}
		\mody{\rla} \in \min \modyr{\prga^\ctau}
		\enspace.
	\end{equation}
	Since $\prgb \entX \prga$, there exists some $\rlb \in \prgb^\ctau$ such
	that $\mody{\rlb} \subseteq \mody{\rla}$ Let $\rlb' \in \prgb^\ctau$ be such
	that $\mody{\rlb'} \in \min \modyr{\prgb^\ctau}$ and $\mody{\rlb'} \subseteq
	\mody{\rlb}$. Since $\prga \entX \prgb$, there exists some $\rla' \in
	\prga^\ctau$ such that $\mody{\rla'} \subseteq \mody{\rlb'}$. In other
	words, we know that $\mody{\rla'} \subseteq \mody{\rlb'} \subseteq
	\mody{\rlb} \subseteq \mody{\rla}$. Thus, due to \eqref{eq:proof:prop:ent vs
	eq:2} we can conclude that $\mody{\rla'} = \mody{\rlb'} = \mody{\rlb} =
	\mody{\rla}$, and it follows from the choice of $\rlb'$ that $\mody{\rla}
	\in \min \modyr{\prgb^\ctau}$.
\end{proof}

\begin{lemma}
	\label{lemma:entailment vs equivalence comparison}
	Let $\entX$, $\entY$ be program entailment relations and $\eqX$, $\eqY$
	program equivalence relations such that $\entX$ is associated with $\eqX$
	and $\entY$ is associated with $\eqY$. Then $\entX \strle \entY$ implies
	$\eqX \strle \eqY$.
\end{lemma}
\begin{proof}
	Suppose that $\entX \strle \entY$ and take some programs $\prga$, $\prgb$
	such that $\prga \eqY \prgb$. We need to show that $\prga \eqX \prgb$.
	Since $\entY$ is associated with $\eqY$, we can conclude that $\prga \entY
	\prgb$ and $\prgb \entY \prga$. Furthermore, from $\entX \strle \entY$ it
	follows that $\prga \entX \prgb$ and $\prgb \entX \prga$, and the assumption
	that $\entX$ is associated with $\eqX$ implies $\prga \eqX \prgb$.
\end{proof}

\begin{corollary}
	\label{cor:entailment vs equivalence comparison}
	Let $\entX$, $\entY$ be program entailment relations and $\eqX$, $\eqY$
	program equivalence relations such that $\entX$ is associated with $\eqX$
	and $\entY$ is associated with $\eqY$. Then:
	\[
		\eqX \strl \eqY \text{ and } \entX \strl \entY
		\qquad \text{if and only if} \qquad
		\entX \strle \entY \text{ and } \eqY \nstrle \eqX
		\enspace.
	\]
\end{corollary}
\begin{proof}
	By the definition, $\eqX \strl \eqY$ and $\entX \strl \entY$ hold if and
	only if $\eqX \strle \eqY$, $\eqY \nstrle \eqX$, $\entX \strle \entY$ and
	$\entY \nstrle \entX$. By Lemma~\ref{lemma:entailment vs equivalence
	comparison}, $\entX \strle \entY$ implies $\eqX \strle \eqY$ and $\eqY
	\nstrle \eqX$ implies $\entY \nstrle \eqX$, so the condition can be
	simplified as desired.
\end{proof}

\begin{lemma}
	\label{lemma:reduct}
	Let $\rl$ be a rule and $\twib$ an interpretation. Then:
	\[
		\twib \ent \rl
		\qquad \text{if and only if} \qquad
		\twib \ent \rl^\twib
		\qquad \text{if and only if} \qquad
		\tpl{\twib, \twib} \in \modre{\rl}
		\enspace.
	\]
\end{lemma}
\begin{proof}
	Follows by the definition of $\rl^\twib$ and of \RE-models.
\end{proof}

\begin{lemma}
	\label{lemma:re-entailment implies se-entailment}
	Let $\gprga$, $\gprgb$ be rules or programs. Then $\modre{\gprga} \subseteq
	\modre{\gprgb}$ implies $\modse{\gprga} \subseteq \modse{\gprgb}$.
\end{lemma}
\begin{proof}
	Assume that $\modre{\gprga} \subseteq \modre{\gprgb}$. Then for all $\twiab
	\in \tris$, $\twia \ent \gprga^\twib$ implies $\twia \ent \gprgb^\twib$ and
	together with Lemma~\ref{lemma:reduct} this implies that for all
	interpretations $\twib$,
	\begin{equation}
		\label{eq:proof:lemma:re-entailment implies se-entailment:2}
		\twib \ent \gprga
		\qquad \text{implies} \qquad
		\twib \ent \gprga^\twib
		\qquad \text{implies} \qquad
		\twib \ent \gprgb^\twib
		\qquad \text{implies} \qquad
		\twib \ent \gprgb
		\enspace.
	\end{equation}
	In order to show that $\modse{\gprga} \subseteq \modse{\gprgb}$, take some
	$\twiab \in \modse{\gprga}$. By the definition of \SE-models, $\twib \ent
	\gprga$ and $\twia \ent \gprga^\twib$, so by
	\eqref{eq:proof:lemma:re-entailment implies se-entailment:2} and our
	assumption we can conclude that $\twib \ent \gprgb$ and $\twia \ent
	\gprgb^\twib$. Thus, by the definition of \SE-models, $\twiab \in
	\modse{\gprgb}$.
\end{proof}

%
%
%
%
%
%
\begin{proof}
	[\textbf{Proof of Proposition~\ref{prop:eq and ent comparison}}]
	\label{proof:prop:eq and ent comparison}
	~
	\begin{textenum}[(1)]
		\item First we show that $\eqSM \strl \eqSE$, i.e.\ that $\eqSM \strle
			\eqSE$ and $\eqSE \nstrle \eqSM$. To verify the former, suppose that
			$\prga$, $\prgb$ are programs with $\prga \eqSE \prgb$. Then $\prga \cup
			\emptyset$ has the same stable models as $\prgb \cup \emptyset$.
			Consequently, $\prga \eqSM \prgb$. To see that $\eqSE \nstrle \eqSM$,
			observe that the programs $\prga = \emptyset$, $\prgb = \set{\atma \lpif
			\atmb.}$ are \SM-equivalent but not \SE-equivalent.

			Turning to the remaining relationships, it follows from
			Corollary~\ref{cor:entailment vs equivalence comparison} that we can
			instead prove that
			\begin{equation}
				\label{eq:proof:prop:eq and ent comparison:1}
				\entSE \strle \entRE \strle \entRMR \strle \entRR \strle \entSU
				\qquad \text{and} \qquad
				\eqSU \nstrle \eqRR \nstrle \eqRMR \nstrle \eqRE \nstrle \eqSE
				\enspace.
			\end{equation}
			We first concentrate on the left-hand side of \eqref{eq:proof:prop:eq
			and ent comparison:1}. In order to show that $\entSE \strle \entRE$,
			suppose that $\prga$, $\prgb$ are programs such that $\prga \entRE
			\prgb$. Then $\modre{\prga} \subseteq \modre{\prgb}$ and it follows from
			Lemma~\ref{lemma:re-entailment implies se-entailment} that
			$\modse{\prga} \subseteq \modse{\prgb}$. Consequently, $\prga \entSE
			\prgb$.

			We also need to prove that $\entRE \strle \entRMR$. Take some programs
			$\prga$, $\prgb$ with $\prga \entRMR \prgb$ and put $\prga^\ctau = \prga
			\cup \set{\ctau}$. It follows that
			\begin{equation}
				\label{eq:proof:prop:eq and ent comparison:1:1}
				\forall \rlb \in \prgb \,\exists \rla_\rlb \in \prga^\ctau :
				\modre{\rla_\rlb} \subseteq \modre{\rlb}
				\enspace.
			\end{equation}
			We need to prove that $\modre{\prga} \subseteq \modre{\prgb}$. Suppose
			that $\tri \in \modre{\prga}$ and take some $\rlb \in \prgb$, our goal
			is to show that $\tri \in \modre{\rlb}$. By \eqref{eq:proof:prop:eq and
			ent comparison:1:1} there exists some $\rla_\rlb \in \prga^\ctau$ such
			that $\modre{\rla_\rlb} \subseteq \modre{\rlb}$. If $\rla_\rlb = \ctau$,
			then it immediately follows that $\tri \in \tris = \modre{\ctau}
			\subseteq \modre{\rlb}$. If $\rla_\rlb \in \prga$, then $\tri \in
			\modre{\rla_\rlb}$ by the choice of $\tri$, so $\tri \in \modre{\rlb}$.
			
			Our next goal is to show that $\entRMR \strle \entRR$. This follows
			directly by the definitions of $\entRMR$ and $\entRR$.

			To prove the final part of the left-hand side of \eqref{eq:proof:prop:eq
			and ent comparison:1}, suppose that $\prga \entSU \prgb$. Then
			$\modse{\prgb \setminus \prga} = \tris$. We need to prove that $\prga
			\entRR \prgb$, i.e.\ that for every $\rlb \in \prgb$ there is some $\rla
			\in \prga^\ctau$ such that $\modre{\rla} = \modre{\rlb}$. Pick some
			$\rlb \in \prgb$ and note that $\prgb = (\prgb \cap \prga) \cup (\prgb
			\setminus \prga)$. If $\rlb \in \prgb \cap \prga$, then $\rlb \in \prga$
			and we can put $\rla = \rlb$ to finish the proof. In the remaining case,
			$\rlb \in \prgb \setminus \prga$ and it follows from our assumption that
			$\modse{\rlb} = \tris$. Thus, putting $\rla = \ctau$ finishes the proof.

			As for the right-hand side of \eqref{eq:proof:prop:eq and ent
			comparison:1}, we can see that $\eqSU \nstrle \eqRR$ because the
			programs $\prga = \set{\lpnot \atma \lpif \atma.}$ and $\prgb =
			\set{\lpif \atma.}$ are \RR-equivalent but not \SU-equivalent.

			Similarly, programs $\prga = \set{\atma.}$ and $\prgb = \set{\atma.,
			\atma \lpif \atmb.}$ are \RMR-equivalent but not \RR-equivalent, so it
			follows that $\eqRR \nstrle \eqRMR$.
			
			Next, to verify that $\eqRMR \nstrle \eqRE$ it suffices to observe that
			the programs $\prga = \set{\atma., \atmb.}$ and $\prgb = \set{\atma.,
			\atmb \lpif \atma.}$ are \RE-equivalent but not \RMR-equivalent.

			Finally, programs $\prga = \set{\lpnot \atm.}$ and $\prgb = \set{\lpif
			\atm.}$ are \SE-equivalent but not \RE-equivalent, proving that $\eqRE
			\nstrle \eqSE$.

		\item It follows from Corollary~\ref{cor:entailment vs equivalence
			comparison} that we can instead prove that
			\begin{equation}
				\label{eq:proof:prop:eq and ent comparison:2}
				\entSE \strle \entSMR \strle \entSR \strle \entRR
				\qquad \text{and} \qquad
				\eqRR \nstrle \eqSR \nstrle \eqSMR \nstrle \eqSE
				\enspace.
			\end{equation}
			We first concentrate on the left-hand side of \eqref{eq:proof:prop:eq
			and ent comparison:2}. To prove that $\entSE \strle \entSMR$, take some
			programs $\prga$, $\prgb$ with $\prga \entSMR \prgb$ and put
			$\prga^\ctau = \prga \cup \set{\ctau}$. It follows that
			\begin{equation}
				\label{eq:proof:prop:eq and ent comparison:2:1}
				\forall \rlb \in \prgb \,\exists \rla_\rlb \in \prga^\ctau :
				\modse{\rla_\rlb} \subseteq \modse{\rlb}
				\enspace.
			\end{equation}
			We need to prove that $\modse{\prga} \subseteq \modse{\prgb}$. Suppose
			that $\tri \in \modse{\prga}$ and take some $\rlb \in \prgb$, our goal
			is to show that $\tri \in \modse{\rlb}$. By \eqref{eq:proof:prop:eq and
			ent comparison:2:1} there exists some $\rla_\rlb \in \prga^\ctau$ such
			that $\modse{\rla_\rlb} \subseteq \modse{\rlb}$. If $\rla_\rlb = \ctau$,
			then it immediately follows that $\tri \in \tris = \modse{\ctau}
			\subseteq \modse{\rlb}$. If $\rla_\rlb \in \prga$, then $\tri \in
			\modse{\rla_\rlb}$ by the choice of $\tri$, so $\tri \in \modse{\rlb}$.

			Our next goal is to show that $\entSMR \strle \entSR$. This follows
			directly by the definitions of $\entSMR$ and $\entSR$.

			To finish the proof of the left-hand side of \eqref{eq:proof:prop:eq and
			ent comparison:1}, suppose that $\prga \entRR \prgb$. Then $\forall \rlb
			\in \prgb \, \exists \rla \in \prga : \modre{\rla} = \modre{\rlb}$ and,
			due to Lemma~\ref{lemma:re-entailment implies se-entailment}, $\forall
			\rlb \in \prgb \, \exists \rla \in \prga : \modse{\rla} = \modse{\rlb}$.
			Consequently, $\prga \entSR \prgb$.
			
			As for the right-hand side of \eqref{eq:proof:prop:eq and ent
			comparison:1}, we can see that $\eqRR \nstrle \eqSR$ because the
			programs $\prga = \set{\lpnot \atma.}$ and $\prgb = \set{\lpif \atma.}$
			are \SR-equivalent but not \RR-equivalent.

			Similarly, programs $\prga = \set{\atma.}$ and $\prgb = \set{\atma.,
			\atma \lpif \atmb.}$ are \SMR-equivalent but not \SR-equivalent, so it
			follows that $\eqSR \nstrle \eqSMR$.
			
			Finally, to verify that $\eqSMR \nstrle \eqSE$ it suffices to observe
			that the programs $\prga = \set{\atma., \atmb.}$ and $\prgb =
			\set{\atma., \atmb \lpif \atma.}$ are \SE-equivalent but not
			\SMR-equivalent.
			
		\item It follows from Corollary~\ref{cor:entailment vs equivalence
			comparison} that we can instead prove that $\entSMR \strle \entRMR$ and
			$\eqRMR \nstrle \eqSMR$.

			To show the former, take some programs $\prga$, $\prgb$ such that $\prga
			\entRMR \prgb$. It follows that $\forall \rlb \in \prgb \, \exists \rla
			\in \prga : \modre{\rla} \subseteq \modre{\rlb}$ and, due to
			Lemma~\ref{lemma:re-entailment implies se-entailment}, $\forall \rlb \in
			\prgb \, \exists \rla \in \prga : \modse{\rla} \subseteq \modse{\rlb}$.
			Consequently, $\prga \entSMR \prgb$.

			As for the latter, it suffices to observe that the programs $\prga =
			\set{\lpnot \atm.}$ and $\prgb = \set{\lpif \atm.}$ are \SMR-equivalent
			but not \RMR-equivalent.

		\item According to Lemma~\ref{lemma:entailment vs equivalence comparison},
			it suffices to show that $\eqRE \nstrle \eqSMR$ and $\eqSMR \nstrle
			\eqRE$. The former follows from the fact that the programs $\prga =
			\set{\lpnot \atm.}$ and $\prgb = \set{\lpif \atm.}$ are \SMR-equivalent
			but not \RE-equivalent. The latter can be verified by observing that
			though the programs $\prga = \set{\atma., \atmb.}$ and $\prgb =
			\set{\atma., \atmb \lpif \atma.}$ are \RE-equivalent, they are not
			\SMR-equivalent.

		\item According to Lemma~\ref{lemma:entailment vs equivalence comparison},
			it suffices to show that $\eqRE \nstrle \eqSR$ and $\eqSR \nstrle
			\eqRE$. The former follows from the fact that the programs $\prga =
			\set{\lpnot \atm.}$ and $\prgb = \set{\lpif \atm.}$ are \SR-equivalent
			but not \RE-equivalent. The latter can be verified by observing that
			though the programs $\prga = \set{\atma., \atmb.}$ and $\prgb =
			\set{\atma., \atmb \lpif \atma.}$ are \RE-equivalent, they are not
			\SR-equivalent.

		\item According to Lemma~\ref{lemma:entailment vs equivalence comparison},
			it suffices to show that $\eqRMR \nstrle \eqSR$ and $\eqSR \nstrle
			\eqRMR$. The former follows from the fact that the programs $\prga =
			\set{\lpnot \atm.}$ and $\prgb = \set{\lpif \atma.}$ are \SR-equivalent
			but not \RMR-equivalent. The latter can be verified by observing that
			though the programs $\prga = \set{\atma.}$ and $\prgb = \set{\atma.,
			\atma \lpif \atmb.}$ are \RMR-equivalent, they are not \SR-equivalent.
			\qedhere
	\end{textenum}
\end{proof}

\section{Proofs: Exception-Based Rule Update Operators}

\label{app:exception-based rule update}


\subsection{Conflicts Between Sets of \sRE-Models}

\label{app:rules:re conflict}


\begin{definition}
	[Additional Notation]
	Let $\tri \in \tris$. Given an atom $\atm$, we say that $\tri$ is an
	\emph{\RE-model of $\atm$} if $\tri(\atm) = \tr$. We say that $\tri$ is an
	\emph{\RE-model of $\lpnot \atm$} if $\tri(\atm) = \fa$. We denote the set
	of all \RE-models of a literal $\lit$ by $\modre{\lit}$. Given a set of
	literals $\slit$, we say that $\tri$ is an \emph{\RE-model of $\slit$} if
	$\tri$ is an \RE-model of all literals in $\slit$. We denote the set of all
	\RE-models of $\slit$ by $\modre{\slit}$. Given a sequence of rule bases
	$\drb = \seq{\rb_\lia}_{\lia < \lng}$, we define $\modrer{\drb} =
	\seq{\modrer{\rb_\lia}}_{\lia < \lng}$.
\end{definition}


\begin{proposition}
	\label{prop:re:undefined atom vs body and head}
	Let $\rl$ be a rule, $\atm$ an atom and $\tri = \twiab \in \tris$. Then
	$\tpl{\twia \setminus \set{\atm}, \twib \cup \set{\atm}}$ is not an
	\RE-model of $\rl$ if and only if the following holds:
	\begin{textenum}[1.]
		\item Neither $\atm$ nor $\lpnot \atm$ belongs to $\brl$;
		\item $\tri$ is an \RE-model of $\brl$;
		\item $\tri$ is not an \RE-model of any literal from $\hrl \setminus
			\set{\atm, \lpnot \atm}$.
	\end{textenum}
\end{proposition}
\begin{proof}
	We prove the direct implication first. From the assumption it follows that
	$\rl^{\twib \cup \set{\atm}}$ is different from $\ctau$. This has two
	consequences. First, $\hrl^-$ is included in $\twib \cup \set{\atm}$, so all
	atoms from $\hrl^-$, except possibly $\atm$, belong to $\twib$, and thus
	$\tri$ is not an \RE-model of any default literal from $\hrl \setminus
	\set{\lpnot \atm}$. The second consequence is that $\brl^- \cap (\twib \cup
	\set{\atm})$ is empty. Hence, $\lpnot \atm$ does not belong to $\brl$.
	Furthermore, $\brl^- \cap \twib$ must also be empty, so we can conclude that
	$\tri$ is an \RE-model of all default literals from $\brl$. It also follows
	from the assumption that $\twia \setminus \set{\atm}$ contains $\brl^+$ but
	does not contain any atom from $\hrl^+$. As a consequence, $\atm$ does not
	belong to $\brl$ and we can also conclude that $\tri$ is an \RE-model of all
	atoms from $\brl$ and $\tri$ is not an \RE-model of any atom from $\hrl
	\setminus \set{\atm}$. Thus, $\tri$ is an \RE-model of $\brl$ and it is not
	an \RE-model of any literal from $\hrl \setminus \set{\atm, \lpnot \atm}$.

	As for the converse implication, we need to prove that $\twia \setminus
	\set{\atm}$ is not a model of $\rl^{\twib \cup \set{\atm}}$. We first need
	to show that $\rl^{\twib \cup \set{\atm}}$ is equal to the rule $\hrl^+
	\lpif \brl^+$. This holds if $\hrl^-$ is included in $\twib \cup \set{\atm}$
	and $\brl^-$ is disjoint with $\twib \cup \set{\atm}$. Since $\tri$ is an
	\RE-model of $\brl$, we can conclude that the set $\brl^-$ is disjoint with
	$\twib$ which, together with the assumption that $\lpnot \atm$ does not
	belong to $\brl$, implies that $\brl^-$ is disjoint with $\twib \cup
	\set{\atm}$. We also know that $\tri$ is not an \RE-model of any literal
	from $\hrl \setminus \set{\atm, \lpnot \atm}$, so we can conclude that
	$\hrl^- \setminus \set{\atm}$ is included in $\twib$. Thus, $\hrl^-$ is
	included in $\twib \cup \set{\atm}$ and we proved that $\rl^{\twib \cup
	\set{\atm}}$ is equal to the rule $\hrl^+ \lpif \brl^+$. It remains to
	show that $\twia \setminus \set{\atm}$ includes $\brl^+$ and that it does
	not contain any atom from $\hrl^+$. We know that $\tri$ is an \RE-model of
	$\brl$, so $\twia$ includes $\brl^+$. Also, since $\atm$ does not belong
	to $\brl$, $\twia \setminus \set{\atm}$ also includes $\brl^+$. Finally,
	we know that $\tri$ is not an \RE-model of any atom from $\hrl^+ \setminus
	\set{\atm}$, so $\twia$ does not contain any atom from $\hrl^+ \setminus
	\set{\atm}$. This implies that $\twia \setminus \set{\atm}$ does not
	contain any atom from $\hrl^+$.
\end{proof}

\begin{corollary}
	\label{cor:re:support}
	Let $\rl$ be a rule, $\atm$ an atom and $\twib$ an interpretation with $\atm
	\in \twib$. If $\tpl{\twib, \twib}$ is an \RE-model of $\rl$ but $\tpl{\twib
	\setminus \set{\atm}, \twib}$ is not, then $\atm \in \hrl$ and $\twib \ent
	\brl$.
\end{corollary}
\begin{proof}
	It follows immediately from Proposition~\ref{prop:re:undefined atom vs body
	and head} that $\twib \ent \brl$. Furthermore, by the definition of
	\RE-model, $\twib$ is a model of $\rl^\twib$ while $\twib \setminus
	\set{\atm}$ is not. Hence, $\twib$ contains some atom from $\hrl$ that is
	not contained in $\twib \setminus \set{\atm}$. This atom can only be $\atm$.
\end{proof}

\begin{proposition}
	\label{prop:re:rule in context}
	Let $\rl$ be a rule, $\stri = \modre{\rl}$, $\atm$ an atom, $\twib$ an
	interpretation and $\val$ a truth value. Then $\stri^\twib(\atm) = \val$ if and
	only if the following holds:
	\begin{textenum}[1.]
		\item Neither $\atm$ nor $\lpnot \atm$ belongs to $\brl$;
		\item $\twib$ is a model of $\brl$;
		\item $\twib$ is not a model of any literal from $\hrl \setminus
			\set{\atm, \lpnot \atm}$;
		\item One of the following conditions holds:
			\begin{textenum}[(a)]
				\item $\val$ is $\tr$ and $\hrl \cap \set{\atm, \lpnot \atm} =
					\set{\atm}$, or
				\item $\val$ is $\fa$ and $\hrl \cap \set{\atm, \lpnot \atm} =
					\set{\lpnot \atm}$.
			\end{textenum}
	\end{textenum}
\end{proposition}
\begin{proof}
	We focus on the direct implication first. Assume that $\val = \tr$. Then
	$\tpl{\twib \cup \set{\atm}, \twib \cup \set{\atm}}$ is an \RE-model of
	$\rl$ although both $\tpl{\twib \setminus \set{\atm}, \twib \cup
	\set{\atm}}$ and $\tpl{\twib \setminus \set{\atm}, \twib \setminus
	\set{\atm}}$ are not. By Proposition~\ref{prop:re:undefined atom vs body and
	head} and Lemma~\ref{lemma:reduct} we can conclude that the first three of
	the properties that we need to prove are satisfied. It remains to show that
	$\hrl \cap \set{\atm, \lpnot \atm} = \set{\atm}$, i.e. that $\atm$ belongs
	to $\hrl^+$ but it does not belong to $\hrl^-$. To see that the former
	holds, note that $\twib \setminus \set{\atm}$ is not a model of $\rl^{\twib
	\cup \set{\atm}}$, so $\twib \setminus \set{\atm}$ includes $\brl^+$ and
	it does not contain any atom from $\hrl^+$. Since we know that $\twib \cup
	\set{\atm}$ \emph{is} a model of $\rl^{\twib \cup \set{\atm}}$, it must be
	the case that $\twib \cup \set{\atm}$ contains an atom from $\hrl^+$. This
	atom can only be $\atm$. Finally, if $\atm$ were a member of $\hrl^-$,
	then $\rl^{\twib \setminus \set{\atm}}$ would coincide with $\ctau$, so
	$\tpl{\twib \setminus \set{\atm}, \twib \setminus \set{\atm}}$ would be an
	\RE-model of $\rl$, contrary to the assumption.

	Now assume that $\val = \fa$. Then $\tpl{\twib \setminus \set{\atm}, \twib
	\setminus \set{\atm}}$ is an \RE-model of $\rl$ although both $\tpl{\twib
	\setminus \set{\atm}, \twib \cup \set{\atm}}$ and $\tpl{\twib \cup
	\set{\atm}, \twib \cup \set{\atm}}$ are not. By
	Proposition~\ref{prop:re:undefined atom vs body and head} and
	Lemma~\ref{lemma:reduct} we can conclude that the first three of
	the properties that we need to prove are satisfied. It remains to show that
	$\hrl \cap \set{\atm, \lpnot \atm} = \set{\lpnot \atm}$, i.e. that $\atm$
	belongs to $\hrl^-$ but it does not belong to $\hrl^+$. To see that the
	former holds, note that by the assumption $\twib \setminus \set{\atm}$ is a
	model of $\rl^{\twib \setminus \set{\atm}}$ while it is not a model of
	$\rl^{\twib \cup \set{\atm}}$. Hence, $\rl^{\twib \setminus \set{\atm}}$
	must be equal to $\ctau$. We know that $\brl^-$ is disjoint
	with $\twib \cup \set{\atm}$, so it must also be disjoint with $\twib
	\setminus \set{\atm}$. Thus, there must exist some atom from $\hrl^-$ that is
	not contained in $\twib \setminus \set{\atm}$ while it was contained in
	$\twib \cup \set{\atm}$. This atom can only be $\atm$. Finally, if $\atm$
	were a member of $\hrl^+$, then $\twib \cup \set{\atm}$ would contain an
	atom from $\hrl^+$, so $\tpl{\twib \cup \set{\atm}, \twib \cup \set{\atm}}$
	would be an \RE-model of $\rl$, contrary to the assumption.

	Finally, we show by contradiction that $\val$ cannot be equal to $\un$.
	Suppose that $\val = \un$. It follows that $(\twib \cup \set{\atm}) \nent
	\rl^{\twib \cup \set{\atm}}$, $(\twib \setminus \set{\atm}) \ent \rl^{\twib
	\cup \set{\atm}}$ and $(\twib \setminus \set{\atm}) \nent \rl^{\twib
	\setminus \set{\atm}}$. Thus, since both $\rl^{\twib \setminus \set{\atm}}$
	and $\rl^{\twib \cup \set{\atm}}$ are different from $\ctau$, they
	must be identical and this is in conflict with our latter two conclusions.

	As for the converse implication, let $\tri^\tr = \tpl{\twib \cup \set{\atm},
	\twib \cup \set{\atm}}$, $\tri^\un = \tpl{\twib \setminus \set{\atm}, \twib
	\cup \set{\atm}}$ and $\tri^\fa = \tpl{\twib \setminus \set{\atm}, \twib
	\setminus \set{\atm}}$ First suppose that $\val$ is $\tr$ and $\hrl \cap
	\set{\atm, \lpnot \atm} = \set{\atm}$. We need to show that $\tri^\tr$ is an
	\RE-model of $\rl$ while both $\tri^\un$ and $\tri^\fa$ are not. The first
	property follows directly from the fact that $\atm$ belongs to $\hrl^+$ and
	$\tri^\tr$ is an \RE-model of $\atm$. The second property follows from
	Proposition~\ref{prop:re:undefined atom vs body and head} and
	Lemma~\ref{lemma:reduct}. To show that the third is also satisfied, note
	that since $\tri^\un$ is not an \RE-model of $\rl$, the rule $\rl^{\twib
	\cup \set{\atm}}$ coincides with the rule $\hrl^+ \lpif \brl^+$. This
	implies that $\brl^-$ is disjoint with $\twib \cup \set{\atm}$ and
	$\hrl^-$ is included in $\twib \cup \set{\atm}$. As a consequence,
	$\brl^-$ is also disjoint with $\twib \setminus \set{\atm}$. Moreover,
	from our assumptions we know that $\hrl \cap \set{\atm, \lpnot \atm} =
	\set{\atm}$, which means that $\atm$ does not belong to $\hrl^-$. Thus,
	$\hrl^-$ is included in $\twib \setminus \set{\atm}$. As a consequence,
	the rule $\rl^{\twib \setminus \set{\atm}}$ also coincides with the rule
	$\hrl^+ \lpif \brl^+$. Furthermore, since $\tri^\un$ is not an \RE-model
	of $\rl$, $\twib \setminus \set{\atm}$ is not a model of $\rl^{\twib \cup
	\set{\atm}}$. Since $\rl^{\twib \cup \set{\atm}} = \rl^{\twib \setminus
	\set{\atm}}$, we obtain that $\twib \setminus \set{\atm}$ is not a model
	of $\rl^{\twib \setminus \set{\atm}}$. Hence $\tri^\fa$ is not an
	\RE-model of $\rl$.

	Next, suppose that $\val$ is $\fa$ and $\hrl \cap \set{\atm, \lpnot \atm} =
	\set{\lpnot \atm}$. We need to show that $\tri^\fa$ is an \RE-model of $\rl$
	while both $\tri^\un$ and $\tri^\tr$ are not. The first property follows
	directly from the fact that $\atm$ belongs to $\hrl^-$ but does not belong
	to $\twib \setminus \set{\atm}$ because in this case $\rl^{\twib \setminus
	\set{\atm}}$ coincides with $\ctau$. The second property follows from
	Proposition~\ref{prop:re:undefined atom vs body and head} and
	Lemma~\ref{lemma:reduct}. To show that the third is also satisfied,
	note that since $\tri^\un$ is not an \RE-model of $\rl$, the rule
	$\rl^{\twib \cup \set{\atm}}$ coincides with the rule $\hrl^+ \lpif \brl^+$
	and $\twib \setminus \set{\atm}$ is not a model of $\rl^{\twib \cup
	\set{\atm}}$, i.e. $\twib \setminus \set{\atm}$ includes $\brl^+$ but does
	not contain any atom from $\hrl^+$. Thus, $\twib \cup \set{\atm}$ also
	includes $\brl^+$ and from our assumption that $\hrl \cap \set{\atm, \lpnot
	\atm} = \set{\lpnot \atm}$ we can conclude that $\atm$ does not belong to
	$\hrl^+$. Thus, $\twib \cup \set{\atm}$ does not contain any atom from
	$\hrl^+$ and, consequently, $\tri^\tr$ is not an \RE-model of $\rl$.
\end{proof}

\begin{proposition}
	\label{prop:re:conflict in context vs syntax}
	Let $\rla$ and $\rlb$ be non-disjunctive rules, $\stria = \modre{\rla}$,
	$\strib = \modre{\rlb}$, and $\twib$ an interpretation. Then $\stria
	\confl{\twib}{\atm} \strib$ if and only if for some $\lit \in \set{\atm,
	\lpnot \atm}$, $\hrla = \set{\lit}$, $\hrlb = \set{\lcmp{\lit}}$, $\twib$ is
	a model of both $\brla$ and $\brlb$, and $\brla$, $\brlb$ do not contain
	$\atm$ nor $\lpnot \atm$.
\end{proposition}
\begin{proof}
	Follows directly from Proposition~\ref{prop:re:rule in context}.
\end{proof}


\subsection{Syntactic Properties of \texorpdfstring{\terone-}{delta\_a-}Based Operators}

\label{app:rules:syntactic properties:erone}

\begin{definition}
	Let \ter{} be a local exception function, $\stria \subseteq \tris$ and
	$\sstri \subseteq \pws{\tris}$. We define $\aug{\stria, \sstri} = \stria
	\cup \bigcup_{\strib \in \sstri} \er(\stria, \strib)$ and extend this
	definition to sequences of sets of sets of three-valued interpretations
	inductively as follows: $\aug{\stria, \seq{}} = \stria$ and $\aug{\stria,
	\seq{\sstri_\lia}_{\lia < \lng + 1}} = \aug{\aug{\stria,
	\seq{\sstri_\lia}_{\lia < \lng}}, \sstri_\lng}$.
\end{definition}


\begin{proposition}
	\label{prop:aug vs uope}
	Let $\drb = \seq{\rb_\lia}_{\lia < \lng}$ be a sequence of rule bases,
	$\sstri_\lia = \modrer{\rb_\lia}$ for all $\lia < \lng$, and $\uopr$ a
	\ter-based rule update operator. Then, $\modrer{\biguopr \drb} = \set{
	\aug{\stri, \seq{\sstri_\lib}_{\lia < \lib < \lng}} | \lia < \lng \land
	\stri \in \sstri_\lia}$.
\end{proposition}
\begin{proof}
	Follows by induction on $\lng$.
\end{proof}


\begin{lemma}
	\label{lemma:syntactic properties:erone:support}
	Let $\stria \subseteq \tris$, $\sstri = \seq{\sstri_\lia}_{\lia < \lng}$
	where $\sstri_\lia \subseteq \pws{\tris}$ for all $\lia < \lng$, $\twib$ an
	interpretation and $\atm$ an atom. If $\tpl{\twib, \twib}$ belongs to
	$\augone{\stria, \sstri}$, but $\tpl{\twib \setminus \set{\atm}, \twib}$
	does not, then $\tpl{\twib, \twib}$ belongs to $\stria$.
\end{lemma}
\begin{proof}
	Follows by induction on $\lng$.
%
\end{proof}

\begin{proposition}
	\label{prop:syntactic properties:erone:support}
	Every \terone-based rule update operator respects support.
\end{proposition}
\begin{proof}
	Let $\uopr$ be some \terone-based rule update operator, pick some DLP $\dprg
	= \seq{\prg_\lia}_{\lia < \lng}$, suppose that $\twib$ is a stable model of
	$\biguopr \dprg$ and take some $\atm \in \twib$. We need to show that for
	some rule $\rl \in \all{\dprg}$, $\atm \in \hrl$ and $\twib \ent \brl$.
	Since $\twib$ is a stable model of $\biguopr \dprg$, we know that
	$\tpl{\twib, \twib}$ belongs to $\modre{\biguopr \dprg}$ and for all $\twia
	\subsetneq \twib$, $\twiab$ does not belong to $\modre{\biguopr \dprg}$. In
	particular, $\tpl{\twib \setminus \set{\atm}, \twib}$ does not belong to
	$\modre{\biguopr \dprg}$. Consequently, there is some $\strib \in
	\modrer{\biguopr \dprg}$ such that $\tpl{\twib \setminus \set{\atm}, \twib}$
	does not belong to $\strib$ although $\tpl{\twib, \twib}$ does. According to
	Proposition~\ref{prop:aug vs uope}, $\strib = \Aug{\stri,
	\seq{\modrer{\prg_\lib}}_{\lia < \lib < \lng}}$ where $\stria \in
	\modrer{\prg_\lia}$ for some $\lia < \lng$. Let $\rl$ be a rule from
	$\prg_\lia$ such that $\modre{\rl} = \stria$. Since $\tpl{\twib \setminus
	\set{\atm}, \twib}$ does not belong to $\strib$, it cannot belong to
	$\stria$ since $\stria$ is a subset of $\strib$. Also, by
	Lemma~\ref{lemma:syntactic properties:erone:support}, $\tpl{\twib,
	\twib}$ belongs to $\stria$. Thus, it follows from
	Corollary~\ref{cor:re:support} that $\atm \in \hrl$ and $\twib \ent
	\brl$.
\end{proof}

\begin{lemma}
	\label{lemma:syntactic properties:erone:fact update}
	Let $\dprg$ be a finite sequence of sets of facts and $\lit$ a literal.
	Then,
	\[
		\augone{\modre{\lit.}, \modrer{\dprg}} = \begin{cases}
			\tris & (\lcmp{\lit}.) \in \all{\dprg} \enspace; \\
			\modre{\lit.} & \text{otherwise} \enspace.
		\end{cases}
	\]
\end{lemma}
\begin{proof}
	Follows using Proposition~\ref{prop:re:conflict in context vs
	syntax} by induction on the length of $\dprg$.
\end{proof}

\begin{corollary}
	\label{cor:syntactic properties:erone:fact update}
	Let $\dprg = \seq{\prg_\lia}_{\lia < \lng}$ be a finite sequence of sets of
	facts and $\uopr$ a \terone-based rule update operator. Then,
	$\modrer{\biguopr \dprg} \cup \set{\tris} = \set{ \modre{\lit.} | \exists
	\lia < \lng : (\lit.) \in \prg_\lia \land (\forall \lib : \lia < \lib < \lng
	\mlthen (\lcmp{\lit}.) \notin \prg_\lib) } \cup \set{\tris}$.
\end{corollary}
\begin{proof}
	By Proposition~\ref{prop:aug vs uope}, $\modrer{\dprg} = \set{
	\aug{\modre{\rl}, \seq{\modrer{\prg_\lib}}_{\lia < \lib < \lng}} | \lia <
	\lng \land \rl \in \prg_\lia }$, which can also be written as $\set{
	\aug{\modre{\lit.}, \seq{\modrer{\prg_\lib}}_{\lia < \lib < \lng}} |
	\lia < \lng \land (\lit.) \in \prg_\lia }$. Furthermore, due to
	Lemma~\ref{lemma:syntactic properties:erone:fact update}, we can
	equivalently write this as
	\begin{multline*}
		\Set{
			\tris
			|
			\exists \lia, \lib, \lit :
			\lia < \lib < \lng
			\land (\lit.) \in \prg_\lia \land (\lcmp{\lit}.) \in \prg_\lib
		} \\
		{}\cup \Set{
			\modre{\lit.}
			|
			\exists \lia < \lng :
			(\lit.) \in \prg_\lia
			\land (
				\forall \lib :
				\lia < \lib < \lng \mlthen (\lcmp{\lit}.) \notin \prg_\lib
			)
		}
		\enspace.
	\end{multline*}
	This directly implies the desired conclusion.
\end{proof}

\begin{proposition}
	\label{prop:syntactic properties:erone:fact update}
	Every \terone-based rule update operator respects fact update.
\end{proposition}
\begin{proof}
	Let $\dprg = \seq{\prg_\lia}_{\lia < \lng}$ be a finite sequence of
	consistent sets of facts, $\twib$ the interpretation
	\[
		\Set{
			\atm |
			\exists \lia < \lng : (\atm.) \in \prg_\lia
			\land
			(\forall \lib :
				\lia < \lib < \lng \mlthen (\lpnot \atm.) \notin \prg_\lib
			)
		}
	\]
	and $\uopr$ a \terone-based rule update operator. We need to show that
	$\twib$ is the unique stable model of $\biguopr \dprg$.

	We start by proving that $\tpl{\twib, \twib}$ belongs to $\modre{\biguopr
	\dprg}$. Pick some $\stri \in \modrer{\biguopr \dprg}$. By
	Corollary~\ref{cor:syntactic properties:erone:fact update} we know that
	$\stri$ is either $\tris$, or it is equal to $\modre{\lit.}$ where
	\[
		\exists \lia < \lng : (\lit.) \in \prg_\lia
		\land
		(\forall \lib :
			\lia < \lib < \lng \mlthen (\lcmp{\lit}.) \notin \prg_\lib
		)
		\enspace.
	\]
	In the former case it trivially holds that $\tpl{\twib, \twib}$ belongs to
	$\stri = \tris$. Now suppose that $\lit$ is an atom $\atm$. Then, by its
	definition, $\twib$ contains $\atm$, so $\tpl{\twib, \twib}$ belongs to
	$\modre{\atm.} = \stri$. On the other hand, if $\lit$ is a default literal
	$\lpnot \atm$, then the fact $(\atm.)$ does not belong to $\prg_\lia$
	because $\prg_\lia$ is consistent, and it also does not belong to
	$\prg_\lib$ for any $\lib$ with $\lia < \lib < \lng$. So $\atm$ does not
	belong to $\twib$ and, hence, $\tpl{\twib, \twib}$ belongs to $\modre{\lpnot
	\atm.} = \stri$.

	Now suppose that $\twiab$ belongs to $\modre{\biguopr \dprg}$ and
	take some $\atm \in \twib$. Then,
	\[
		\exists \lia < \lng : (\atm.) \in \prg_\lia
		\land
		(\forall \lib :
			\lia < \lib < \lng \mlthen (\lpnot \atm.) \notin \prg_\lib
		)
		\enspace,
	\]
	so, by Corollary~\ref{cor:syntactic properties:erone:fact update},
	$\modre{\atm.}$ belongs to $\modrer{\biguopr \dprg}$. Since $\twiab$ belongs
	to $\modre{\biguopr \dprg}$, it must also belong to $\modre{\atm.}$. Thus,
	$\atm$ belongs to $\twia$ and as the choice of $\atm \in \twib$ was
	arbitrary, we can conclude that $\twia = \twib$. As a consequence, $\twib$
	is indeed a stable model of $\biguopr \dprg$.

	It remains to prove that $\twib$ is the only stable model of $\biguopr
	\dprg$. Suppose that $\twib'$ is a stable model of $\biguopr \dprg$ and take
	some $\atm \in \twib$. We will show that $\atm$ belongs to $\twib'$. We know
	that
	\[
		\exists \lia < \lng : (\atm.) \in \prg_\lia
		\land
		(\forall \lib :
			\lia < \lib < \lng \mlthen (\lpnot \atm.) \notin \prg_\lib
		)
		\enspace,
	\]
	so, by Corollary~\ref{cor:syntactic properties:erone:fact update},
	$\modre{\atm.}$ belongs to $\modrer{\biguopr \dprg}$. Since $\twib'$ is a
	stable model of $\biguopr \dprg$, $\tpl{\twib', \twib'}$ belongs to
	$\modre{\biguopr \dprg}$ and, consequently, also to $\modre{\atm.}$. Thus,
	$\atm$ must belong to $\twib'$. Now take some atom $\atm$ that does not belong
	to $\twib$. We will show that $\atm$ does not belong to $\twib'$ either.
	There are two cases to consider:
	\begin{textenum}[a)]
		\item If $(\atm.)$ does not belong to $\prg_\lia$ for all $\lia < \lng$,
			then it can be seen that $\tpl{\twib' \setminus \set{\atm}, \twib'}$
			belongs to all elements of $\modrer{\biguopr \dprg}$. Thus, since
			$\twib'$ is a stable model of $\biguopr \dprg$, $\twib' \setminus
			\set{\atm} = \twib'$ and, consequently, $\atm$ does not belong to
			$\twib'$.
			
		\item If $(\atm.)$ belongs to $\prg_{\lia_0}$ for some $\lia_0 < \lng$ and
			whenever $(\atm.)$ belongs to $\prg_\lia$ for some $\lia$, there is some
			$\lib$ with $\lia < \lib < \lng$ such that $(\lpnot \atm.)$ belongs to
			$\prg_\lib$, then there must exist some $\lia_1$ such that $(\lpnot
			\atm.)$ belongs to $\prg_{\lia_1}$ and for all $\lib$ with $\lia_1 <
			\lib < \lng$, $(\atm.)$ does not belong to $\prg_\lib$. Consequently,
			$\modre{\lpnot \atm.}$ belongs to $\modrer{\biguopr \dprg}$. Thus, since
			$\tpl{\twib', \twib'}$ belongs to $\modre{\biguopr \dprg}$, it follows
			that $\atm$ cannot belong to $\twib'$. \qedhere
	\end{textenum}
\end{proof}

\begin{proposition}
	\label{prop:if beta stable then ju}
	Let $\dprg = \seq{\prga, \prgu}$ be a dynamic logic program, $\uopr$ a
	\terone-based rule update operator and $\twib$ an interpretation. If $\twib$
	is a stable model of $\biguopr \dprg$, then $\twib$ is a \JU-model of
	$\dprg$.
\end{proposition}
\begin{proof}
	From the assumption we can conclude that $\tpl{\twib, \twib}$ is an
	\RE-model of $\prga \uopr \prgu$ and for every $\twia \subsetneq \twib$,
	$\twiab$ is not an \RE-model of $\prga \uopr \prgu$.

	We need to show that $\twib$ is a minimal model of the program $\prgb = [
	\all{\dprg} \setminus \rej{\dprg, \twib} ]^\twib$. First we prove that
	$\twib$ is a model of $\prgb$. Take some rule $\rla' \in \prgb$ and let
	$\rla$ be a rule from $[ \all{\dprg} \setminus \rej{\dprg, \twib} ]$ such
	that $\rla' = \rla^\twib$. We consider two cases:
	\begin{textenum}[a)]
		\item If $\rla$ belongs to $\prgu$, then since $\tpl{\twib, \twib}$ belongs
			to $\modre{\prga \uopr \prgu}$ and $\modrer{\prga \uopr \prgu}$ contains
			$\modre{\rla}$, $\tpl{\twib, \twib}$ must also belong to $\modre{\rla}$.
			Thus, $\twib$ is a model of $\rla$ and consequently also a model
			of $\rla' = \rla^\twib$.

		\item If $\rla$ belongs to $\prga \setminus \rej{\dprg, \twib}$, then
			since $\tpl{\twib, \twib}$ belongs to $\modre{\prga \uopr \prgu}$ and
			$\modrer{\prga \uopr \prgu}$ contains $\modre{\rla} \cup \bigcup_{\rlb
			\in \prgu} \erone(\modre{\rla}, \modre{\rlb})$, $\tpl{\twib, \twib}$
			must also belong to $\modre{\rla} \cup \bigcup_{\rlb \in \prgu}
			\erone(\modre{\rla}, \modre{\rlb})$. Suppose first that for some rule
			$\rlb \in \prgu$, $\tpl{\twib, \twib}$ belongs to the set
			$\erone(\modre{\rla}, \modre{\rlb})$. This implies that $\modre{\rla}
			\confl{\twib}{\atm} \modre{\rlb}$ for some atom $\atm$ and by
			Proposition~\ref{prop:re:conflict in context vs syntax} we can conclude
			that $\rla$ belongs to $\rej{\dprg, \twib}$, contrary to the assumption.
			Thus, $\tpl{\twib, \twib}$ does not belong to the set $\bigcup_{\rlb \in
			\prgu} \erone(\modre{\rla}, \modre{\rlb})$ and, consequently, it belongs
			to $\modre{\rla}$. Hence, $\twib$ is a model of $\rla$ and,
			consequently, it is also a model of $\rla' = \rla^\twib$.
	\end{textenum}

	It remains to prove that $\twib$ is a minimal model of $\prgb$. Take some
	model $\twia$ of $\prgb$ such that $\twia$ is a subset of $\twib$. We need
	to show that $\twia = \twib$. In the following we will show that $\twiab$
	is a member of the set $\modre{\prga \uopr \prgu}$ which, together with the
	assumption that $\twib$ is a stable model of $\prga \uopr \prgu$, implies
	that $\twia = \twib$.

	So in order to finish the proof, take some set $\stri$ from $\modrer{\prga
	\uopr \prgu}$. We need to show that $\twiab$ belongs to $\stri$. Recall
	that $ \modrer{\prga \uopr \prgu} = \set{ \modre{\rla} \cup \bigcup_{\rlb
	\in \prgu} \erone(\modre{\rla}, \modre{\rlb}) | \rla \in \prga } \cup
	\modrer{\prgu}$. If $\stri$ belongs to $\modrer{\prgu}$, then there is some
	rule $\rla \in \prgu$ such that $\stri = \modre{\rla}$. Moreover,
	$\rla^\twib$ belongs to $\prgb$, so $\twia$ is a model of $\rla^\twib$. It
	then follows that $\twiab$ is an \RE-model of $\rla$, i.e.\ that $\twiab$
	belongs to $\stri$, as we wanted to show.

	The remaining case is when for some $\rla \in \prga$, $\stri = \modre{\rla}
	\cup \bigcup_{\rlb \in \prgu} \erone(\modre{\rla}, \modre{\rlb})$. Suppose
	that $\twiab$ does not belong to $\modre{\rla}$. Then $\twia$ is not a model
	of $\rla^\twib$. Since $\twia$ is a subset of $\twib$, we can conclude from
	this that
	\begin{equation} \label{eq:proof:if beta stable then ju:2}
		\twib \ent \brla \enspace.
	\end{equation}
	Furthermore, from our assumption that $\twia$ is a model of $\prgb$ it then
	follows that $\rla^\twib$ does not belong to $\prgb$ and, consequently,
	$\rla$ belongs to $\rej{\dprg, \twib}$. So there must be some rule $\rlb \in
	\prgu$ such that $\hrlb = \lcmp{\hrla}$ and $\twib \ent \brlb$. Since we
	know from the previous part of the proof that $\twib$ is a model of $\prgb$,
	we can conclude that $\twib \ent \hrlb$, so $\twib \nent \hrla$.

	Thus, it follows from \eqref{eq:proof:if beta stable then ju:2} that $\twib$
	is not a model of $\rla$, so $\tpl{\twib, \twib}$ is not an \RE-model of
	$\rla$. But since $\twib$ is a stable model of $\prga \uopr \prgu$,
	$\tpl{\twib, \twib}$ must belong to $\erone(\modre{\rla}, \modre{\rlb'})$ for
	some $\rlb' \in \prgu$ and from the definition of $\erone(\cdot, \cdot)$ we
	obtain that $\twiab$ also belongs to $\erone(\modre{\rla},
	\modre{\rlb'})$. This implies that $\twiab$ belongs to $\stri$.
\end{proof}

\begin{proposition}
	\label{prop:if ju then beta stable}
	Let $\dprg = \seq{\prga, \prgu}$ be a DLP free of local cycles, $\uopr$ a
	\terone-based rule update operator and $\twib$ an interpretation. If $\twib$
	is a \JU-model of $\dprg$, then $\twib$ is a stable model of $\biguopr
	\prg$.
\end{proposition}
\begin{proof}
	Suppose that $\twib$ is a justified update model of $\seq{\prga, \prgu}$.
	Then it is a minimal model of the program $\prgb = [ \all{\dprg} \setminus
	\rej{\dprg, \twib} ]^\twib$. We need to prove that $\tpl{\twib, \twib}$
	is an \RE-model of $\prga \uopr \prgu$ and for every $\twia \subsetneq
	\twib$, $\twiab$ is not an \RE-model of $\prga \uopr \prgu$. In order to
	show that $\tpl{\twib, \twib}$ is an \RE-model of $\prga \uopr \prgu$,
	recall that $\modrer{\prga \uopr \prgu} = \Set{ \modre{\rla} \cup
	\bigcup_{\rlb \in \prgu} \erone(\modre{\rla}, \modre{\rlb}) | \rla \in
	\prga } \cup \modrer{\prgu}$ and take some set $\stri$ from $\modrer{\prga
	\uopr \prgu}$. If $\stri$ belongs to $\modrer{\prgu}$, then there is a rule
	$\rla \in \prgu$ such that $\stri = \modre{\rla}$. Also, $\rla^\twib$
	belongs to $\prgb$, so $\twib$ is a model of $\rla^\twib$. Consequently,
	$\tpl{\twib, \twib}$ belongs to $\stri$.

	Now suppose that for some $\rla$ from $\prga$, $\stri = \modre{\rla} \cup
	\bigcup_{\rlb \in \prgb} \erone(\modre{\rla}, \modre{\rlb})$. If
	$\tpl{\twib, \twib}$ does not belong to $\modre{\rla}$, then $\twib$ is not
	a model of $\rla^\twib$, so $\rla$ belongs to $\rej{\dprg, \twib}$. So there
	exists a rule $\rlb$ from $\prgu$ such that $\hrlb = \lcmp{\hrla}$ and
	$\twib \ent \brlb$. The previous conclusions, together with the fact that
	$\rla$ and $\rlb$ are not local cycles, allow us to use
	Proposition~\ref{prop:re:conflict in context vs syntax} and conclude that
	$\modre{\rla} \confl{\twib}{\atm} \modre{\rlb}$ holds for some atom $\atm$.
	Hence, $\tpl{\twib, \twib}$ belongs to $\erone(\modre{\rla}, \modre{\rlb})$,
	and consequently also to $\stri$.

	Now suppose that $\twiab$ belongs to $\modre{\rla}$. We will show that
	$\twia$ is a model of $\prgb$, which implies that $\twia = \twib$ because
	$\twib$ is by assumption a minimal model of $\prgb$. Take some rule $\rla'$
	from $\prgb$ and suppose that $\rla' = \rla^\twib$ for some $\rla \in
	[\all{\dprg} \setminus \rej{\dprg, \twib}]$. If $\rla$ belongs to $\prgu$,
	then $\modre{\rla}$ belongs to $\modrer{\prga \uopr \prgu}$. Consequently,
	$\twiab$ belongs to $\modre{\rla}$, so $\twia$ is a model of $\rla'$.

	The final case to consider is when $\rla$ belongs to $\prga$. We will prove
	by contradiction that $\twia$ is a model of $\rla'$. So suppose that $\twia$
	is not a model of $\rla'$. Then $\twiab$ is not an \RE-model of $\rla$.
	However, since by assumption $\twiab$ belongs to $\modre{\prga \uopr
	\prgu}$, it must also belong to the set $\modre{\rla} \cup \bigcup_{\rlb \in
	\prgu} \erone(\modre{\rla}, \modre{\rlb})$. We have already shown that it is
	not a member of $\modre{\rla}$, so there must exist some $\rlb \in \prgu$
	such that $\twiab$ belongs to $\erone(\modre{\rla}, \modre{\rlb})$. Thus,
	$\modre{\rla} \confl{\twib}{\atm} \modre{\rlb}$ holds for some atom $\atm$.
	We can use Proposition~\ref{prop:re:conflict in context vs syntax} to
	conclude that $\hrlb = \lcmp{\hrla}$ and $\twib \ent \brlb$. Hence, $\rla$
	belongs to $\rej{\dprg, \twib}$, contrary to our assumption.
\end{proof}

\begin{proof}
	[\textbf{Proof of Theorem~\ref{thm:syntactic properties:erone}}]
	\label{proof:thm:re:syntactic properties:erone}
	Follows from Propositions~\ref{prop:syntactic properties:erone:support},
	\ref{prop:syntactic properties:erone:fact update}, \ref{prop:if beta stable
	then ju}, \ref{prop:if ju then beta stable} and the fact that the
	\JU-semantics has these properties.
\end{proof}

\subsection{Syntactic Properties of \texorpdfstring{\tertwo-}{delta\_b-} and
\texorpdfstring{\terthree-}{delta\_c-}Based Operators}

\label{app:rules:syntactic properties:ertwo}

\begin{definition}
	We say that $\stri \subseteq \tris$ is \emph{\RE-rule-expressible} if there
	exists a rule $\rl$ such that $\stri = \modre{\rl}$.
\end{definition}

\begin{lemma}
	\label{lemma:aug delta value}
	Let $\stria \subseteq \tris$ be \RE-rule-expressible, $\sstri$ a set of
	\RE-rule-expressible sets of three-valued interpretations, $\twib$ an
	interpretation, $\atm$ an atom, $\val_0$ a truth value and $\er \in
	\set{\ertwo, \erthree}$. If $\aug{\stria, \sstri}^{\twib}(\atm) = \val_0$,
	then $\stria^{\twib}(\atm) = \val_0$.
\end{lemma}
\begin{proof}
	Suppose that $\aug{\stria, \sstri}^{\twib}(\atma) = \val_0$. By the
	definition we then obtain that for all truth values $\val$,
	\begin{equation}
		\label{eq:proof:aug delta value:1}
		\twib[\val/\atma] \in \aug{\stria, \sstri}
			\text{ if and only if } \val = \val_0 \enspace.
	\end{equation}
	If the interpretation $\twib[\val_0/\atma]$ belongs to $\stria$, then we can
	use \eqref{eq:proof:aug delta value:1} together with the fact that $\stria$
	is a subset of $\aug{\stria, \sstri}$ to conclude that
	$\stria^{\twib}(\atma) = \val_0$ and our proof ends.

	So suppose that $\twib[\val_0/\atma]$ does not belong to $\stria$. Then it
	follows from \eqref{eq:proof:aug delta value:1} and from the fact that
	$\stria$ is a subset of $\aug{\stria, \sstri}$ that the interpretations
	$\twib[\tr/\atma]$, $\twib[\un/\atma]$ and $\twib[\fa/\atma]$ do not belong
	to $\stria$. Thus, since $\twib[\val_0/\atma]$ belongs to $\aug{\stria,
	\sstri}$, there must exist some $\strib \in \sstri$ such that
	$\twib[\val_0/\atma] = \tpl{\twia, \twic}$ belongs to $\er(\stria, \strib)$.
	In other words, there exists an atom $\atmb$ and an interpretation $\twib'$
	such that $\twia \subseteq \twib' \subseteq \twic$ and
	$\stria^{\twib'}(\atmb) \neq \strib^{\twib'}(\atmb)$. Note that $\twib
	\setminus \set{\atma} \subseteq \twia \subseteq \twib' \subseteq \twic
	\subseteq \twib \cup \set{\atma}$. Thus, $\tpl{\twib', \twib'} =
	\twib[\val_1/\atma]$ for some $\val_1 \in \set{\tr, \fa}$. We distinguish
	two cases:
	\begin{textenum}[a)]
		\item If $\atma \neq \atmb$, then by the definition of \tertwo{} and
			\terthree{} we obtain that $\twib[\val_1/\atma], \twib[\un/\atma] \in
			\er(\stria, \strib)$.

		\item If $\atma = \atmb$, let $\val_2 = \tr$ is $\val_1 = \fa$ and $\val_2
			= \fa$ if $\val_1 = \tr$. It follows that
			$\stria^{\twib[\val_2/\atma]}(\atma) \neq
			\strib^{\twib[\val_2/\atma]}(\atma)$ and by the definition of \tertwo{}
			and \terthree{} we obtain that $\twib[\val_1/\atma], \twib[\val_2/\atma]
			\in \er(\stria, \strib)$.
	\end{textenum}
	In either case, it is not possible fo $\aug{\stria, \sstri}^{\twib}(\atma)$
	to be defined, a conflict with our assumption.
\end{proof}

\begin{proposition}
	[Exception Independence for Rules]
	\label{prop:exception independence rule:delta}
	Let $\stri \subseteq \tris$ be \RE-rule-expressible by a non-disjunctive
	rule, $\sstria$, $\sstrib$ be sets of \RE-rule-expressible sets of
	three-valued interpretations and $\er \in \set{\ertwo, \erthree}$. Then,
	$\aug{\aug{\stri, \sstria}, \sstrib} = \aug{\stri, \sstria \cup \sstrib}$.
\end{proposition}
\begin{proof}
	By applying the definition of $\aug{\cdot, \cdot}$ we can see that our
	goal is to show that the set
	\begin{equation}
		\label{eq:proof:exc independence rule:delta:lhs}
		\aug{\stria, \sstria} \cup
			\bigcup_{\strib \in \sstrib}
			\er(\aug{\stria, \sstria}, \strib)
	\end{equation}
	is equal to the set $\stria \cup \bigcup_{\strib \in \sstria \cup \sstrib}
	\er(\stria, \strib)$ which can also be written as
	\begin{equation}
		\label{eq:proof:exc independence rule:delta:rhs}
		\aug{\stria, \sstria} \cup
			\bigcup_{\strib \in \sstrib} \er(\stria, \strib)
			\enspace.
	\end{equation}

	First suppose that some $\tri \in \tris$ belongs to \eqref{eq:proof:exc
	independence rule:delta:lhs}. If $\tri$ belongs to $\aug{\stria, \sstria}$,
	then it directly follows that $\tri$ also belongs to \eqref{eq:proof:exc
	independence rule:delta:rhs}. So suppose that $\tri$ belongs to
	$\er(\aug{\stria, \sstria}, \strib)$ for some $\strib \in \sstrib$. By the
	definition of \tertwo{} and \terthree{} we obtain that there exists some
	atom $\atm$ and some interpretation $\twib$ with certain properties relative
	to $\tri$ such that $\aug{\stria, \sstria}^{\twib}(\atm) \neq
	\strib^{\twib}(\atm)$. By Lemma \ref{lemma:aug delta value} we then conclude
	that $\aug{\stria, \sstria}^{\twib}(\atm) = \stria^{\twib}(\atm)$. Thus,
	$\tri$ also belongs to $\er(\stria, \strib)$ and, consequently, also to the
	set \eqref{eq:proof:exc independence rule:delta:rhs}.

	Now suppose that some $\tri = \tpl{\twia, \twic} \in \tris$ belongs to
	\eqref{eq:proof:exc independence rule:delta:rhs}. The case when $\tri$
	belongs to $\aug{\stria, \sstria}$ is trivial, so we assume that $\tri$
	belongs to $\er(\stria, \strib)$ for some $\strib \in \sstrib$. This implies
	that there exists an atom $\atma$ and some interpretation $\twib$ such that
	$\twia \subseteq \twib \subseteq \twic$ and $\stria^{\twib}(\atma) \neq
	\strib^{\twib}(\atma)$. Suppose that $\stria^{\twib}(\atma) = \val_0$. If
	it also holds that $\aug{\stria, \sstria}^{\twib}(\atma) = \val_0$, then it
	can be seen that $\tri$ belongs to \eqref{eq:proof:exc independence
	rule:delta:lhs}. Otherwise it follows from the fact that $\stria$ is a
	subset of $\aug{\stria, \sstria}$ that $\aug{\stria,
	\sstria}^{\twib}(\atma)$ is undefined and it contains both
	$\twib[\val_0/\atma]$ and $\twib[\val_1/\atma] = \tpl{\twia', \twic'}$ for
	some $\val_1 \neq \val_0$. Thus, for some $\strib' \in \sstria$ it holds
	that $\tpl{\twia', \twic'}$ belongs to $\er(\stria, \strib')$. In other
	words, there exists an atom $\atmb$ and an interpretation $\twib'$ such that
	$\twia' \subseteq \twib' \subseteq \twic'$ and $\stria^{\twib'}(\atmb) \neq
	\strib'^{\twib'}(\atmb)$. Since $\stria$ is expressible by a non-disjunctive
	rule, it follows from Proposition~\ref{prop:re:rule in context} that $\atmb
	= \atma$. Also, $\twib$ and $\twib'$ may only differ in the valuation of
	$\atma$, so we obtain that $\stria^{\twib}(\atma) \neq
	\strib'^{\twib}(\atma)$. Consequently, $\tri$ belongs to $\er(\stria,
	\strib')$, so it also belongs to \eqref{eq:proof:exc independence
	rule:delta:lhs}.
\end{proof}


\begin{definition}
	Let $\dprg = \seq{\prg_\lia}_{\lia < \lng}$ be a DLP. We define
	$\after{\lia}{\dprg} = \bigcup_{\lia < \lib < \lng} \prg_\lib$.
\end{definition}

\begin{proposition}
	[Exception Independence for Programs]
	\label{prop:exception independence program:delta}
	Let $\dprg = \seq{\prg_i}_{\lia < \lng}$ be a DLP, $\er \in \set{\ertwo,
	\erthree}$ and $\uopr$ a \ter-based rule update operator. Then,
	$
		\modrer{\biguopr \dprg}
		= 
		\set{
			\aug{\stri, \modrer{\after{\lia}{\dprg}}}
			|
			\lia < \lng \land \stri \in \modrer{\prg_\lia}
		}
	$.
\end{proposition}
\begin{proof}
	By Proposition~\ref{prop:aug vs uope},
	$
		\modrer{\biguopr \dprg}
		=
		\set{
			\aug{\stri, \seq{\modrer{\prg_\lib}}_{\lia < \lib < \lng}}
			|
			\lia < \lng \land \stri \in \modrer{\prg_\lia}
		}
	$.
	The claim follows by induction on $\lng$ using
	Proposition~\ref{prop:exception independence rule:delta}.
%
%
\end{proof}

\begin{proposition}
	\label{prop:if delta stable then ju}
	Let $\dprg$ be a DLP and $\twib$ an interpretation. The following holds:
	\begin{textenum}[(i)]
		\item If $\uopr$ is a \tertwo-based rule update operator and $\twib$ is a
			stable model of $\biguopr \dprg$, then $\twib$ is a \JU-model of
			$\dprg$.

		\item If $\uopr$ is a \terthree-based rule update operator and $\twib$ is
			a stable model of $\biguopr \dprg$, then $\twib$ is a \AS-model of
			$\dprg$.
	\end{textenum}
\end{proposition}
\begin{proof}
	Let $\dprg = \seq{\prg_\lia}_{\lia < \lng}$. Also, put $\rej{\cdot, \cdot} =
	\rejju{\cdot, \cdot}$ if $\er = \ertwo$ and $\rej{\cdot, \cdot} =
	\rejas{\cdot, \cdot}$ if $\er = \erthree$. From the assumption we can
	conclude that $\tpl{\twib, \twib}$ belongs to $\modre{\biguopr \dprg}$ and
	for every $\twia \subsetneq \twib$, $\twiab$ does not belong to
	$\modre{\biguopr \dprg}$. We need to show that $\twib$ is a minimal model of
	the program $\prg' = [ \all{\dprg} \setminus \rej{\dprg, \twib} ]^\twib$.

	First we prove that $\twib$ is a model of $\prg'$. Take some rule $\rla' \in
	\prg'$ and let $\rla$ be a rule from $\all{\dprg} \setminus \rej{\dprg,
	\twib}$ such that $\rla' = \rla^\twib$. Then there is some $\lia < \lng$
	such that $\rla$ belongs to $\prg_\lia$. Let $\stria = \modre{\rla}$. Since
	$\rla$ belongs to $\prg_\lia$, we can use Proposition~\ref{prop:exception
	independence program:delta} to conclude that $\modrer{\biguopr \dprg}$
	contains the set
	\begin{equation}
		\label{eq:proof:if delta stable then ju:1}
		\aug{\stria, \modrer{\after{\lia}{\dprg}}}
		= \stria \cup \bigcup_{\strib \in \modrerscr{\after{\lia}{\dprg}}}
				\er(\stria, \strib) \enspace.
	\end{equation}
	Furthermore, since $\tpl{\twib, \twib}$ belongs to $\modre{\biguopr \dprg}$,
	it must also belong to \eqref{eq:proof:if delta stable then ju:1}. If
	$\tpl{\twib, \twib} \in \stria$, then $\twib \ent \rla$ and thus also $\twib
	\ent \rla'$ as desired. So suppose that $\tpl{\twib, \twib} \notin \stria$.
	Then for some $\lib$ with $\lia < \lib < \lng$ there exists some $\strib \in
	\modrer{\prg_\lib}$ such that $\tpl{\twib, \twib} \in \er(\stria, \strib)$.
	Thus, for some atom $\atm$, $\stria \confl{\twib}{\atm} \strib$, and by
	Proposition~\ref{prop:re:conflict in context vs syntax} we conclude that
	there is a rule $\rlb \in \prg_\lib$ with $\modre{\rlb} = \strib$, $\hrlb =
	\lcmp{\hrla}$ and $\twib \ent \brlb$. We consider two cases:
	\begin{textenum}[a)]
		\item If $\er = \ertwo$, then the assumption $\rla \notin \rejju{\dprg,
			\twib}$ is in direct conflict with the existence of $\rlb$.

		\item If $\er = \erthree$, then it follows from $\rla \notin \rejas{\dprg,
			\twib}$ that $\rlb \in \rejas{\dprg, \twib}$. Thus, there is some rule
			$\rlc \in (\prg_\lic \setminus \rejas{\dprg, \twib})$ with $\lib < \lic
			< \lng$ such that $\hrlc = \lcmp{\hrlb} = \hrla$ and $\twib \ent \brlc$.
			It follows from $\rlc \notin \rejas{\dprg, \twib}$ that no rule from
			$\after{\lic}{\dprg}$ has the head $\lcmp{\hrlc}$ and a body satisfied
			in $\twib$. Thus, the definition of \terthree{} and
			Proposition~\ref{prop:re:conflict in context vs syntax} imply that the
			set $\erthree(\modre{\rlc}, \strib)$ is empty for all $\strib \in
			\modrer{\after{\lic}{\dprg}}$. But since $\tpl{\twib, \twib}$ belongs to
			$\modre{\biguopr \dprg}$ by assumption, it must also belong to the set
			$\modre{\rlc} \cup \bigcup_{\strib \in \modrer{\after{\lic}{\dprg}}}
			\erthree(\modre{\rlc}, \strib)$ and we conclude that $\tpl{\twib, \twib}
			\in \modre{\rlc}$. Consequently, $\twib \ent \rlc$ and from $\twib \ent
			\brlc$ we conclude $\twib \ent \hrlc$. Since $\hrla = \hrlc$, we have
			shown that $\twib \ent \rla$ and thus also $\twib \ent \rla'$, as
			desired.
	\end{textenum}

	It remains to prove that $\twib$ is a minimal model of $\prg'$. Take some
	$\twia \subseteq \twib$ such that $\twia \ent \prg'$, we need to show that
	$\twia = \twib$. In the following we will show that $\twiab \in
	\modre{\biguopr \dprg}$ which, together with the assumption that $\twib$ is
	a stable model of $\biguopr \dprg$, implies that $\twia = \twib$. So take
	some set
	\begin{equation}
		\label{eq:proof:if delta stable then ju:3}
		\aug{\stria, \modrer{\after{\lia}{\dprg}}}
			= \stria \cup \bigcup_{\strib \in \modrerscr{\after{\lia}{\dprg}}}
				\er(\stria, \strib) \enspace.
	\end{equation}
	from $\modrer{\biguopr \dprg}$ with $\stria = \modre{\rla}$ and $\rla \in
	\prg_\lia$. We need to show that $\twiab$ belongs to \eqref{eq:proof:if
	delta stable then ju:3}. This obviously holds if $\twiab \in \stria$, so
	suppose that $\twiab \notin \stria$. Then, $\twia \nent \rla^\twib$.
	Thus, $\rla^\twib$ is different from $\ctau$ and, consequently, $\twib \ent
	\brla^-$. Also, $\brla^+ \subseteq \twia$ but $\hrla^+ \cap \twia =
	\emptyset$ and since $\twia \subseteq \twib$, this implies that
	\begin{equation}
		\label{eq:proof:if delta stable then ju:4}
		\twib \ent \brla
		\enspace.
	\end{equation}
	Moreover, since $\twia \ent \prg'$, it follows that $\rla^\twib \notin
	\prg'$, so $\rla \in \rej{\dprg, \twib}$. Thus, there exists a rule $\rlb
	\in \prg_\lib$ for some $\lib$ with $\lia < \lib < \lng$ such that for some
	atom $\atm$ and literal $\lit \in \set{\atm, \lpnot \atm}$,
	\begin{equation}
		\label{eq:proof:if delta stable then ju:5}
		\hrla = \set{\lit}
		\text{ and }
		\hrlb = \set{\lcmp{\lit}}
		\text{ and }
		\twib \ent \brlb
		\enspace.
	\end{equation}
	Let $\strib = \modre{\rlb}$. We consider the following five cases:
	\begin{textenum}[a)]
		\item If $\lit \in \brla$, then $\rla$ is tautological and we arrive at a
			conflict with the assumption $\twiab \notin \modre{\rla}$.

		\item If $\lcmp{\lit} \in \brla$, then it follows from \eqref{eq:proof:if
			delta stable then ju:4} and \eqref{eq:proof:if delta stable then ju:5}
			that $\tpl{\twib, \twib} \notin \modre{\rla}$. At the same time,
			$\er(\stria, \strib)$ is empty for all $\strib$ because for all
			interpretations $\twib'$ and atoms $\atmb$, it is impossible for
			$\stria^{\twib'}(\atmb)$ to be defined. Thus, we obtain a conflict with
			the assumption that $\tpl{\twib, \twib} \in \modre{\biguopr \dprg}$.

		\item If $\lit \in \brlb$, then it follows from \eqref{eq:proof:if delta
			stable then ju:5} that $\tpl{\twib, \twib} \notin \modre{\rlb}$. At the
			same time, $\er(\strib, \strib')$ is empty for all $\strib'$ because for
			all interpretations $\twib'$ and atoms $\atmb$, it is impossible for
			$\strib^{\twib'}(\atmb)$ to be defined. Thus, we obtain a conflict with
			the assumption that $\tpl{\twib, \twib} \in \modre{\biguopr \dprg}$.

		\item If $\lcmp{\lit} \in \brlb$, then it follows from \eqref{eq:proof:if
			delta stable then ju:5} that $\twib \ent \lcmp{\lit}$ and together with
			\eqref{eq:proof:if delta stable then ju:4} we obtain that $\twib \nent
			\rla$. Thus, $\tpl{\twib, \twib} \notin \modre{\rla}$ and since
			$\tpl{\twib, \twib} \in \modre{\biguopr \dprg}$, there must exist some
			$\strib' \in \modrer{\after{\lia}{\dprg}}$ such that $\tpl{\twib, \twib}
			\in \er(\stria, \strib')$. It follows from the definitions of \tertwo{}
			and \terthree{} that $\twiab \in \er(\stria, \strib')$, and thus also
			$\twiab \in \aug{\stria, \modrer{\after{\lia}{\dprg}}}$.
			
		\item In the remaining case, $\brla$ and $\brlb$ do not contain $\lit$ nor
			$\lcmp{\lit}$. Thus, we can use Proposition~\ref{prop:re:conflict in
			context vs syntax} to conclude that $\stria \confl{\twib}{\atm} \strib$.
			Note that if $\er = \erthree$, then the definition of $\rejas{\cdot,
			\cdot}$ implies that $\rlb \notin \rejas{\dprg, \twib}$, so $\twib \ent
			\rlb$. Using \eqref{eq:proof:if delta stable then ju:5} and
			\eqref{eq:proof:if delta stable then ju:4} we can conclude that $\twib
			\nent \rla$, so $\tpl{\twib, \twib} \notin \stria$. Consequently, by the
			definitions of \tertwo{} and \terthree{} it follows that $\twiab \in
			\er(\stria, \strib)$, and thus also $\twiab \in \aug{\stria,
			\modrer{\after{\lia}{\dprg}}}$. \qedhere
	\end{textenum}
\end{proof}

\begin{proposition}
	\label{prop:if ju then delta stable}
	Let $\dprg$ be a DLP free of local cycles and $\twib$ an interpretation. The
	following holds:
	\begin{textenum}[(i)]
		\item If $\uopr$ is a \tertwo-based rule update operator and $\twib$ is a
			\JU-model of $\dprg$, then $\twib$ is a stable model of $\biguopr
			\dprg$.

		\item If $\uopr$ is a \terthree-based rule update operator and $\twib$ is
			a \AS-model of $\dprg$, then $\twib$ is a stable model of $\biguopr
			\dprg$.
	\end{textenum}
\end{proposition}
\begin{proof}
	Let $\dprg = \seq{\prg_\lia}_{\lia < \lng}$. Also, put $\rej{\cdot, \cdot} =
	\rejju{\cdot, \cdot}$ if $\er = \ertwo$ and $\rej{\cdot, \cdot} =
	\rejas{\cdot, \cdot}$ if $\er = \erthree$. From the assumption we can
	conclude that $\twib$ is a minimal model of the program $\prg' =
	[\all{\dprg} \setminus \rej{\dprg, \twib}]^\twib$. We need to prove that
	$\tpl{\twib, \twib} \in \modre{\biguopr \dprg}$ and for every $\twia
	\subsetneq \twib$, $\twiab \notin \modre{\biguopr \dprg}$.

	In order to show that $\tpl{\twib, \twib} \in \modre{\biguopr \dprg}$, take
	some set
	\begin{equation}
		\label{eq:proof:if ju then delta stable:1}
		\aug{\stria, \modrer{\after{\lia}{\dprg}}}
			= \stria \cup \bigcup_{\strib \in \modrerscr{\after{\lia}{\dprg}}}
				\er(\stria, \strib)
	\end{equation}
	from $\modrer{\biguopr \dprg}$ where $\stria = \modre{\rla}$ for some $\rla
	\in \prg_\lia$. If $\tpl{\twib, \twib}$ belongs to $\stria$, then it
	obviously belongs to \eqref{eq:proof:if ju then delta stable:1}. So assume
	that $\tpl{\twib, \twib} \notin \stria$. Then $\twib \nent \rla$, so
	\begin{equation}
		\label{eq:proof:if ju then delta stable:2}
		\twib \ent \brla \enspace,
	\end{equation}
	and we can also conclude that $\rla$ belongs to $\rej{\dprg, \twib}$. As a
	consequence, there exists a rule $\rlb$ from $\prg_\lib$ for some $\lib$
	with $\lia < \lib < \lng$ such that for some atom $\atm$ and literal $\lit
	\in \set{\atm, \lpnot \atm}$,
	\begin{equation}
		\label{eq:proof:if ju then delta stable:3}
		\hrla = \set{\lit} \text{ and } \hrlb = \set{\lcmp{\lit}} \text{ and }
		\twib \ent \brlb
		\enspace.
	\end{equation}
	Let $\strib = \modre{\rlb}$. It can be verified that \eqref{eq:proof:if ju
	then delta stable:2} and \eqref{eq:proof:if ju then delta stable:3},
	together with the assumption that $\rla$ and $\rlb$ are not local cycles,
	allow us to use Proposition~\ref{prop:re:conflict in context vs syntax} and
	conclude that for some atom $\atm$, $\stria \confl{\twib}{\atm} \strib$.
	Thus, $\tpl{\twib, \twib}$ belongs to $\er(\stria, \strib)$, and
	consequently also to \eqref{eq:proof:if ju then delta stable:1}.

	Now suppose that $\twiab \in \modre{\biguopr \dprg}$. We will show that
	$\twia \ent \prg'$, which implies that $\twia = \twib$ because $\twib$ is by
	assumption a minimal model of $\prg'$. Take some $\rla' \in \prg'$. Then
	there is a rule $\rla$ from $\all{\dprg} \setminus \rej{\dprg, \twib}$ such
	that $\rla' = \rla^\twib$. Suppose that $\rla \in \prg_\lia$. We will prove
	by contradiction that $\twia \ent \rla'$. So suppose that $\twia \nent
	\rla'$. Then $\rla^\twib$ is different from $\ctau$ and, consequently,
	$\twib \ent \brla^-$. Also, $\brla^+$ is included in $\twia$, so since
	$\twia$ is a subset of $\twib$, $\brla^+$ is included in $\twib$ as well.
	Hence,
	\begin{equation}
		\label{eq:proof:if ju then delta stable:4}
		\twib \ent \brla \enspace.
	\end{equation}
	Also, $\twiab \notin \modre{\rla}$. By our assumption, $\twiab \in \stria
	\cup \bigcup_{\strib \in \modrerscr{\after{\lia}{\dprg}}} \er(\stria,
	\strib)$ where $\stria = \modre{\rla}$. We have already shown that $\twiab
	\notin \stria$, so there must be a rule $\rlb \in \prg_\lib$ for some $\lib$
	with $\lia < \lib < \lng$ such that $\twiab \in \er(\stria, \strib)$ where
	$\strib = \modre{\rlb}$. Thus, there exists some interpretation $\twic$ and
	an atom $\atm$ such that $\twia \subseteq \twic \subseteq \twib$, $\stria
	\confl{\twic}{\atm} \strib$ and if $\atm \in \twib \setminus \twia$, then
	$\twib = \twic$. By Proposition~\ref{prop:re:conflict in context vs syntax},
	there is a literal $\lit \in \set{\atm, \lpnot \atm}$ such that $\hrla =
	\set{\lit}$ and $\hrlb = \set{\lcmp{\lit}}$. We consider three cases:
	\begin{textenum}[a)]
		\item If $\atm \notin \twib \setminus \twia$, then $\twi(\atm) =
			\twib(\atm)$. Furthermore, $\twib \ent \rla$ because $\twib \ent \prg'$
			and from \eqref{eq:proof:if ju then delta stable:4} we obtain that
			$\twib \ent \hrla$. Thus, $\twi \ent \rla^\twib$, a conflict with the
			assumption that $\twi \nent \rl'$. 

		\item If $\atm \in \twib \setminus \twia$ and $\er = \ertwo$, then
			it follows from the definition of $\ertwo$ that $\twib = \twic$, so
			$\stria \confl{\twib}{\atm} \strib$. Thus, by
			Proposition~\ref{prop:re:conflict in context vs syntax} we conclude that
			$\rla \in \rej{\dprg, \twib}$, contrary to the way $\rla$ was picked.

		\item If $\atm \in \twib \setminus \twia$ and $\er = \erthree$, then the
			definition of $\erthree$ implies that $\tpl{\twib, \twib} \notin
			\modre{\rla}$, contrary to the assumption that $\twib \ent \prg'$.
			\qedhere
	\end{textenum}
\end{proof}

\subsection{Syntactic Properties of \texorpdfstring{\terfour-}{delta\_d-} and
\texorpdfstring{\terfive-}{delta\_e-}Based Operators}

\label{app:rules:syntactic properties:erfour}

\begin{lemma}
	\label{lemma:ertwo vs erfour mod sm:1}
	Let $\dprg$ be a DLP, $\uoprtwo$, $\uoprthree$, $\uoprfour$ and $\uoprfive$ be
	\tertwo-, \terthree-, \terfour- and \terfive-based rule update operators,
	respectively. Then the following holds:
	\begin{textenum}[(i)]
		\item If $\stria \in \modrer{\biguoprfour \dprg}$, then either
			$\stria = \tris$ or $\stria \in \modrer{\biguoprtwo \dprg}$;

		\item If $\stria \in \modrer{\biguoprfive \dprg}$, then either
			$\stria = \tris$ or $\stria \in \modrer{\biguoprthree \dprg}$.
	\end{textenum}
\end{lemma}
\begin{proof}
	Follows by induction on the length of $\dprg$.
\end{proof}

\begin{lemma}
	\label{lemma:ertwo vs erfour mod sm:2}
	Let $\dprg$ be a DLP, $\uoprtwo$, $\uoprthree$, $\uoprfour$ and $\uoprfive$ be
	\tertwo-, \terthree-, \terfour- and \terfive-based rule update operators,
	respectively. Then the following holds:
	\begin{textenum}[(i)]
		\item If $\stria \in \modrer{\biguoprtwo \dprg}$, then for some $\strib
			\in \modrer{\biguoprfour \dprg}$, $\strib \subseteq \stria$;

		\item If $\stria \in \modrer{\biguoprthree \dprg}$, then for some $\strib
			\in \modrer{\biguoprfive \dprg}$, $\strib \subseteq \stria$.
	\end{textenum}
\end{lemma}
\begin{proof}
	The following two stronger statements follow by induction on the length of
	$\dprg$ using Proposition~\ref{prop:exception independence rule:delta}:
	\begin{textenum}[(i)]
		\item If $\stria \in \modrer{\biguoprtwo \dprg}$, then for some set of
			\RE-rule-expressible sets of three-valued interpretations $\sstri$ and
			some $\strib \in \modrer{\biguoprfour \dprg}$, $\stria = \augtwo{\strib,
			\sstri}$.

		\item If $\stria \in \modrer{\biguoprthree \dprg}$, then for some set of
			\RE-rule-expressible sets of three-valued interpretations $\sstri$ and
			some $\strib \in \modrer{\biguoprfive \dprg}$, $\stria = \augtwo{\strib,
			\sstri}$. \qedhere
	\end{textenum}
\end{proof}

\begin{proposition}
	\label{prop:ertwo vs erfour mod sm}
	Let $\dprg$ be a DLP, $\uoprtwo$, $\uoprthree$, $\uoprfour$ and $\uoprfive$
	be \tertwo-, \terthree-, \terfour- and \terfive-based rule update operators,
	respectively. Then $\modsm{\biguoprtwo \dprg} = \modsm{\biguoprfour \dprg}$
	and $\modsm{\biguoprthree \dprg} = \modsm{\biguoprfive \dprg}$.
\end{proposition}
\begin{proof}
	By Lemma~\ref{lemma:ertwo vs erfour mod sm:1}, $\modrer{\biguoprtwo \dprg}
	\cup \set{\tris} \supseteq \modrer{\biguoprfour \dprg}$ and
	$\modrer{\biguoprthree \dprg} \cup \set{\tris} \supseteq
	\modrer{\biguoprfive \dprg}$, so
	\begin{align*}
		\modre{\biguoprtwo \dprg}
		&=
		\bigcap \modrer{\biguoprtwo \dprg}
		=
		\bigcap (\modrer{\biguoprtwo \dprg} \cup \set{\tris})
		\subseteq
		\bigcap \modrer{\biguoprfour \dprg}
		=
		\modre{\biguoprfour \dprg}
		\enspace,
		\\
		\modre{\biguoprthree \dprg}
		&=
		\bigcap \modrer{\biguoprthree \dprg}
		=
		\bigcap (\modrer{\biguoprthree \dprg} \cup \set{\tris})
		\subseteq
		\bigcap \modrer{\biguoprfive \dprg}
		=
		\modre{\biguoprfive \dprg}
		\enspace.
	\end{align*}
	Furthermore, by Lemma~\ref{lemma:ertwo vs erfour mod sm:2},
	\begin{align*}
		\modre{\biguoprtwo \dprg}
		&=
		\bigcap \modrer{\biguoprtwo \dprg}
		=
		\bigcap_{\stria \in \modrer{\uoprtwo \dprg}} \stria
		\supseteq
		\bigcap_{\strib \in \modrer{\uoprfour \dprg}} \strib
		=
		\bigcap \modrer{\biguoprfour \dprg}
		=
		\modre{\biguoprfour \dprg}
		\enspace,
		\\
		\modre{\biguoprthree \dprg}
		&=
		\bigcap \modrer{\biguoprthree \dprg}
		=
		\bigcap_{\stria \in \modrer{\uoprthree \dprg}} \stria
		\supseteq
		\bigcap_{\strib \in \modrer{\uoprfive \dprg}} \strib
		=
		\bigcap \modrer{\biguoprfive \dprg}
		=
		\modre{\biguoprfive \dprg}
		\enspace.
	\end{align*}
	Thus, $\modre{\biguoprtwo \dprg} = \modre{\biguoprfour \dprg}$ and
	$\modre{\biguoprthree \dprg} = \modre{\biguoprfive \dprg}$, and the rest
	follows from Proposition~\ref{prop:eq and ent comparison}.
\end{proof}

\begin{proof}
	[\textbf{Proof of Theorem~\ref{thm:er vs ju and as}}]
	\label{proof:thm:er vs ju and as}
	Follows from Propositions~\ref{prop:if delta stable then ju}, \ref{prop:if
	ju then delta stable} and \ref{prop:ertwo vs erfour mod sm}.
\end{proof}

\begin{proof}
	[\textbf{Proof of Theorem~\ref{thm:syntactic properties:ertwo and erfour}}]
	\label{proof:thm:re:syntactic properties:ertwo and erfour}
	Follows from Theorem~\ref{thm:er vs ju and as} and from the fact that the
	\JU- and \AS-models satisfy all of these properties.
\end{proof}

\begin{proof}
	[\textbf{Proof of Corollary~\ref{cor:state condensing}}]
	\label{proof:cor:state condensing}
	It suffices to put $\rb = \biguopr \seq{\prg_\lia}_{\lia \leq \lib}$ and
	apply Theorem~\ref{thm:er vs ju and as}.
\end{proof}

\subsection{Semantic Properties}

\label{app:rules:semantic properties}

\begin{proposition}
	Let $\uopr$ be a \ter-based rule update operator. Then $\uopr$ satisfies
	\pup{Initialisation}, \pup{Disjointness}, \pu{1} and \putwotop{} with
	respect to \RR{}, \SR{}, \RMR{}, \SMR{}, \RE{}, \SE{} and \SM{} (where
	applicable).
\end{proposition}
\begin{proof}
	We prove these properties with respect to \RR-equivalence; their
	satisfaction with respect to the other notions of program equivalence
	follows from Proposition~\ref{prop:eq and ent comparison}. To verify that
	\pup{Initialisation} holds, it suffices to observe that
	\[
		\modrer{\emptyset \uopr \rbu}
		=
		\Set{\aug{\stri, \modrer{\rbu}} | \stri \in \emptyset}
		\cup
		\modrer{\rbu}
		=
		\modrer{\rbu}
		\enspace.
	\]
	Thus, $\emptyset \uopr \rbu$ is \RR-equivalent to $\rbu$. As for
	\pup{Disjointness}, it suffices to observe that $\modrer{(\rba \cup \rbb)
	\uopr \rbu}$ coincides with
	\begin{align*}
		&\Set{\Aug{\stri, \modrer{\rbu}} | \stri \in \rba \cup \rbb}
		\cup
		\modrer{\rbu} \\
		&\hspace{2em}=
		\Br{
			\Set{\Aug{\stri, \modrer{\rbu}} | \stri \in \rba}
			\cup
			\modrer{\rbu}
		}
		\cup
		\Br{
			\Set{\Aug{\stri, \modrer{\rbu}} | \stri \in \rbb}
			\cup
			\modrer{\rbu}
		} \\
		&\hspace{2em}=
		\modrer{\rba \uopr \rbu}
		\cup
		\modrer{\rbb \uopr \rbu}
		=
		\modrer{(\rba \uopr \rbu) \cup (\rbb \uopr \rbu)}
		\enspace.
	\end{align*}
	In order to prove that \pu{1} holds, consider that $\modrer{\rbu}$ is a
	subset of $\modrer{\rba \uopr \rbu}$. Consequently, $\rba \uopr \rbu \entRR
	\rbu$. Finally, \putwotop{} follows from the fact that $\aug{\stri,
	\emptyset} = \stri$ for all $\stri \subseteq \tris$.
\end{proof}

\begin{lemma}
	\label{lemma:semantic properties:non-interference:1}
	Let $\prga$, $\prgb$ be programs over disjoint alphabets and $\er \in
	\set{\erone, \ertwo, \erthree, \erfour, \erfive}$. Then for all $\stria \in
	\modrer{\prga}$, either $\aug{\stria, \modrer{\prgb}} = \stria$ or
	$\aug{\stria, \modrer{\prgb}} = \tris$.
\end{lemma}
\begin{proof}
	Follows from Proposition~\ref{prop:re:rule in context} and the definitions
	of \terone{}, \tertwo{}, \terthree{}, \terfour{}, and \terfive{}.
\end{proof}

\begin{lemma}
	\label{lemma:semantic properties:non-interference:2}
	Let $\stria \subseteq \tris$ be \RE-rule-expressible, $\prgu$ a program,
	$\atma$ an atom and $\val_0$ a truth value. If $\augone{\stria,
	\modrer{\prgu}}^\twib(\atma) = \val_0$, then either $\stria^\twib(\atma) =
	\val_0$, or $\atma$ occurs in $\prgu$.
\end{lemma}
\begin{proof}
	Suppose that $\augone{\stria, \modrer{\prgu}}^\twib(\atma) = \val_0$ and
	$\stria^\twib(\atma) \neq \val_0$. Thus, $\augone{\stria, \modrer{\prgu}}$
	contains $\twib[\val/\atma]$ if and only if $\val = \val_0$. Since $\stria$
	is a subset of $\augone{\stria, \modrer{\prgu}}$, it follows that $\stria$
	does not contain $\twib[\tr/\atma]$, $\twib[\un/\atma]$ nor
	$\twib[\fa/\atma]$. Let $\rl$ be some rule such that $\modre{\rl} = \stria$.
	We can conclude that $\atma$ does not occur in $\rl$. Furthermore, if
	$\val_0 = \tr$, then by the definition of \terone{} we obtain a conflict
	with the fact that $\twib[\un/\atma]$ does not belong to $\augone{\stria,
	\modrer{\prgu}}$. Consequently, $\val_0 = \fa$. Furthermore, since
	$\twib[\tr/\atma]$ does not belong to $\augone{\stria, \modrer{\prgu}}$
	but $\twib[\fa/\atma]$ does, there exists some $\strib \in \modrer{\prgu}$
	and some atom $\atmb$ such that $\stria \confl{\twib[\fa/\atma]}{\atmb}
	\strib$ but it is not the case that $\stria
	\confl{\twib[\tr/\atma]}{\atmb} \strib$. Since $\atma$ does not occur in
	$\rl$, this is only possible if $\atma$ occurs in $\prgu$.
\end{proof}

\begin{proposition}
	Let $\uopr$ be a \terone-, \tertwo-, \terthree, \terfour or \terfive-based
	rule update operator. Then $\uopr$ satisfies \pup{Non-interference} for
	non-disjunctive programs with respect to \RR{}, \SR{}, \RMR{}, \SMR{},
	\RE{}, \SE{} and \SM{}.
\end{proposition}
\begin{proof}
	We prove this property with respect to \RR-equivalence; its satisfaction
	with respect to the other notions of program equivalence follows from
	Proposition~\ref{prop:eq and ent comparison}. Suppose that $\prga$, $\prgu$
	and $\prgv$ are non-disjunctive programs, $\er \in \set{\erone, \ertwo,
	\erthree, \erfour, \erfive}$ and $\uopr$ is a \ter-based rule update
	operator. Take some $\stria_0 \in \modrer{(\prga \uopr \prgu) \uopr \prgv}$.
	We will show that $\stria_0$ either belongs to $\modrer{(\prga \uopr \prgv)
	\uopr \prgu}$ or $\stria_0 = \tris$. Note that $\modrer{(\prga \uopr \prgu)
	\uopr \prgv}$ coincides with
	$
		\Set{
			\Aug{\stria, \modrer{\prgv}}
			|
			\stria \in \modrer{\prga \uopr \prgu}
		}
		\cup
		\modrer{\prgv}
	$.
	We consider three cases:
	\begin{textenum}[a)]
		\item Suppose that $\stria_0 \in \modrer{\prgv}$. Then $\stria_0 \in
			\modrer{\prga \uopr \prgv}$ and since $\modrer{(\prga \uopr \prgv) \uopr
			\prgu}$ coincides with
			$
				\Set{
					\aug{\stria, \modrer{\prgu}}
					|
					\stria \in \modrer{\prga \uopr \prgv}
				}
				\cup
				\modrer{\prgu}
			$,
			it must contain the set $\aug{\stria_0, \modrer{\prgu}}$. Furthermore,
			since $\prgu$ and $\prgv$ are over disjoint alphabets, it follows from
			Lemma~\ref{lemma:semantic properties:non-interference:1} that
			$\aug{\stria_0, \modrer{\prgu}}$ is either $\stria_0$ or $\tris$, as
			desired.

		\item Suppose that $\stria_0 = \aug{\stria, \modrer{\prgv}}$ and $\stria
			\in \modrer{\prgu}$. As in the previous case, since $\prgu$ and $\prgv$
			are over disjoint alphabets, it follows from Lemma~\ref{lemma:semantic
			properties:non-interference:1} that $\aug{\stria, \modrer{\prgv}}$ is
			either $\stria$ or $\tris$. If the former is true, then $\stria_0 \in
			\modrer{\prgu}$, so $\stria_0 \in \modrer{(\prga \uopr \prgv) \uopr
			\prgu}$.

		\item Suppose that $\stria_0 = \Aug{\Aug{\stria, \modrer{\prgu}},
			\modrer{\prgv}}$ for some $\stria \in \modrer{\prga}$. If $\er \in
			\set{\ertwo, \erthree}$, then it follows directly from
			Proposition~\ref{prop:exception independence rule:delta} that
			\[
				\stria_0
				=
				\Aug{\stria, \modrer{\prgu} \cup \modrer{\prgv}} \\
				=
				\Aug{\Aug{\stria, \modrer{\prgv}}, \modrer{\prgu}}
				\enspace.
			\]
			Consequently, $\stria_0 \in \modrer{(\prga \uopr \prgv) \uopr \prgu}$.

			If $\er \in \set{\erfour, \erfive}$, then either $\stria_0 = \tris$ or
			$\stria_0 = \Augtwo{\Augtwo{\stria, \modrer{\prgu}}, \modrer{\prgv}}$.
			The rest follows by the previous paragraph.

			Finally, if $\er = \erone$, then we consider two subcases:
			\begin{textenum}[(i)]
				\item If $\augone{\stria, \modrer{\prgu}} = \stria$, then $\stria_0 =
					\augone{\stria, \modrer{\prgv}}$, so $\stria_0$ belongs to
					$\modrer{\prga \uopr \prgv}$. Take some $\strib \in \modrer{\prgu}$
					and suppose that $\stria_0 \confl{\atm}{\twib} \strib$. Then, by
					Lemma~\ref{lemma:semantic properties:non-interference:2}, either
					$\stria^\twib(\atm) = \val_0$ for some truth value $\val_0$, or
					$\atm$ occurs in $\prgv$. In the former case we obtain a conflict
					with the assumption that $\augone{\stria, \modrer{\prgu}} = \stria$
					while the latter case is in conflict with the assumption that
					$\prgu$ and $\prgv$ are over disjoint alphabets. Thus, no such
					$\strib$ exists and $\augone{\stria_0, \modrer{\prgu}} = \stria_0$.
					Consequently, $\stria_0$ belongs to $\modrer{(\prga \uopr \prgv)
					\uopr \prgu}$.

				\item If $\augone{\stria, \modrer{\prgu}} \neq \stria$, then there is
					some $\strib \in \modrer{\prgu}$ such that $\stria
					\confl{\atma}{\twib} \strib$ for some atom $\atma$ and
					interpretation $\twib$. Thus since $\prgu$ and $\prgv$ are over
					disjoint alphabets, it follows from
					Proposition~\ref{prop:re:conflict in context vs syntax} that
					$\augone{\stria, \modrer{\prgv}} = \stria$, so $\augone{\stria,
					\modrer{\prgu}}$ belongs to $\modrer{(\prga \uopr \prgv) \uopr
					\prgu}$.

					It remains to show that $\stria_0 = \augone{\stria, \modrer{\prgu}}$.
					Put $\stria_1 = \augone{\stria, \modrer{\prgu}}$ and suppose that
					for some $\strib' \in \modrer{\prgv}$, some atom $\atmb$ and some
					interpretation $\twic$, $\stria_1 \confl{\twic}{\atmb} \strib'$.
					Then, by Lemma~\ref{lemma:semantic properties:non-interference:2},
					either $\stria^\twic(\atmb) = \val_0$ for some truth value $\val_0$,
					or $\atmb$ occurs in $\prgu$. In the former case, $\atma = \atmb$ by
					Proposition~\ref{prop:re:conflict in context vs syntax} and we
					obtain a conflict with the fact that $\prgu$ and $\prgv$ are over
					disjoint alphabets. In the latter case, $\atmb$ occurs in both
					$\prgu$ and $\prgv$, so the same conflict follows. Consequently, no
					such $\strib'$ exists and $\stria_0 = \Augone{\Augone{\stria,
					\modrer{\prgu}}, \modrer{\prgv}} = \Augone{\stria_1, \modrer{\prgv}}
					= \Augone{\stria, \modrer{\prgu}}$.
			\end{textenum}
	\end{textenum}
	The proof of the other inclusion is symmetric.
\end{proof}

\begin{remark}
	[\pup{Non-interference} for Disjunctive Programs]
	\label{ex:counterexample for non-interference}
	The following programs form a counterexample to \pup{Non-interference} for
	\tertwo-, \terthree-, \terfour- and \terfive-based operators under
	\SM-equivalence (and thus under all stronger notions of equivalence as
	well):
	\begin{align*}
		\prga: \quad
			\atma; \atmb; \atmc &.
		&\prgu: \quad
			\lpnot \atma &\lpif \lpnot \atmc.
		&\prgv: \quad
			\lpnot \atmb &. \\
		&& \atma &\lpif \atmc. \\
		&& \atmc &\lpif \atma.
	\end{align*}
	As for \terone-based operators, the singleton programs $\prga' = \set{\atma;
	\atmb.}$, $\prgu' = \set{\lpnot \atma \lpif \lpnot \atmc.}$ and $\prgv' =
	\set{\lpnot \atmb \lpif \atmd.}$ form a counterexample to
	\pup{Non-interference} under \SE-equivalence and all stronger notions of
	equivalence.
\end{remark}

\begin{proposition}
	Let $\uopr$ be a \ter-based rule update operator where $\er(\stri, \tris)
	\subseteq \stri$ for all $\stri \subseteq \tris$. Then $\uopr$ satisfies
	\pup{Tautology} and \pup{Immunity to Tautologies} with respect to \RR{},
	\SR{}, \RMR{}, \SMR{}, \RE{}, \SE{} and \SM{}.
\end{proposition}
\begin{proof}
	For \RR-equivalence this can be verified in a straight-forward manner. For
	the remaining notions of program equivalence this follows from
	Proposition~\ref{prop:eq and ent comparison}.
\end{proof}

\begin{proposition}
	Let $\uopr$ be a \ter-based rule update operator. Then $\uopr$ satisfies
	\pup{Idempotence} with respect to \RMR{}, \SMR{}, \RE{}, \SE{} and \SM{}.
	Moreover, if $\uopr$ is \terfour- or \terfive-based, then it also satisfies
	\pup{Idempotence} with respect to \RR{} and \SR{}.
\end{proposition}
\begin{proof}
	\pup{Idempotence} states the following: $\rb \uopr \rb \equiv \rb$. We will
	show that this is true under \RMR-equivalence which, together with
	Proposition~\ref{prop:eq and ent comparison}, implies that it holds under
	\SMR-, \RE-, \SE- and \SM-equivalence.

	First take some $\stria \in \min \modrer{(\rb \uopr \rb) \cup \set{\ctau}}$.
	If $\stria = \tris$, then $\rb \uopr \rb$ is tautological and since
	$\modrer{\rb}$ is a subset of $\modrer{\rb \uopr \rb}$, $\rb$ is itself
	tautological. Thus, $\tris$ also belongs to $\min \modrer{\rb \cup
	\set{\ctau}}$. In the principal case, either $\stria \in \modrer{\rb}$, or
	$\stria = \aug{\stria_0, \modrer{\rb}}$ for some $\stria_0 \in
	\modrer{\rb}$. In the latter case we have that $\stria_0$ is a subset of
	$\stria$ and since $\stria_0$ belongs to $\modrer{\rb \uopr \rb}$, by the
	minimality of $\stria$ we obtain that $\stria = \stria_0$, so $\stria$
	belongs to $\modrer{\rb}$. Now it follows that $\stria$ is minimal
	$\modrer{\rb}$ because $\modrer{\rb}$ is a subset of $\modrer{\rb \uopr
	\rb}$ and $\stria$ is minimal in the latter set.

	Now take some $\stria \in \min \modrer{\rb \cup \set{\ctau}}$. If $\stria =
	\tris$, then $\rb$ is tautological and it follows by the properties of
	$\uopr$ that $\rb \uopr \rb$ is also tautological. Thus, $\tris$ also
	belongs to $\modrer{(\rb \uopr \rb) \cup \set{\ctau}}$. In the principal
	case, $\stria \in \modrer{\rb}$. Take some $\strib \in \modrer{\rb \uopr
	\rb}$ such that $\strib$ is a subset of $\stria$. If $\strib$ belongs to
	$\modrer{\rb}$, then it follows by minimality of $\stria$ that $\stria =
	\strib$. On the other hand, if $\strib$ is of the form $\aug{\strib_0,
	\modrer{\rb}}$ for some $\strib_0$ from $\modrer{\rb}$, then $\strib_0
	\subseteq \strib \subseteq \stria$, so by the minimality of $\stria$,
	$\strib_0 = \strib = \stria$. Thus, in either case, $\stria = \strib$,
	which proves that it is minimal within $\modrer{\rb \uopr \rb}$.

	Now consider some \terfour- or \terfive-based rule update operator $\uopr$.
	Then, $\modrer{\rb \uopr \rb} \subseteq \modrer{\rb} \cup \set{\tris}$, so
	obviously $\modrer{(\rb \uopr \rb) \cup \set{\ctau}} = \modrer{\rb \cup
	\set{\ctau}}$.
\end{proof}


\begin{remark}
	[\pup{Idempotence} under $\eqSR$ and $\eqRR$]
	\label{ex:counterexample to idempotence}
	The rule base $\rb = \set{(\atma \lpif \lpnot \atmb.), (\lpnot \atma \lpif
	\atmc.)}$ forms a counterexample to \pup{Idempotence} for \terone-, \tertwo-
	and \terthree-based operators with respect to \SR{} and \RR{}.

\end{remark}

\begin{lemma}
	\label{lemma:semantic properties:absorption}
	Let $\stria \subseteq \tris$ be \RE-rule-expressible, $\sstri$ a set of
	\RE-rule-expressible sets of three-valued interpretations. The following
	holds:
	\begin{textenum}[(i)]
		\item If $\er \in \set{\ertwo, \erthree}$, then $\aug{\aug{\stria,
			\sstri}, \sstri} = \aug{\stria, \sstri}$.

		\item If $\er \in \set{\erfour, \erfive}$, then either $\aug{\aug{\stria,
			\sstri}, \sstri} = \aug{\stria, \sstri}$, or both $\aug{\aug{\stria,
			\sstri}, \sstri} = \tris$ and $\aug{\stria, \sstri} \in \sstri$.
	\end{textenum}
\end{lemma}
\begin{proof}
	Let $\er \in \set{\ertwo, \erthree}$, put $\stria' = \aug{\stria, \sstri}$
	and take some $\strib \in \sstri$ such that $\stria' \confl{\atm}{\twib}
	\strib$ for some atom $\atm$ and some interpretation $\twib$. Then
	$\stria'^\twib(\atm) = \val_0$ for some truth value $\val_0$, so it follows
	from Lemma~\ref{lemma:aug delta value} that $\stria^\twib(\atm) = \val_0$.
	But this implies $\stria \confl{\atm}{\twib} \strib$, so both $\tpl{\twib
	\setminus \set{\atm}, \twib \setminus \set{\atm}}$ and $\tpl{\twib \cup
	\set{\atm}, \twib \cup \set{\atm}}$ belong to $\stria'$, a conflict with
	the assumption that $\stria'^\twib(\atm)$ is defined. As a consequence,
	no such $\strib \in \sstri$ exists, so $\aug{\stria', \sstri} =
	\stria' = \aug{\stria, \sstri}$.

	On the other hand, if $\er \in \set{\erfour, \erfive}$, then we can observe
	that either the previous case applies, or both $\aug{\aug{\stria, \sstri},
	\sstri} = \tris$ and $\aug{\stria, \sstri} \in \sstri$.
\end{proof}

\begin{proposition}
	\label{prop:semantic properties:absorption}
	Let $\uopr$ be a \tertwo-, \terthree-, \terfour- or \terfive-based rule
	update operator. Then $\uopr$ satisfies \pup{Absorption} with respect to
	\RMR{}, \SMR{}, \RE{}, \SE{} and \SM{}. Moreover, if $\uopr$ is \terfour-
	or \terfive-based, then it satisfies \pup{Absorption} with respect to \RR{}
	and \SR{}.
\end{proposition}
\begin{proof}
	\pup{Absorption} states the following: $(\rba \uopr \rbu) \uopr \rbu \equiv
	\rba \uopr \rbu$. We will show that this is true under \RMR-equivalence
	which, together with Proposition~\ref{prop:eq and ent comparison}, implies
	that it holds under \SMR-, \RE-, \SE- and \SM-equivalence.

	So suppose that $\uopr$ is \tertwo-, \terthree-, \terfour- or
	\terfive-based. By Lemma~\ref{lemma:semantic
	properties:absorption}, $\modrer{((\rba \uopr \rbu) \uopr \rbu) \cup
	\set{\ctau}}$ is a superset of $\modrer{(\rba \uopr \rbu) \cup
	\set{\ctau}}$. Thus, whenever $\stria$ is minimal in $\modrer{((\rba
	\uopr \rbu) \uopr \rbu) \cup \set{\ctau}}$, it is also minimal in
	$\modrer{(\rba \uopr \rbu) \cup \set{\ctau}}$. Furthermore, the extra
	elements of $\modrer{((\rba \uopr \rbu) \uopr \rbu) \cup \set{\ctau}}$
	are never smaller than the elements of $\modrer{(\rba \uopr \rbu) \cup
	\set{\ctau}}$ because they are of the form $\aug{\strib,
	\modrer{\rbu}}$ for some $\strib \in \modrer{\rbu} \subseteq
	\modrer{(\rba \uopr \rbu) \cup \set{\ctau}}$. Thus, whenever
	$\stria$ is minimal in $\modrer{(\rba \uopr \rbu) \cup
	\set{\ctau}}$, it must also be minimal in $\modrer{((\rba \uopr
	\rbu) \uopr \rbu) \cup \set{\ctau}}$.

	Furthermore, if $\uopr$ is \terfour- or \terfive-based, then it additionally
	holds that $\modrer{((\rba \uopr \rbu) \uopr \rbu) \cup \set{\ctau}} =
	\modrer{(\rba \uopr \rbu) \cup \set{\ctau}}$.
\end{proof}

\begin{proposition}
	Let $\uopr$ be a \tertwo-, \terthree-, \terfour- or \terfive-based rule
	update operator. Then $\uopr$ satisfies \pup{Augmentation} for
	non-disjunctive programs with respect to \RMR{}, \SMR{}, \RE{}, \SE{} and
	\SM{}. Moreover, if $\uopr$ is \terfour- or \terfive-based, then it
	satisfies \pup{Augmentation} for non-disjunctive programs with respect to
	\RR{} and \SR{}.
\end{proposition}
\begin{proof}
	The proof for \RMR-equivalence follows from Proposition~\ref{prop:exception
	independence rule:delta} and Lemma~\ref{lemma:ertwo vs erfour mod sm:1} and
	from the fact that the extra elements of $\modrer{((\rba \uopr \rbu) \uopr
	\rbv) \cup \set{\ctau}}$, as compared to $\modrer{(\rba \uopr \rbv) \cup
	\set{\ctau}}$, are non-minimal in the latter set. If $\uopr$ is
	\terfour- or \terfive-based, then there are no extra elements and the
	rest follows from Proposition~\ref{prop:exception independence
	rule:delta} and Lemma~\ref{lemma:ertwo vs erfour mod sm:1}.
\end{proof}

\begin{remark}
	[\pup{Absorption} and \pup{Augmentation} violated by \terone{}]
	\label{ex:counterexample to absorption}
	The rule bases $\rba = \set{\atma.}$ and $\rbu = \rbv = \set{\lpnot \atma
	\lpif \lpnot \atmb., \atmb.}$ form counterexamples to \pup{Absorption} and
	\pup{Augmentation} for \terone-based operators with respect \SM-equivalence
	and any stronger notion of equivalence.
\end{remark}

\begin{remark}
	[\pup{Augmentation} for Disjunctive Programs]
	The following programs form a counterexample to \pup{Augmentation} for
	\tertwo-, \terthree-, \terfour- and \terfive-based operators under
	\SM-equivalence (and thus under all stronger notions of equivalence as
	well):
	\begin{align*}
		\prga: \quad
			\atma; \atmb; \atmc &.
		&\prgu: \quad
			\lpnot \atma &\lpif \lpnot \atmc.
		&\prgv: \quad
			\lpnot \atma &\lpif \lpnot \atmc. \\
		&& \atma &\lpif \atmc.
		& \atma &\lpif \atmc. \\
		&& \atmc &\lpif \atma.
		& \atmc &\lpif \atma. \\
		&&&& \lpnot \atmb &.
	\end{align*}
\end{remark}

\begin{remark}
	[\pup{Associativity} and \terone{}, \tertwo{}, \terthree{}, \terfour{},
		\terfive{}]
	\label{ex:counterexample to associativity} 
	The rule $\rla = (\lpnot \atma.)$, when updated by $\rlb = (\atma \lpif
	\atmb.)$, must be weakened, anticipating the potential conflict. In the case
	of \terone-, \tertwo-, \terthree{}, \terfour{} and \terfive-based operators,
	the resulting rule is $\rla' = (\lpnot \atma \lpif \lpnot \atmb.)$ (or
	another \RE-equivalent rule). Consider the following rule bases:
	\[
		\rba = \set{\atma.}, \quad
		\rbu = \set{\lpnot \atma.}, \quad
		\rbv = \set{\atma \lpif \atmb., \atmb \lpif \atma.} \enspace.
	\]
	Note that \pup{Associativity} states the following: $\rba \uopr (\rbu \uopr
	\rbv) \equiv (\rba \uopr \rbu) \uopr \rbv$. However, while in $(\rba \uopr
	\rbu) \uopr \rbv$ the fact from $\rba$ is completely annihilated (i.e.
	transformed into a tautological rule) due to the negative fact $\rla$ in
	$\rbu$, this does not happen in $\rba \uopr (\rbu \uopr \rbv)$ because the
	$\rla$ is first weakened into $\rla'$. As a consequence, $\rba \uopr (\rbu
	\uopr \rbv)$ has one extra stable model comparing to $(\rba \uopr \rbu)
	\uopr \rbv$: $\set{\atma, \atmb}$. This implies that \pup{Associativity}
	does not hold for \terone-, \tertwo-, \terthree-, \terfour- and
	\terfive-based rule update operators under \SM-equivalence, nor under any
	stronger equivalence.
\end{remark}

\begin{proposition}
	Let $\uopr$ be a \ter-based rule update operator. Then $\uopr$ satisfies
	\pu{2.1} and \pu{5} with respect to \RMR{}, \SMR{}, \RE{} and
	\SE{}.
\end{proposition}
\begin{proof}
	Under \RMR-entailment \pu{2.1} follows from the fact that $\stria$ is
	a subset of $\aug{\stria, \sstri}$ and \pu{5} follows from the fact
	that $\aug{\stria, \sstria}$ is a subset of $\aug{\stria, \sstria \cup
	\sstrib}$. For the remaining notions of program entailment this follows
	from Proposition~\ref{prop:eq and ent comparison}.
\end{proof}


\begin{remark}
	[\pu{2.1} under $\entSR$ and $\entRR$]
	\label{ex:counterexample to U2.1}
	Consider again the rules $\rla$, $\rlb$ from Remark~\ref{ex:counterexample
	to associativity} and rule bases $\rba = \set{\rla}$, $\rbu = \set{\rlb}$.
	Note that \pu{2.1} states the following: $\rba \cup \rbu \ent \rba \uopr
	\rbu$. However, if $\uopr$ is \terone-, \tertwo-, \terthree-, \terfour- or
	\terfive-based, $\rba \uopr \rbu$ will contain $\rla'$ (or another
	\RE-equivalent rule or program) which results from weakening of $\rla$ by
	$\rlb$. Consequently, when \SR- or \RR-entailment is used, $\rba \cup \rbu$
	cannot entail $\rba \uopr \rbu$ simply because $\rla'$ (or another
	\RE-equivalent rule or program) does not belong to $\rba \cup \rbu$.
\end{remark}

\begin{remark}
	[\pu{2.2} and Rule Updates]
	\label{ex:counterexample to U2.2} 
	Consider $\rba = \set{\atma.}$ and $\rbu = \set{\lpnot \atma.}$ and note
	that \pu{2.2} states the following: $(\rba \cup \rbu) \uopr \rbu \ent \rba$.
	In other words, it requires that
	\[
		\set{\atma., \lpnot \atma.} \uopr \set{\lpnot \atma.} \ent \atma \enspace.
	\]
	In the presence of \pu{1} this amounts to postulating that one can never
	recover from an inconsistent state. Such a requirement is out of line with
	the way these situations are treated in state-of-the-art approaches to rule
	updates which allow for recovery from an inconsistent state if all involved
	conflicts are resolved by the update. Note that, though for different
	reasons, \bu{2.2} has also been subject of harsh criticism in belief update
	literature \citep{Herzig1999}.
\end{remark}

\begin{proposition}
	Let $\uopr$ be a \ter-based rule update operator where $\er(\stri, \tris)
	\subseteq \stri$ for all $\stri \subseteq \tris$. Then $\uopr$ satisfies
	\pu{4}, \pua{4.1} and \pua{4.2}
	with respect to \RR{}.
\end{proposition}
\begin{proof}
	Principle \pu{4} can be verified straightforwardly and \pua{4.1} as well as
	\pua{4.2} are its consequences. The condition on \ter{} is necessary to
	ensure that $\aug{\stria, \modrer{\rbu}} = \stria$ whenever $\rbu$ is
	tautological, to keep it in line with the case when $\rbu = \emptyset$.
\end{proof}


\section{Proofs: Condensing into a Nested Program}

\label{app:proofs}

\begin{lemma}
	\label{lemma:ju binary operator:1}
	For any DLP $\dprg = \seq{\prg_\lia}_{\lia < \lng}$, $\biguoprju
	\tones{\dprg}$ consists of the following rules:
	\begin{textenum}[1.]
		\item for all $\rla \in \tones{\expv{\prg_\lia}}$ with $\lia < \lng - 1$,
			the nested rule
			$
				\Br{
					\hrla
					\lpif
					\brla
					\land
					\bigland_{\lia < \lib < \lng}
						\lpnot \Acond{\tones{\expv{\prg_\lib}}}{\lcmp{\hrla}}
					.
				}
			$;

		\item all nested rules in $\tones{\expv{\prg_{\lng - 1}}}$.
	\end{textenum}
\end{lemma}
\begin{proof}
	Follows by induction on $\lng$.
%
\end{proof}

\begin{proposition}
	\label{prop:ju binary operator}
	Let $\twia$, $\twib$ be interpretations and $\dprg$ a DLP. Then,
	\begin{align*}
		\twia &\ent \Br{
			\all{\expv{\dprg}} \setminus \rejju{\expv{\dprg}, \twib}
		}^\twib
		&& \text{if and only if}
		& \twia &\ent \Br{
			\biguoprju \tones{\dprg}
		}^\twib
		\enspace.
	\end{align*}
\end{proposition}
\begin{proof}
	First suppose that $\twia \ent (\all{\expv{\dprg}} \setminus
	\rejju{\expv{\dprg}, \twib})^\twib$ and take some rule $\rlb^\twib \in
	(\biguoprju \tones{\dprg})^\twib$. We need to prove that $\twia \ent
	\rlb^\twib$. It follows from Lemma~\ref{lemma:ju binary operator:1} that one
	of the following cases occurs:
	\begin{textenum}[a)]
	\item If $\rlb \in \tones{\expv{\prg_{\lng - 1}}}$, then it immediately
		follows that $\rlb \in \tones{(\all{\expv{\dprg}} \setminus
		\rejju{\expv{\dprg}, \twib})}$. Consequently, since $\twia \ent \Br{
		\all{\expv{\dprg}} \setminus \rejju{\expv{\dprg}, \twib} }^\twib$,
		it follows that $\twia \ent \rlb^\twib$.

		\item Otherwise,
			$
				\rlb =
				\Br{
					\hrla
					\lpif
					\brla
					\land
					\bigland_{\lia < \lib < \lng}
						\lpnot \Acond{\tones{\expv{\prg_\lib}}}{\lcmp{\hrla}}
					.
				}
			$
			for some $\rla \in \tones{\expv{\prg_\lia}}$ with $\lia < \lng - 1$.
			Suppose that $\twia \ent \brlb^\twib$. Then $\twia \ent \brla^\twib$ and
			it also follows that for all $\lib$ with $\lia < \lib < \lng$, $\twib
			\nent \acond{\tones{\expv{\prg_\lib}}}{\lcmp{\hrla}}$. Consequently,
			$\rla \in \tones{(\all{\dprg} \setminus \rejju{\dprg, \twib})}$ and by
			the assumption we conclude that $\twia \ent \rla^\twib$. Hence, from
			$\twia \ent \brla^\twib$ it follows that $\twia \ent \hrla^\twib$,
			implying that $\twia \ent \rlb^\twib$.
	\end{textenum}

	For the converse implication, suppose that $\twia \ent (\biguoprju
	\tones{\dprg})^\twib$ and take some $\rla \in \tones{\Br{ \all{\expv{\dprg}}
	\setminus \rejju{\expv{\dprg}, \twib} }}$. We need to show that $\twia \ent
	\rla^\twib$. If $\rla \in \tones{\expv{\prg_{\lng - 1}}}$, then it follows
	by Lemma~\ref{lemma:ju binary operator:1} that $\rla \in (\biguoprju
	\tones{\dprg})$ and since $\twia \ent \Br{ \biguoprju \tones{\dprg} }^\twib$
	by assumption, we can immediately conclude that $\twia \ent \rla^\twib$.

	In the principal case, $\rla \in \tones{\expv{\prg_\lia}}$ for some $\lia <
	\lng - 1$. Consequently, by Lemma~\ref{lemma:ju binary operator:1},
	$\biguoprju \tones{\dprg}$ contains a rule
	$
		\rlb = \Br{
			\hrla \lpif \brla
				\land \bigland_{\lia < \lib < \lng}
				\lpnot \Acond{\tones{\expv{\prg_\lib}}}{\lcmp{\hrla}}
				.
		}
	$.
	Since $\rla$ is not part of the set of rejected rules, we can conclude that
	for all $\lib$ with $\lia < \lib < \lng$, $\twib \nent
	\acond{\tones{\expv{\prg_\lib}}}{\lcmp{\hrla}}$. Hence,
	$
		\brlb^\twib = \brla^\twib
			\land \bigland_{\lia < \lib < \lng} \top
	$.
	It follows that if $\twia \ent \brla^\twib$, then $\twia \ent \brlb^\twib$
	and using our assumption we conclude that $\twia \ent \hrlb^\twib$. Since
	$\hrla = \hrlb$, we demonstrated that $\twia \ent \rla^\twib$.
\end{proof}

\begin{lemma}
	\label{lemma:as binary operator:1}
	For any DLP $\dprg = \seq{\prg_\lia}_{\lia < \lng}$, $\biguopras
	\tones{\dprg}$ consists of the following rules:
	\begin{textenum}[1.]
		\item for all $\rla \in \tones{\expv{\prg_\lia}}$ with $\lia < \lng - 1$,
			the nested rule
			$
				\Br{
					\hrla
					\lpif
					\brla
					\land
					\bigland_{\lia < \lib < \lng}
						\lpnot \Acond{\tones{\expv{\prg_\lib}}}{\lcmp{\hrla}}
					.
				}
			$;

		\item for all $\rla \in \tones{\prg_\lia}$ with $\hrla \in \atms$ and $\lia <
			\lng$, the nested rule
			$
				\Br{
					\hrla \lor \lcmp{\hrla} \lpif \brla.
				}
			$;

		\item all nested rules in $\tones{\expv{\prg_{\lng - 1}}}$.
	\end{textenum}
\end{lemma}
\begin{proof}
	Follows by induction on $\lng$.
%
\end{proof}

\begin{proposition}
	\label{prop:as binary operator:1}
	Let $\dprg$ be a DLP and $\twib$ an interpretation. If $\twib$ is a
	\AS-model of $\dprg$, then it is a stable model of $\biguopras
	\tones{\dprg}$.
\end{proposition}
\begin{proof}
	In order to show that $\twib$ is a stable model of $\biguopras
	\tones{\dprg}$, we first establish that $\twib$ satisfies $(\biguopras
	\tones{\dprg})^\twib$. Take some rule $\rlb$ from $\biguopras
	\tones{\dprg}$. We need to prove that $\twib \ent \rlb^\twib$. Due to
	Lemma~\ref{lemma:as binary operator:1}, we need to consider the following
	cases:
	\begin{textenum}[1)]
		\item In the first case,
			$
				\rlb = \Br{
					\hrla
					\lpif
					\brla
					\land
					\bigland_{\lia < \lib < \lng}
						\lpnot \Acond{\tones{\expv{\prg_\lib}}}{\lcmp{\hrla}}
					.
				}
			$
			where $\rla \in \tones{\expv{\prg_\lia}}$ and $\lia < \lng - 1$. Suppose
			that $\twib \ent \brlb^\twib$. Then $\twib \ent \brla^\twib$ and for all
			$\lib$ such that $\lia < \lib < \lng$, $\twib \nent
			\acond{\tones{\expv{\prg_\lib}}}{\lcmp{\hrla}}$. Thus, we can conclude
			that $\rla \in \tones{(\all{\dprg} \setminus \rejas{\dprg, \twib})}$, by
			the assumption, $\twib \ent \rla^\twib$. So since $\twib \ent
			\brla^\twib$, it also holds that $\twib \ent \hrla^\twib$ and since
			$\hrla = \hrlb$, we conclude that $\twib \ent \rlb^\twib$.

		\item In the second case, $\rlb = \Br{ \hrla \lor \lcmp{\hrla} \lpif
			\brla.}$ for some $\rla \in \tones{\prg_\lia}$ with $\hrla \in \atms$ and
			$\lia < \lng$. If $\twib \ent \hrla$, then it easily follows that $\twib
			\ent \rlb^\twib$. On the other hand, if $\twib \nent \hrla$, then the head
			of $\rlb^\twib$ contains $\top$ as the second disjunct and, once again, it
			follows that $\twib \ent \rlb^\twib$.

		\item In the third case, $\rlb \in \tones{\expv{\prg_{\lng - 1}}}$. It
			immediately follows that
			$
				\rlb \in \tones{\Br{
					\all{\expv{\dprg}} \setminus \rejas{\expv{\dprg}, \twib}
				}}
			$.
			Thus, our assumption implies that $\twib \ent \rlb^\twib$.
	\end{textenum}

	It remains to verify that $\twib$ is also subset-minimal among
	interpretations that satisfy $(\biguopras \tones{\dprg})^\twib$. To show
	that this is the case, take an interpretation $\twia$ with $\twia \subseteq
	\twib$ that satisfies $(\biguopras \tones{\dprg})^\twib$. In the following
	we will prove that $\twia \ent \Br{ \all{\expv{\dprg}} \setminus
	\rejas{\expv{\dprg}, \twib} }^\twib$. Since $\twib$ is subset-minimal
	among interpretations satisfying this program, it will follow that $\twia =
	\twib$ as desired.

	Take some rule $\rla \in \tones{\Br{\all{\expv{\dprg}} \setminus
	\rejas{\expv{\dprg}, \twib}}}$, our goal is to prove that $\twia \ent
	\rla^\twib$. We consider the following cases:
	\begin{textenum}[1)]
		\item In case $\rla \in \tones{\expv{\prg_\lia}}$ with $\lia < \lng - 1$
			and there is no rule $\rla' \in \expv{\prg_\lib}$ with $\lia < \lib <
			\lng$ such that $\rla' \confl{}{} \rla$ and $\twib \ent \brl[\rla']$, we
			can use Lemma~\ref{lemma:as binary operator:1} to conclude that
			$\biguopras \tones{\dprg}$ contains a rule
			$
				\rlb = \Br{
					\hrla
					\lpif
					\brla
					\land
					\bigland_{\lia < \lib < \lng}
						\lpnot \Acond{\tones{\expv{\prg_\lib}}}{\lcmp{\hrla}}
					.
				}
			$.
			It also follows that for all $\lib$ with $\lia < \lib < \lng$, $\twib
			\nent \acond{\tones{\expv{\prg_\lib}}}{\lcmp{\hrla}}$, so that
			$\brlb^\twib = \brla^\twib \land \bigland_{\lia < \lib < \lng} \top$.
			Thus, if $\twia \ent \brla^\twib$, then $\twia \ent \brlb^\twib$ and by
			the assumption that $\twia \ent (\biguopras \tones{\dprg})^\twib$ we
			obtain $\twia \ent \hrlb^\twib$, so it also follows that $\twia \ent
			\rla^\twib$.

		\item In case $\rla \in \tones{\expv{\prg_\lia}}$ with $\lia < \lng - 1$
			and there is a rule $\rla' \in \expv{\prg_\lib}$ with $\lia < \lib <
			\lng$ such that $\rla' \confl{}{} \rla$ and $\twib \ent \brl[\rla']$, it
			follows that since $\rla$ is unrejected, $\rla'$ is itself rejected.
			Consequently, there is also an unrejected rule $\rla'' \in
			\expv{\prg_{\lic}}$ with $\lib < \lic < \lng$ such that $\rla''
			\confl{}{} \rla'$ and $\twib \ent \brl[\rla'']$. Furthermore,
			$\hrl[\rla''] = \hrla$ and by the assumption we know that $\twib$
			satisfies the rule $\rla''$, so we can conclude that $\twib \ent \hrla$.
			If $\hrla$ is a default literal, then it follows that $(\hrla)^\twib =
			\top$, so trivially $\twia \ent \rla^\twib$.

			If $\hrla$ is an atom, then, by Lemma~\ref{lemma:as binary operator:1},
			$\biguopras \tones{\dprg}$ contains a rule $\rlb = \Br{ \hrla \lor
				\lcmp{\hrla} \lpif \brla. }$. Note that $\hrlb^\twib = \hrla \lor
			\bot$. Thus, if $\twia \ent \brla^\twib$, then $\twia \ent \brlb^\twib$
			and it follows from our assumption that $\twia \ent \hrlb^\twib$, so
			that $\twia \ent \hrla^\twib$. Hence, $\twia \ent \rla^\twib$.

		\item In case $\rla \in \tones{\expv{\prg_{\lng - 1}}}$, it immediately
			follows from Lemma~\ref{lemma:as binary operator:1} that $\rla \in
			\Br{\biguopras \tones{\dprg}}$. Thus, by the assumption that
			$\twia \ent (\biguopras \tones{\dprg})^\twib$ we obtain $\twia \ent
			\rla^\twib$. \qedhere
	\end{textenum}
\end{proof}

\begin{proposition}
	\label{prop:as binary operator:2}
	Let $\dprg$ be a DLP and $\twib$ an interpretation. If $\twib$ is a stable
	model of $\biguopras \tones{\dprg}$, then it is a \AS-model of $\dprg$.
\end{proposition}
\begin{proof}
	To show that $\twib$ is a \AS-model of $\dprg$, we first establish that
	$\twib$ satisfies
	$
		\prg
		=
		\Br{
			\all{\expv{\dprg}} \setminus \rejas{\expv{\dprg}, \twib}
		}^\twib
	$.
	Take some rule $\rla \in \tones{\Br{\all{\expv{\dprg}} \setminus
	\rejas{\expv{\dprg}, \twib}}}$, our goal is to prove that $\twib \ent
	\rla^\twib$. We consider the following cases:
	\begin{textenum}[1)]
		\item In case $\rla \in \tones{\expv{\prg_\lia}}$ with $\lia < \lng - 1$
			and there is no rule $\rla' \in \expv{\prg_\lib}$ with $\lia < \lib <
			\lng$ such that $\rla' \confl{}{} \rla$ and $\twib \ent \brl[\rla']$, we
			can use Lemma~\ref{lemma:as binary operator:1} to conclude that
			$\biguopras \tones{\dprg}$ contains a rule
			$
				\rlb = \Br{
					\hrla
					\lpif
					\brla
					\land
					\bigland_{\lia < \lib < \lng}
						\lpnot \Acond{\tones{\expv{\prg_\lib}}}{\lcmp{\hrla}}
					.
				}
			$.
			It also follows that for all $\lib$ with $\lia < \lib < \lng$, $\twib
			\nent \acond{\tones{\expv{\prg_\lib}}}{\lcmp{\hrla}}$, so that
			$\brlb^\twib = \brla^\twib \land \bigland_{\lia < \lib < \lng} \top$.
			Thus, if $\twib \ent \brla^\twib$, then $\twib \ent \brlb^\twib$ and by
			the assumption that $\twib \ent (\biguopras \tones{\dprg})^\twib$ we
			obtain $\twib \ent \hrlb^\twib$, so it follows that $\twib \ent
			\rla^\twib$.

		\item In case $\rla \in \tones{\expv{\prg_\lia}}$ with $\lia < \lng - 1$
			and there is a rule $\rla' \in \expv{\prg_\lib}$ with $\lia < \lib <
			\lng$ such that $\rla' \confl{}{} \rla$ and $\twib \ent \brl[\rla']$, it
			follows that since $\rla$ is unrejected, $\rla'$ is itself rejected.
			Take the maximal index $\lic$ such that $\expv{\prg_\lic}$ contains a
			rule $\rla''$ with $\rla'' \confl{}{} \rla'$ and $\twib \ent
			\hrl[\rla'']$. It follows that $\rla''$ satisfies the condition of the
			previous case, and thus $\twib \ent \hrl[\rla'']^\twib$. Since $\hrla =
			\hrl[\rla'']$, we conclude that $\twib \ent \rla^\twib$.

		\item In case $\rla \in \tones{\expv{\prg_{\lng - 1}}}$, it immediately
			follows from Lemma~\ref{lemma:as binary operator:1} that $\rla \in
			\Br{ \biguopras \tones{\dprg} }$. Thus, by the assumption that
			$\twib \ent (\biguopras \tones{\dprg})^\twib$ we obtain that $\twib \ent
			\rla^\twib$.
	\end{textenum}

	It remains to verify that $\twib$ is also subset-minimal among
	interpretations that satisfy $\prg$. To show that this is the case, take an
	interpretation $\twia$ with $\twia \subseteq \twib$ that satisfies $\prg$.
	In the following we will prove that $\twia$ also satisfies the program
	$(\biguopras \tones{\dprg})^\twib$. Since $\twib$ is subset-minimal among
	interpretations satisfying this program, it will follow that $\twia = \twib$
	as desired.

	So take some $\rlb \in (\biguopras \tones{\dprg})$, our goal is to prove
	that $\twia \ent \rlb^\twib$. Due to Lemma~\ref{lemma:as binary operator:1},
	we need to consider the following cases:
	\begin{textenum}[1)]
		\item In the first case,
			$
				\rlb = \Br{
					\hrla
					\lpif
					\brla
					\land
					\bigland_{\lia < \lib < \lng}
						\lpnot \Acond{\tones{\expv{\prg_\lib}}}{\lcmp{\hrla}}
					.
				}
			$
			where $\rla \in \tones{\expv{\prg_\lia}}$ and $\lia < \lng - 1$. Suppose
			that $\twia \ent \brlb^\twib$. Then $\twia \ent \brla^\twib$ and for all
			$\lib$ such that $\lia < \lib < \lng$, $\twib \nent
			\acond{\tones{\expv{\prg_\lib}}}{\lcmp{\hrla}}$. Thus, we can conclude
			that $\rla \in \tones{\Br{\all{\dprg} \setminus \rejas{\dprg, \twib}}}$.
			Furthermore, from $\twia \ent \brla^\twib$ and the assumption that
			$\twia \ent \prg$, it follows that $\twia \ent \hrla^\twib$.
			Consequently, since $\hrla = \hrlb$, $\twia \ent \rlb^\twib$.

		\item In the second case, $\rlb = (\hrla \lor \lcmp{\hrla} \lpif \brla.)$
			for some $\rla \in \tones{\prg_\lia}$ with $\hrla \in \atms$ and $\lia <
			\lng$. If $\rla$ is not rejected, then it follows from our assumption
			that $\twia \ent \rla^\twib$. In the principal case, there exists a rule
			$\rla' \in \expv{\prg_\lib}$ with $\lia < \lib < \lng$ such that $\rla'
			\confl{}{} \rla$ and $\twib \ent \brl[\rla']$. Furthermore, $\rla'$
			itself is not rejected, so due to our previous considerations we can
			conclude that $\twib \ent \hrl[\rla']$. Note that since $\hrla$ is an
			atom, $\hrl[\rla']$ is a default literal, so the rule $\rlb^\twib$ has
			$\top$ as one of the disjuncts in its head. Thus, $\twia$ trivially
			satisfies $\rlb^\twib$.

		\item In the third case, $\rlb \in \tones{\expv{\prg_{\lng - 1}}}$. It
			immediately follows that $\rlb \in \tones{\Br{\all{\expv{\dprg}}
			\setminus \rejas{\expv{\dprg}, \twib}}}$. Thus, our assumption implies
			that $\twia \ent \rlb^\twib$. \qedhere
	\end{textenum}
\end{proof}

\begin{proof}
	[\textbf{Proof of Theorem~\ref{thm:ju binary operator}}]
	\label{proof:thm:ju binary operator}
	Follows from Propositions~\ref{prop:ju binary operator}, \ref{prop:as binary
	operator:1} and \ref{prop:as binary operator:2}.
\end{proof}

\section{Proofs: Condensing into a Disjunctive Program}

\begin{remark}
	Throughout the following proofs we abuse notation by ignoring differences
	between formulas that can be eliminated by regrouping and reordering
	conjuncts and disjuncts within them. That is, when a formula can be obtained
	from another formula only by using the commutative and associative laws for
	conjunction and disjunction, we consider the two formulas identical. We can
	afford to do this because the order and grouping of conjuncts and disjuncts
	has no effect on the resulting semantics.
\end{remark}

\begin{definition}
	[Strong Equivalence \cite{Lifschitz2001}]
	Let $\prga$, $\prgb$ be programs. We say that \emph{$\prga$ is strongly
	equivalent to $\prgb$} if for every program $\prgc$, the stable models of
	$\prga \cup \prgc$ coincide with the stable models of $\prgb \cup \prgc$.
	Strong equivalence is extended to rules by treating each rule $\rl$ as the
	program $\set{\rl}$.
\end{definition}

\begin{proposition}
	[\cite{Lifschitz2001,Turner2003}]
	\label{prop:strong equivalence condition}
	Let $\prga$, $\prgb$ be programs and $\rl$ a rule. If for all
	interpretations $\twia$, $\twib$ with $\twia \subseteq \twib$, $\twib \ent
	\prga \land \twia \ent \prga^\twib$ if and only if $\twib \ent \prgb \land
	\twia \ent \prgb^\twib$, then $\prga$ is strongly equivalent to $\prgb$.
\end{proposition}
\begin{proof}
	Take some program $\prgc$ and some interpretation $\twib$. $\twib$ is a
	stable model of $\prga \cup \prgc$ if and only if $\twib \ent (\prga \cup
	\prgc)^\twib$ and $\forall \twia \subsetneq \twib : \twia \nent (\prga \cup
	\prgc)^\twib$. Due to the definition of reduct and the assumption, this is
	equivalent to $\twib \ent (\prgb \cup \prgc)^\twib$ and $\forall \twia
	\subsetneq \twib : \twia \nent (\prgb \cup \prgc)^\twib$. In other words,
	$\twib$ is a stable model of $\prga \cup \prgc$ if and only if it is a
	stable model of $\prgb \cup \prgc$.
\end{proof}

\begin{corollary}
	\label{cor:strong equivalence condition:1}
	Let $\prga$, $\prgb$ be programs and $\rl$ a rule. If for all
	interpretations $\twia$, $\twib$, $\twia \ent \prga^\twib$ if and only if
	$\twia \ent \prgb^\twib$ \enspace, then $\prga$ is strongly equivalent to
	$\prgb$.
\end{corollary}
\begin{proof}
	Follows from Proposition~\ref{prop:strong equivalence condition} and the
	fact that for every interpretation $\twic$ and program $\prgc$, $\twic \ent
	\prgc$ if and only if $\twic \ent \prgc^\twic$.
\end{proof}

\begin{corollary}
	\label{cor:strong equivalence condition:2}
	Let $\frma$, $\frmb$ be formulas and $\rl$ a rule. If for all
	interpretations $\twia$, $\twib$, $\twia \ent \frma^\twib$ if and only if
	$\twia \ent \frmb^\twib$, then the rules $(\hrla \lpif \brla \land \frma.)$
	and $(\hrla \lpif \brla \land \frmb.)$ are strongly equivalent.
\end{corollary}
\begin{proof}
	Follows from the definition of rule reduct and from
	Corollary~\ref{cor:strong equivalence condition:1}.
\end{proof}

\begin{definition}
	Given a set of formulas $\sfrm$, we define $\lpnot \sfrm = \Set{\lpnot \frm
	| \frm \in \sfrm}$.
\end{definition}

\begin{lemma}
	\label{lemma:strong equivalence:1}
	Let $\prg$ be a program, $\lit$ a literal, $\sfrm$ be a set of formulas and
	$\twia$, $\twib$ interpretations. Then,
	\begin{textenum}[(i)]
		\item
			\label{lemma:strong equivalence:1:1}
			$\twia \ent \Br{ \bigland \lpnot \sfrm }^\twib$ if and only if
			$\twia \ent \Br{ \lpnot \biglor \sfrm }^\twib$;

		\item
			\label{lemma:strong equivalence:1:2}
			$\twia \ent (\lpnot \acond{\tones{\prg}}{\lit})^\twib$ if and only if
			$
				\twia \ent \Br{
					\biglor_{\sblk \in \sblks{\prg}{\lit}}
					\bigland \lpnot \lpnot \sblk^+
					\land
					\bigland \lpnot \sblk^-
				}^\twib
			$.
	\end{textenum}
\end{lemma}
\begin{proof}
	~
	\begin{textenum}[(i)]
		\item By the definition of reduct, $\twia \ent (\bigland \lpnot
			\sfrm)^\twib$ if and only if for all $\frm \in \sfrm$ it holds that
			$\twib \nent \frm$. This in turn holds if and only if $\twib \nent
			\biglor \sfrm$, which is the case if and only if $\twia \ent (\lpnot
			\biglor \sfrm)^\twib$.

		\item Suppose that $\acond{\tones{\prg}}{\lit}$ is the formula
			$
				(\lit_1^1 \land \dotsb \land \lit_{\lic_1}^1)
				\lor
				\dotsb
				\lor
				(\lit_1^\lng \land \dotsb \land \lit_{\lic_\lng}^\lng)
			$.
			By the definition of reduct, $\twia \ent (\lpnot
			\acond{\tones{\prg}}{\lit})^\twib$ if and only if $\twib \nent
			\acond{\tones{\prg}}{\lit}$. Equivalently, for every $\lib$ with $1 \leq
			\lib \leq \lng$ there exists some $\lia_\lib$ with $1 \leq \lia_\lib
			\leq \lic_\lib$ such that $\twib \nent \lit_{\lia_\lib}^\lib$. By the
			definition of blocking sets, this is equivalent to $\twib \ent \sblk$
			for some $\sblk \in \sblks{\prg}{\lit}$. Equivalently, for some $\sblk
			\in \sblks{\prg}{\lit}$,
			\begin{align*}
				& \forall \lit \in \lpnot \sblk^+ : \twib \nent \lit
				&& \text{and}
				&& \forall \lit \in \sblk^- : \twib \nent \lit
				\enspace,
			\end{align*}
			or in other words,
			$
				\Br{
					\bigland \lpnot \lpnot \sblk^+ \land \bigland \lpnot \sblk^-
				}^\twib
				=
				\bigland_{\lit \in \sblk} \top
			$.
			Equivalently, we can also write 
			$
				\textstyle
				\twia \ent \Br{
					\biglor_{\sblk \in \sblks{\prg}{\lit}}
					\bigland \lpnot \lpnot \sblk^+
					\land
					\bigland \lpnot \sblk^-
				}^\twib
			$.
			\qedhere
	\end{textenum}
\end{proof}


\begin{lemma}
	\label{se:disjunction in body}
	Let $\rla$ be a rule and $\sfrm$ a set of formulas. Then the rule $(\hrla
	\lpif \brla \land \biglor \sfrm.)$ is strongly equivalent to the program
	$\Set{\hrla \lpif \brla \land \frm | \frm \in \sfrm}$.
\end{lemma}
\begin{proof}
	Let $\rlb$ denote the rule and $\prg$ the program. According to
	Corollary~\ref{cor:strong equivalence condition:1}, it suffices to prove
	that for all interpretations $\twia$, $\twib$, it holds that $\twia \ent
	\rlb^\twib$ if and only if it holds that $\twia \ent \prg^\twib$. This
	easily follows from the fact that $\twia \ent \rlb^\twib$ if and only if
	$\twia \ent \brlb^\twib$ implies $\twia \ent \hrlb^\twib$, or, equivalently,
	for all $\frm \in \sfrm$, $\twia \ent \brla^\twib \land \frm^\twib$ implies
	$\twia \ent \hrla^\twib$, which is another way of writing $\twia \ent
	\prg^\twib$.
\end{proof}

\begin{lemma}
	\label{lemma:se double negation in body}
	Let $\rla$ be a rule and $\frm$ a formula. Then the rules $(\hrla \lpif
	\brla \land \lpnot \lpnot \frm.)$ and $(\hrla \lor \lpnot \frm \lpif
	\brla.)$ are strongly equivalent.
\end{lemma}
\begin{proof}
	Let $\rlb_1$ denote the first rule and $\rlb_2$ the second. By
	Corollary~\ref{cor:strong equivalence condition:1}, it suffices to prove
	that for all interpretations $\twia$, $\twib$, it holds that $\twia \ent
	\rlb_1^\twib$ if and only if it holds that $\twia \ent \rlb_2^\twib$. If
	$\twib \ent \frm$, then $\rlb_1^\twib = (\hrla^\twib \lpif \brla^\twib \land
	\top)$ and $\rlb_2^\twib = (\hrla^\twib \lor \bot \lpif \brla^\twib)$, so it
	follows that $\twia \ent \rlb_1^\twib$ holds if and only if $\twia \ent
	\rlb_2^\twib$.

	On the other hand, if $\twib \nent \frm$, then $\rlb_1^\twib = (\hrla \lpif
	\brla \land \bot)$ $\rlb_2^\twib = (\hrla \lor \top \lpif \brla)$ and
	trivially both $\twia \ent \rlb_1^\twib$ and $\twia \ent \rlb_2^\twib$ hold.
\end{proof}

\begin{lemma}
	\label{lemma:se double negation in body with negative head}
	Let $\rla$ be a rule such that $\hrla$ is a default literal and $\brla$ is a
	conjunction of literals and double-negated atoms. For any atom $\atm$, the
	rules $(\hrla \lpif \brla \land \lpnot \lpnot \atm.)$ and $(\hrla \lpif
	\brla \land \atm.)$ are are strongly equivalent.
\end{lemma}
\begin{proof}
	Let $\hrla = \lpnot \atmb$, where $\atmb$ is an atom. Also,
	let $\rlb_1$ denote the first rule and $\rlb_2$ the second. By
	Proposition~\ref{prop:strong equivalence condition}, it suffices to prove
	that for all interpretations $\twia$, $\twib$ with $\twia \subseteq \twib$,
	\begin{align*}
		& \twib \ent \rlb_1 \land
			\twia \ent \rlb_1^\twib
		&& \text{if and only if}
		&& \twib \ent \rlb_2 \land
			\twia \ent \rlb_2^\twib
		\enspace,
	\end{align*}

	First suppose that $\twib \ent \rlb_1$ and $\twia \ent \rlb_1^\twib$. Then
	clearly $\twib \ent \rlb_2$ and it remains to prove that $\twia \ent
	\rlb_2^\twib$. Suppose that $\twia \ent \brl[\rlb_2]^\twib$. Then $\twia
	\ent \brla^\twib$ and from $\twia \subseteq \twib$ we conclude that $\twib
	\ent \atm$, so $\twia \ent (\lpnot \lpnot \atm)^\twib$. Consequently,
	$\twia \ent \brl[\rlb_1]^\twib$ and from the assumption that $\twia \ent
	\rlb_1^\twib$ we conclude that $\twia \ent \hrla^\twib$ as desired.

	Now suppose that $\twib \ent \rlb_2$ and $\twia \ent \rlb_2^\twib$. Then
	clearly $\twib \ent \rlb_1$ and it remains to prove that $\twia \ent
	\rlb_1^\twib$. Suppose that $\twia \ent \brl[\rlb_1]^\twib$. Then $\twia
	\ent \brla^\twib$, so it follows from $\twia \subseteq \twib$ that $\twib
	\ent \brla$. Furthermore, from $\twia \ent (\lpnot \lpnot \atm)^\twib$ it
	follows that $\twib \ent \atm$. Thus, $\twib \ent \brl[\rlb_2]$ and from
	$\twib \ent \rlb_2$ we conclude that $\twib \ent \hrla$. This implies that
	$\hrla^\twib = (\lpnot \atmb)^\twib = \top$ and, thus, $\twia \ent
	\hrla^\twib$ as desired.
\end{proof}

\begin{lemma}
	\label{lemma:se double negation in body with choice rule}
	Let $\rla$ be a rule and $\frm$ a formula such that $\hrla$ is an atom and
	both $\brla$ and $\frm$ are conjunctions of literals and double-negated
	atoms. For any atom $\atm$, the programs
	$
		\Set{
			\hrla \lor \lcmp{\hrla} \lpif \brla;\;
			\hrla \lpif \brla \land \frma \land \lpnot \lpnot \atm
		}
	$
	and
	$
		\Set{
			\hrla \lor \lcmp{\hrla} \lpif \brla;\;
			\hrla \lpif \brla \land \frma \land \atm
		}
	$
	are strongly equivalent.
\end{lemma}
\begin{proof}
	Let $\prg_1$ denote the first program and $\prg_2$ the second one. Also, let
	$\rlb$ denote the rule $(\hrla \lor \lcmp{\hrla} \lpif \brla)$, $\rlb_1$ the
	rule $(\hrla \lpif \brla \land \frma \land \lpnot \lpnot \atm)$ and
	$\rlb_2$ the rule $(\hrla \lpif \brla \land \frma \land \atm)$. Thus,
	$\prg_1 = \set{\rlb, \rlb_1}$ and $\prg_2 = \set{\rlb, \rlb_2}$.

	According to Proposition~\ref{prop:strong equivalence condition}, it
	suffices to prove that for all interpretations $\twia$, $\twib$ with $\twia
	\subseteq \twib$,
	\begin{align*}
		& \twib \ent \prg_1 \land
			\twia \ent \prg_1^\twib
		&& \text{if and only if}
		&& \twib \ent \prg_2 \land
			\twia \ent \prg_2^\twib
		\enspace,
	\end{align*}

	First assume that $\twib \ent \prg_1$ and $\twia \ent \prg_1^\twib$.
	Clearly, it follows that $\twib \ent \prg_2$ and $\twia \ent \rlb^\twib$, so
	it remains to verify that $\twia \ent \rlb_2^\twib$. Suppose that $\twia
	\ent \brl[\rlb_2]^\twib$. Then, since $\twia \subseteq \twib$, it follows
	that $\twib \ent \atm$ and so $(\lpnot \lpnot \atm)^\twib = \top$.
	Consequently, $\twia \ent \brl[\rlb_1]^\twib$ and from the assumption that
	$\twia \ent \prg_1^\twib$ we can conclude that $\twia \ent \hrla^\twib$ as
	desired.

	For the converse implication, assume that $\twib \ent \prg_2$ and $\twia
	\ent \prg_2^\twib$. It immediately follows that $\twib \ent \prg_1$ and
	$\twia \ent \rlb^\twib$, so it remains to verify that $\twia \ent
	\rlb_1^\twib$. Suppose that $\twia \ent \brl[\rlb_1]^\twib$. From $\twia
	\subseteq \twib$ and the assumption that both $\brla$ and $\frm$ are
	conjunctions of literals and double-negated atoms, we conclude
	that $\twib \ent \brl[\rlb_1]$, so it follows that $\twib \ent \hrla$.
	Consequently, $(\lcmp{\hrla})^\twib = \bot$ and since $\twia \ent
	\rlb^\twib$, $\twia \ent \brla^\twib$ implies that $\twia \ent \hrla^\twib$
	as desired.
\end{proof}

\begin{lemma}
	\label{lemma:blocking set decomposition}
	Let $\prga$, $\prgb$ be programs and $\lit$ a literal. Then,
	\[
		\sblks{\prga \cup \prgb}{\lit}
		=
		\Set{
			\sblka \cup \sblkb
			|
			\sblka \in \sblks{\prga}{\lit} \land \sblkb \in \sblks{\prgb}{\lit}
		}
		\enspace.
	\]
\end{lemma}
\begin{proof}
	Suppose that $\acond{\tones{\prga}}{\lit}$, $\acond{\tones{\prgb}}{\lit}$
	and $\acond{\tones{(\prga \cup \prgb)}}{\lit}$ are, respectively, of the forms
	\begin{align*}
		(\lit_1^1 \land \dotsb \land \lit_{\lic_1}^1)
		&\lor
		\dotsb
		\lor
		(\lit_1^\lnga \land \dotsb \land \lit_{\lic_\lnga}^\lnga)
		\enspace, \\
		(\lit_1^{\lnga + 1} \land \dotsb \land \lit_{\lic_{\lnga + 1}}^{\lnga + 1})
		&\lor
		\dotsb
		\lor
		(\lit_1^\lngb \land \dotsb \land \lit_{\lic_\lngb}^\lngb)
		\enspace, \\
		(\lit_1^1 \land \dotsb \land \lit_{\lic_1}^1)
		&\lor
		\dotsb
		\lor
		(\lit_1^\lngb \land \dotsb \land \lit_{\lic_\lngb}^\lngb)
		\enspace.
	\end{align*}
	By the definition, $\sblka \in \sblks{\prga}{\lit}$ and $\sblkb \in
	\sblks{\prgb}{\lit}$ if and only if
	\begin{align*}
		\sblka
		&=
		\Set{
			\lcmp{\lit_{\lia_1}^1},
			\dotsc,
			\lcmp{\lit_{\lia_\lnga}^\lnga}
		}
		&& \text{and}
		& \sblkb
		&=
		\Set{
			\lcmp{\lit_{\lia_{\lnga + 1}}^{\lnga + 1}},
			\dotsc,
			\lcmp{\lit_{\lia_\lngb}^\lngb}
		}
	\end{align*}
	where $1 \leq \lia_\lib \leq \lic_\lib$ for every $\lib$ with $1 \leq \lib
	\leq \lngb$. This is equivalent to $\sblka \cup \sblkb$ being a member of
	$\sblks{\prga \cup \prgb}{\lit}$.
\end{proof}

\begin{lemma}
	\label{lemma:ju disjunctive binary operator:1}
	Let $\dprg = \seq{\prg_\lia}_{\lia < \lng}$ be a DLP. Then $\biguoprjulor
	\dprg$ consists of the following rules:
	\begin{textenum}[1.]
		\item for all $\rla \in \expv{\prg_\lia}$ with $\lia < \lng - 1$ such that
			$\hrla = \set{\atm}$ for some $\atm \in \atms$, and all $\sblk \in
			\sblks{\bigcup_{\lia < \lib < \lng} \expv{\prg_\lib}}{\lpnot \atm}$, the
			rule
			$
				\Br{
					\hrla; \lcmp{\sblk^+}
					\lpif
					\brla, \lcmp{\sblk^-}
					.
				}
			$;

		\item for all $\rla \in \expv{\prg_\lia}$ with $\lia < \lng - 1$ such that
			$\hrla = \set{\lpnot \atm}$ for some $\atm \in \atms$, and all $\sblk
			\in \sblks{\bigcup_{\lia < \lib < \lng} \expv{\prg_\lib}}{\atm}$, the
			rule
			$
				\Br{
					\hrla
					\lpif
					\brla, \sblk^+, \lcmp{\sblk^-}
					.
				}
				\enspace;
			$;

		\item all rules in $\expv{\prg_{\lng - 1}}$.
	\end{textenum}
\end{lemma}
\begin{proof}
	Follows by induction on $\lng$ using Lemma~\ref{lemma:blocking set
	decomposition}.
\end{proof}

\begin{lemma}
	\label{lemma:as disjunctive binary operator:1}
	Let $\dprg = \seq{\prg_\lia}_{\lia < \lng}$ be a DLP. Then $\biguopraslor
	\dprg$ consists of the following rules:
	\begin{textenum}[1.]
		\item for all $\rla \in \expv{\prg_\lia}$ with $\lia < \lng - 1$ and all
			$\sblk \in \Sblks{\bigcup_{\lia < \lib < \lng}
			\expv{\prg_\lib}}{\lcmp{\hrla}}$, the rule
			$
				\Br{
					\hrla
					\lpif
					\brla, \sblk^+, \lcmp{\sblk^-}
					.
				}
			$;

		\item for all $\rla \in \prg_\lia$ with $\lia < \lng$ such that $\hrla =
			\set{\atm}$ for some $\atm \in \atms$, the rule
			$
				\Br{
					\atm; \lpnot \atm \lpif \brla.
				}
			$;

		\item all rules in $\expv{\prg_{\lng - 1}}$.
	\end{textenum}
\end{lemma}
\begin{proof}
	Follows by induction on $\lng$ using Lemma~\ref{lemma:blocking set
	decomposition}.
\end{proof}

\begin{proof}
	[\textbf{Proof of Theorem~\ref{thm:ju disjunctive binary operator}}]
	\label{proof:thm:ju disjunctive binary operator}
	~
	\begin{textenum}[(i)]
		\item Due to Theorem~\ref{thm:ju binary operator}, it suffices to show
			that the programs $\biguoprju \tones{\dprg}$ and $\biguoprjulor \dprg$
			have the same stable models. To see that this is indeed the case,
			consider the contents of these programs, as established in
			Lemmas~\ref{lemma:ju binary operator:1} and \ref{lemma:ju disjunctive
			binary operator:1}. For all $\rl \in \tones{\expv{\prg_\lia}}$ with
			$\lia < \lng - 1$, $\biguoprju \tones{\dprg}$ contains the nested rule
			$
				\Br{
					\hrla
					\lpif
					\brla \land \bigland_{\lia < \lib < \lng}
					\lpnot \acond{\tones{\expv{\prg_\lib}}}{\lcmp{\hrla}}
					.
				}
			$.
			By Lemma~\ref{lemma:strong equivalence:1}\eqref{lemma:strong
			equivalence:1:1} and Corollary~\ref{cor:strong equivalence condition:2},
			this rule is strongly equivalent to the rule
			$
				\Br{
					\hrla
					\lpif
					\brla \land \lpnot \biglor_{\lia < \lib < \lng}
					\acond{\tones{\expv{\prg_\lib}}}{\lcmp{\hrla}}
					.
				}
			$
			which, by the definition of activation condition, can also be written as
			$
				\Br{
					\hrla
					\lpif
					\brla \land \lpnot
					\acond{\bigcup_{\lia < \lib < \lng}
					\tones{\expv{\prg_\lib}}}{\lcmp{\hrla}}
					.
				}
			$.
			Furthermore, due to Lemma~\ref{lemma:strong
			equivalence:1}\eqref{lemma:strong equivalence:1:2} and
			Corollary~\ref{cor:strong equivalence condition:2}, the latter rule is
			strongly equivalent to the rule
			\[
				\hrla
				\lpif
				\brla \land
					\biglor_{\sblk \in \sblks
						{\bigcup_{\lia < \lib < \lng} \expv{\prg_\lib}}
						{\lcmp{\hrla}}
					} \bigland \lpnot \lpnot \sblk^+ \land \bigland \lpnot \sblk^-
					.
			\]
			and by using Lemma~\ref{se:disjunction in body} we obtain the strongly
			equivalent program that contains, for each $\sblk \in \sblks{\bigcup_{\lia <
			\lib < \lng} \expv{\prg_\lib}}{\lcmp{\hrla}}$, the rule
			\[
				\hrla
				\lpif
				\brla
					\land \bigland \lpnot \lpnot \sblk^+
					\land \bigland \lpnot \sblk^-
				.
			\]
			Double-negated atoms from the bodies of these rules
			can be eliminated using Lemmas~\ref{lemma:se double negation in body} and
			\ref{lemma:se double negation in body with negative head}, obtaining
			$
				\Br{
					\hrla \lor \biglor \lpnot \sblk^+
					\lpif
					\brla \land \bigland \lpnot \sblk^-
					.
				}
			$
			if $\hrla$ is an atom and
			$
				\Br{
					\hrla
					\lpif
					\brla \land \bigland \sblk^+ \land \bigland \lpnot \sblk^-
					.
				}
			$
			if $\hrla$ is a default literal. In this way, the original nested rules
			can be converted, one at a time, into a strongly equivalent disjunctive
			program. After this process is finished, the nested syntax can be
			converted to the syntax of disjunctive programs and the result coincides
			with the disjunctive program $\biguoprjulor \dprg$.

		\item Due to Theorem~\ref{thm:ju binary operator}, it suffices to show
			that the programs $\biguopras \tones{\dprg}$ and $\biguopraslor \dprg$
			have the same stable models. To see that this is indeed the case,
			consider the contents of these programs, as established in
			Lemmas~\ref{lemma:as binary operator:1} and \ref{lemma:as disjunctive
			binary operator:1}. For all $\rl \in \tones{\expv{\prg_\lia}}$ with
			$\lia < \lng - 1$, $\biguopras \dprg$ contains the nested rule
			$
				\Br{
					\hrla
					\lpif
					\brla \land \bigland_{\lia < \lib < \lng}
					\lpnot \acond{\tones{\expv{\prg_\lib}}}{\lcmp{\hrla}}
					.
				}
			$.
			By Lemma~\ref{lemma:strong equivalence:1}\eqref{lemma:strong
			equivalence:1:1} and Corollary~\ref{cor:strong equivalence condition:2},
			this rule is strongly equivalent to the rule
			$
				\Br{
					\hrla
					\lpif
					\brla \land \lpnot \biglor_{\lia < \lib < \lng}
					\acond{\tones{\expv{\prg_\lib}}}{\lcmp{\hrla}}
					.
				}
			$
			which, by the definition of activation condition, can also be written as
			$
				\Br{
					\hrla
					\lpif
					\brla \land \lpnot
					\acond{\bigcup_{\lia < \lib < \lng}
					\tones{\expv{\prg_\lib}}}{\lcmp{\hrla}}
					.
				}
			$.
			Furthermore, due to Lemma~\ref{lemma:strong
			equivalence:1}\eqref{lemma:strong equivalence:1:2} and
			Corollary~\ref{cor:strong equivalence condition:2}, the latter rule is
			strongly equivalent to the rule
			\[
				\hrla
				\lpif
				\brla \land
					\biglor_{\sblk \in \sblks
						{\bigcup_{\lia < \lib < \lng} \expv{\prg_\lib}}
						{\lcmp{\hrla}}
					} \bigland \lpnot \lpnot \sblk^+ \land \bigland \lpnot \sblk^-
					.
			\]
			and by using Lemma~\ref{se:disjunction in body} we obtain the strongly
			equivalent program that contains, for each $\sblk \in \sblks{\bigcup_{\lia <
			\lib < \lng} \expv{\prg_\lib}}{\lcmp{\hrla}}$, the rule
			\[
				\hrla
				\lpif
				\brla
					\land \bigland \lpnot \lpnot \sblk^+
					\land \bigland \lpnot \sblk^-
				.
			\]
			Finally, double-negated atoms from the bodies of these rules can be
			eliminated using Lemmas~\ref{lemma:se double negation in body with
			choice rule} and \ref{lemma:se double negation in body with negative
			head}, obtaining
			$
				\Br{
					\hrla
					\lpif
					\brla
						\land \bigland \sblk^+
						\land \bigland \lpnot \sblk^-
					.
				}
			$
			In this way, the original nested rules can be converted, one at a time,
			into a strongly equivalent disjunctive program. After this process is
			finished, the nested syntax can be converted to the syntax of
			disjunctive programs and the result coincides with the disjunctive
			program $\biguopraslor \dprg$. \qedhere
	\end{textenum}
\end{proof}

\section{Proofs: Belief Updates Using Exception-Based Operators}

\label{app:belupd}

\subsection{Model-Based Update Operators}

\begin{theorem*}{thm:belupd:model-based by exception-based}
	If $\uopb$ is an update operator that satisfies \bu{1}, \bu{2.1} and
	\bu{4}, then there exists an exception function \te{} such that for every
	\te-based update operator $\uope$ and all finite sequences of knowledge
	bases $\dkb$, $\mod{\biguopb \dkb} = \mod{\biguope \dkb}$.
\end{theorem*}
\begin{proof}
	\label{proof:belupd:model-based by exception-based}
	Let the exception function \te{} be defined for all sets of interpretations
	$\stwi \subseteq \twis$ and all sets of sets of interpretations $\sstwia,
	\sstwib \subseteq \pws{\twis}$ as
	\begin{equation}
		\label{eq:proof:belupd:model-based by exception-based:1}
		\e (\stwi, \sstwia, \sstwib) = \mod{\kba \uopb \kbu}
	\end{equation}
	where $\kba$, $\kbu$ are some knowledge bases such that $\mod{\kba} =
	\bigcap \sstwia$ and $\mod{\kbu} = \bigcap \sstwib$. Note that this
	definition is unambiguous since the existence of such $\kba$ and $\kbu$ is
	guaranteed and regardless of which pair of knowledge bases with these
	properties we choose, we obtain the same result due to the assumption that
	$\uopb$ satisfies $\bu{4}$. Take some \te-based operator $\uope$. We
	proceed by induction on the length $\lng$ of $\dkb = \seq{\kb_\lia}_{\lia <
	\lng}$.
	\begin{textenum}[1$^\circ$]
		\item If $\lng = 0$, then it immediately follows that
			$
				\mod{\biguopb \dkb}
				= \mod{\biguopb \seq{\kb_0}}
				= \mod{\kb_0}
				= \mod{\biguope \seq{\kb_0}}
				= \mod{\biguope \dkb}
			$.

		\item Suppose that the claim holds for $\lng$, i.e.\ for $\dkb =
			\seq{\kb_\lia}_{\lia < \lng}$ we have
			$
				\mod{\biguopb \dkb}
				=
				\mod{\biguope \dkb}
			$.
			Our goal is to show that it also holds for $\lng + 1$, i.e.\ for $\dkb'
			= \seq{\kb_\lia}_{\lia < \lng + 1}$. It follows that
			\[
				\modr{\biguope \dkb'}
				= \modr{\biguope \dkb \uope \kb_\lng}
				= \Set{
					\mod{\frm}
						\cup \e(\mod{\frm}, \modr{\biguope \dkb}, \modr{\kb_\lng})
					|
					\frm \in \biguope \dkb
				} \cup \modr{\kb_\lng} \enspace.
				\]
			By \eqref{eq:proof:belupd:model-based by
			exception-based:1} and the inductive assumption,
			$
				\e(\mod{\frm}, \modr{\biguope \dkb}, \modr{\kb_\lng})
				= \mod{\biguopb \dkb \uopb \kb_\lng}
				= \mod{\biguopb \dkb'}
			$.
			Consequently,
			\[
				\mod{\biguope \dkb'}
				= \bigcap \modr{\biguope \dkb'}
				= \bigcap \Br{ \Set{
					\mod{\frm}
						\cup \mod{\biguopb \dkb'}
					|
					\frm \in \biguope \dkb
				} \cup \modr{\kb_\lng} }
				\enspace.
			\]
			This can also be written as
			$
				\Br{
					\mod{\biguopb \dkb'}
					\cup
					\bigcap \modr{\biguope \dkb}
				} \cap \bigcap \modr{\kb_\lng}
			$.
			Substituting $\mod{\biguopb \dkb}$ for $\bigcap \modr{\biguope \dkb}$
			and distributing $\cap$ over $\cup$ yields
			$
				\Br{
					\mod{\biguopb \dkb'}
					\cap
					\mod{\kb_\lng}
				} \cup \Br{
					\mod{\biguopb \dkb}
					\cap 
					\mod{\kb_\lng}
				}
			$.
			Finally, using \bu{1} and \bu{2.1} we can write this as
			\begin{align*}
				\Br{
					\mod{\biguopb \dkb \uopb \kb_\lng}
					\cap
					\mod{\kb_\lng}
				} \cup \mod{(\biguopb \dkb) \cup \kb_\lng}
				&= \mod{\biguopb \dkb \uopb \kb_\lng}
					\cup \mod{(\biguopb \dkb) \cup \kb_\lng} \\
				&= \mod{\biguopb \dkb \uopb \kb_\lng}
				= \mod{\biguopb \dkb'}
				\enspace. \qedhere
			\end{align*}
	\end{textenum}
\end{proof}

\subsection{Formula-Based Update Operators}

The set of possible remainders has a number of important properties from which
properties of specific formula-based operators follow. We start with two
auxiliary results which make it possible to construct a subset of a
knowledge base with important properties on the semantic level.

\begin{lemma}
	\label{lemma:belupd:remof non-empty}
	Let $\kba$, $\kbu$ be knowledge bases. Then $\kbu$ is consistent if and only
	if $\remof{\kba, \kbu}$ is non-empty.
\end{lemma}
\begin{proof}
	First suppose that $\kbu$ is consistent and let $\srem$ be the set of all
	subsets $\kba'$ of $\kba$ such that $\kba' \cup \kbu$ is consistent.
	$\srem$ must be non-empty because $\emptyset$ clearly belongs to $\srem$. So
	take some subset-maximal element $\kba^*$ of $\srem$. It is easy to see that
	$\kba^*$ belongs to $\remof{\kba, \kbu}$.

	On the other hand, if $\remof{\kba, \kbu}$ is non-empty, then it contains
	some set $\kba'$ such that $\kba' \cup \kbu$ is consistent. Thus it follows
	directly that $\kbu$ is also consistent.
\end{proof}

\begin{lemma}
	\label{lemma:belupd:modr subset}
	Let $\kba$, $\kbb$ be knowledge bases with $\modr{\kba}^\twis =
	\modr{\kbb}^\twis$, $\kba' \subseteq \kba$ and
	$
		\kbb' = \Set{ \frm \in \kbb | \mod{\frm} \in \modr{\kba'}^\twis }
	$.
	Then $\modr{\kba'}^\twis = \modr{\kbb'}^\twis$.
\end{lemma}
\begin{proof}
	Suppose first that $\stwi$ belongs to $\modr{\kba'}^\twis$. Then it also
	belongs to $\modr{\kba}^\twis$, so by our assumption either $\stwi = \twis$
	or $\stwi$ belongs to $\modr{\kbb}$. In the former case $\stwi$ belongs to
	$\modr{\kbb'}^\twis$ and we are finished. In the latter case there is a
	formula $\frm \in \kbb$ such that $\mod{\frm} = \stwi$ and $\frm$ belongs to
	$\kbb'$ by its definition. Consequently, $\stri$ belongs to
	$\modr{\kbb'}^\twis$.

	As for the other inclusion, if $\stwi$ belongs to $\modr{\kbb'}^\twis$, then
	either $\stwi = \twis$ or for some formula $\frm \in \kbb'$ we have $\stwi =
	\mod{\frm}$. Therefore, $\stwi$ belongs to $\modr{\kba'}^\twis$ by the
	definition of $\kbb'$.
\end{proof}

\begin{lemma}
	\label{lemma:belupd:consistency}
	Let $\kba$, $\kbb$, $\kbu$, $\kbv$ be knowledge bases such that
	$\modr{\kba}^\twis = \modr{\kbb}^\twis$ and $\modr{\kbu}^\twis =
	\modr{\kbv}^\twis$, $\kba' \subseteq \kba$ such that $\kba' \cup \kbu$
	is consistent and
	$
		\kbb' = \Set{ \frm \in \kbb | \mod{\frm} \in \modr{\kba'}^\twis }
	$.
	Then $\kbb' \cup \kbv$ is consistent.
\end{lemma}
\begin{proof}
	To verify that $\kbb' \cup \kbv$ is consistent, we only need to use Lemma
	\ref{lemma:belupd:modr subset} and observe that
	\begin{align*}
		\mod{\kbb' \cup \kbv}
			&= \bigcap \modr{\kbb' \cup \kbv}
			= \bigcap (\modr{\kbb'} \cup \modr{\kbv})
			= \bigcap (\modr{\kbb'}^\twis \cup \modr{\kbv}) \\
			&= \bigcap (\modr{\kba'}^\twis \cup \modr{\kbu})
			= \bigcap (\modr{\kba'} \cup \modr{\kbu})
			= \bigcap \modr{\kba' \cup \kbu}
			= \mod{\kba' \cup \kbu} \enspace. \qedhere
	\end{align*}
\end{proof}


\begin{proposition}
	[Syntax-Independence of Remainders]
	\label{prop:belupd:remainder syntax-independence}
	Let $\kba$, $\kbb$, $\kbu$, $\kbv$ be knowledge bases such that
	$\modr{\kba}^\twis = \modr{\kbb}^\twis$ and $\modr{\kbu}^\twis =
	\modr{\kbv}^\twis$. Then
	$
		\modrr{\remof{\kba, \kbu}}^\twis = \modrr{\remof{\kbb, \kbv}}^\twis
	$.
\end{proposition}
\begin{proof}
	We prove that $\modrr{\remof{\kba, \kbu}}^\twis \subseteq
	\modrr{\remof{\kbb, \kbv}}^\twis$, the other inclusion follows by the same
	arguments since the formulation of the proposition is symmetric.

	Take some $\kba'$ from $\remof{\kba, \kbu}$ and put
	$
		\kbb' = \Set{ \frm \in \kbb | \mod{\frm} \in \modr{\kba'}^\twis }
	$.
	We need to show that $\modr{\kba'}^\twis$ belongs to $\modrr{\remof{\kbb,
	\kbv}}^\twis$. Due to Lemma \ref{lemma:belupd:modr subset},
	$\modr{\kba'}^\twis = \modr{\kbb'}^\twis$, so it suffices to prove that
	$\kbb'$ belongs to $\remof{\kbb, \kbv}$. First, note that $\kbb'$ is clearly
	a subset of $\kbb$ and, by Lemma \ref{lemma:belupd:consistency}, $\kbb' \cup
	\kbv$ is consistent. We prove that $\kbb'$ is subset-maximal with these
	properties by contradiction. Suppose that $\kbb^*$ is such that $\kbb'
	\subsetneq \kbb^* \subseteq \kbb$ and $\kbb^* \cup \kbv$ is consistent and
	let
	$
		\kba^* = \Set{ \frm \in \kba | \mod{\frm} \in \modr{\kbb^*}^\twis }
	$.
	Clearly, $\kba^*$ is a subset of $\kba$ and, by
	Lemma~\ref{lemma:belupd:consistency}, $\kba^* \cup \kbu$ is consistent. To
	reach a conflict, we need to show that $\kba'$ is a proper subset of
	$\kba^*$. First note that $\modr{\kbb'}^\twis$ cannot be equal to
	$\modr{\kbb^*}^\twis$ -- if it were, then for every formula $\phi \in
	\kbb^*$ it would hold that $\phi$ belongs to $\kbb$ and $\mod{\phi}$ belongs
	to $\modr{\kba'}^\twis$, so $\phi$ belongs to $\kbb'$ by its definition,
	contrary to the assumption that $\kbb'$ is a proper subset of $\kbb^*$.
	This, together with Lemma \ref{lemma:belupd:modr subset}, implies that
	\begin{equation}
		\label{eq:proof:prop:belupd:remainder syntax-independence:1}
		\modr{\kba'}^\twis = \modr{\kbb'}^\twis
		\subsetneq
		\modr{\kbb^*}^\twis = \modr{\kba^*}^\twis
		\enspace.
	\end{equation}
	It immediately follows that $\kba' \neq \kba^*$. Furthermore, for any
	formula $\frm$ from $\kba'$, $\frm$ belongs to $\kba$ and it follows from
	\eqref{eq:proof:prop:belupd:remainder syntax-independence:1} that
	$\mod{\frm}$ belongs to $\modr{\kbb^*}^\twis$, so $\frm$ belongs to $\kba^*$
	by its definition. This means that $\kba'$ is a proper subset of $\kba^*$,
	contrary to the assumption that $\kba'$ belongs to $\remof{\kba, \kbu}$.
\end{proof}


\begin{lemma}
	[Equivalent Formulas in Remainders]
	\label{lemma:belupd:remainder equivalent formulas}
	Let $\kba$, $\kbu$ be knowledge bases, $\frma, \frmb \in \kba$ formulas such
	that $\mod{\frma} = \mod{\frmb}$ and $\kba' \in \remof{\kba, \kbu}$ a
	possible remainder. Then $\frma \in \kba'$ if and only if $\frmb \in \kba'$.
\end{lemma}
\begin{proof}
	Without loss of generality, assume that $\frma$ belongs to $\kba'$ but
	$\frmb$ does not. Then $\kba' \cup \set{\frmb}$ is a subset of $\kba$ that
	is consistent with $\kbu$. This is in conflict with the maximality of
	$\kba'$.
\end{proof}


\begin{corollary}
	\label{cor:belupd:remainder intersection:basic}
	Let $\kba$, $\kbu$ be knowledge bases and $\srem \subseteq \remof{\kba, \kbu}$
	a set of possible remainders. Then
	$
		\bigcap \modrr{\srem} = \modr{\bigcap{\srem}}
	$.
\end{corollary}
\begin{proof}
	First suppose that $\stwi$ belongs to $\bigcap \modrr{\srem}$ and take some
	$\kba' \in \srem$ and some formula $\frma \in \kba'$ such that $\mod{\frma}
	= \stwi$. Now take an arbitrary $\kba^* \in \srem$. Since $\stwi$ belongs to
	$\modr{\kba^*}$, there must exist a formula $\frmb \in \kba^*$ such that
	$\mod{\frmb} = \stwi$. Consequently, $\mod{\frma} = \mod{\frmb}$ and by
	Lemma~\ref{lemma:belupd:remainder equivalent formulas} we obtain that
	$\frma$ also belongs to $\kba^*$. Thus, $\frma$ belongs to $\bigcap \srem$
	and $\stwi$ belongs to $\modr{\bigcap \srem}$.

	On the other hand, if $\stwi$ belongs to $\modr{\bigcap \srem}$, then there
	is a formula $\phi \in \bigcap \srem$ such that $\mod{\phi} = \stwi$.
	Consequently, $\stwi$ belongs to all members of $\modrr{\srem}$, thus also
	belongs to their intersection.
\end{proof}

\begin{corollary}
	\label{cor:belupd:remainder intersection}
	Let $\kba$, $\kbu$ be knowledge bases, $\srem \subseteq \remof{\kba, \kbu}$
	a set of possible remainders. Then,
	$
		\bigcap \modrr{\srem}^\twis = \modr{\bigcap{\srem}}^\twis
	$.
\end{corollary}
\begin{proof}
	Follows from Corollary~\ref{cor:belupd:remainder intersection:basic} and
	from the fact that $\twis$ belongs to both sides of the equation.
\end{proof}

\begin{proposition}
	[Properties of the WIDTIO Operator]
	\label{prop:belupd:widtio properties}
	The WIDTIO operator satisfies \fu{1}, \fu{2.1} and \fu{4}.
\end{proposition}
\begin{proof}
	By definition $\kbu \subseteq \kba \uopwidtio \kbu$ and \fu{1} is obtained
	by applying $\modr{\cdot}$ to both sides of this inclusion.

	In order to verify that \fu{2.1} holds, suppose that $\stwi$ belongs to
	$\modr{\kba \uopwidtio \kbu}$. Then there is some formula $\frm$ from $\kbu
	\cup \bigcap \remof{\kba, \kbu}$ such that $\mod{\frm} = \stwi$. If $\frm$
	belongs to $\kbu$, then it immediately follows that $\stwi$ belongs to
	$\modr{\kbu}$, and consequently also to $\modr{\kba \cup \kbu}$. If $\frm$
	belongs to $\kba'$ for all $\kba' \in \remof{\kba, \kbu}$, then $\frm$ also
	belongs to $\kba$. Thus, $\stwi$ is a member of $\modr{\kba}$, and
	consequently also of $\modr{\kba \cup \kbu}$.

	Finally, to verify \fu{4}, suppose that $\modr{\kba}^\twis =
	\modr{\kbb}^\twis$ and $\modr{\kbu}^\twis = \modr{\kbv}^\twis$. The
	following follows from the definition of the WIDTIO operator,
	Corollary~\ref{cor:belupd:remainder intersection} and
	Proposition~\ref{prop:belupd:remainder syntax-independence}.
	\begin{align*}
		\modr{\kba \uopwidtio \kbu}^\twis
		&= \Modr{\kbu \cup \bigcap \remof{\kba, \kbu}}^\twis
		= \modr{\kbu}^\twis \cup \Modr{\bigcap \remof{\kba, \kbu}}^\twis \\
		&= \modr{\kbu}^\twis \cup \bigcap \modrr{\remof{\kba, \kbu}}^\twis
		= \modr{\kbv}^\twis \cup \bigcap \modrr{\remof{\kbb, \kbv}}^\twis \\
		&= \modr{\kbv}^\twis \cup \Modr{\bigcap \remof{\kbb, \kbv}}^\twis
		= \Modr{\kbv \cup \bigcap \remof{\kbb, \kbv}}^\twis
		= \modr{\kbb \uopwidtio \kbv}^\twis
		\enspace. \qedhere
	\end{align*}
\end{proof}

\begin{proposition}
	[Properties of Regular Bold Operators]
	\label{prop:belupd:bold properties}
	Regular Bold operators satisfy \fu{1}, \fu{2.1} and \fu{4}.
\end{proposition}
\begin{proof}
	By definition $\kbu \subseteq \kba \uopbold \kbu$ and \fu{1} is obtained by
	applying $\modr{\cdot}$ to both sides of this inclusion.

	In order to verify that \fu{2.1} holds, suppose that $\stwi$ belongs to
	$\modr{\kba \uopbold \kbu}$. Then there is some formula $\frm$ from $\kbu
	\cup \sfn(\remof{\kba, \kbu})$ such that $\mod{\frm} = \stwi$. If $\frm$
	belongs to $\kbu$, then it immediately follows that $\stwi$ belongs to
	$\modr{\kbu}$, and consequently also to $\modr{\kba \cup \kbu}$. If $\frm$
	belongs to $\kba' = \sfn(\remof{\kba, \kbu})$, then $\frm$ also belongs to
	$\kba$. Thus, $\stwi$ is a member of $\modr{\kba}$, and consequently also of
	$\modr{\kba \cup \kbu}$.

	Finally, to verify \fu{4}, suppose that $\modr{\kba}^\twis =
	\modr{\kbb}^\twis$ and $\modr{\kbu}^\twis = \modr{\kbv}^\twis$. The
	following follows from the definition of the WIDTIO operator,
	Proposition~\ref{prop:belupd:remainder syntax-independence} and the
	regularity property of $\uopbold$.
	\begin{align*}
		\modr{\kba \uopbold \kbu}^\twis
		&= \modr{\kbu \cup \sfn(\remof{\kba, \kbu})}^\twis
		= \modr{\kbu}^\twis \cup \modr{\sfn(\remof{\kba, \kbu})}^\twis \\
		&= \modr{\kbv}^\twis \cup \modr{\sfn(\remof{\kbb, \kbv})}^\twis
		= \modr{\kbv \cup \sfn(\remof{\kbb, \kbv})}^\twis
		= \modr{\kbb \uopbold \kbv}^\twis
		\enspace. \qedhere
	\end{align*}
\end{proof}

\begin{proposition}
	[Properties of the Cross-Product Operator]
	\label{prop:belupd:cp properties}
	The Cross-Product operator satisfies \fu{1}, \bu{2.1} and \fu{4} but does
	not satisfy \fu{2.1}.
\end{proposition}
\begin{proof}
	By definition $\kbu \subseteq \kba \uopcross \kbu$ and \fu{1} is
	obtained by applying $\modr{\cdot}$ to both sides of this inclusion.

	To see that $\uopcross$ does not satisfy \fu{2.1}, note that
	$
		\set{\atma, \atmb} \uopcross \set{\lnot \atma \lor \lnot \atmb}
			= \set{\atma \lor \atmb, \lnot \atma \lor \lnot \atmb}
	$
	and $\mod{\atma \lor \atmb}$ does not belong to $\modr{\set{\atma, \atmb,
	\lnot \atma \lor \lnot \atmb}}$.

	In order to verify \bu{2.1}, take some $\twi$ from $\mod{\kba \cup
	\kbu}$. We need to show that $\twi$ is a model of $\kba \uopcross \kbu$.
	Obviously, $\twi$ is a model of $\kbu$, so it remains to prove that $\twi$
	is a model of the formula
	\[
		\frmb
			= \biglor_{\kba' \in \remof{\kba, \kbu}}
				\bigland_{\frma \in \kba'} \frma \enspace.
	\]
	Since $\twi$ is a model of $\kbu$, we conclude that $\kbu$ is consistent, so
	according to Lemma~\ref{lemma:belupd:remof non-empty}, $\remof{\kba, \kbu}$
	is non-empty. Take some $\kba^*$ from $\remof{\kba, \kbu}$. We obtain the
	following:
	$
		\mod{\frmb}
			= \bigcup_{\kba' \in \remof{\kba, \kbu}} \mod{\kba'}
			\supseteq \mod{\kba^*}
			\supseteq \mod{\kba}
	$.
	Hence, since $\twi$ belongs to $\mod{\kba}$, it also belongs to
	$\mod{\frmb}$.

	Finally, to verify \fu{4}, suppose that $\modr{\kba}^\twis =
	\modr{\kbb}^\twis$ and $\modr{\kbu}^\twis = \modr{\kbv}^\twis$ and take some
	$\stwi$ from $\modr{\kba \uopcross \kbu}^\twis$. In the trivial case when
	$\stwi = \twis$ it immediately follows that $\stwi$ belongs to $\modr{\kbb
	\uopcross \kbv}^\twis$. Otherwise, there is a formula $\frma$ from $\kbu
	\cup \set{\frmb}$ such that $\mod{\frma} = \stwi$. If $\frma$ belongs to
	$\kbu$, then $\stwi$ belongs to $\modr{\kbu}$ and by assumption also to
	$\modr{\kbv}^\twis$. By \fu{1} we then obtain that $\stwi$ belongs to
	$\modr{\kbb \uopcross \kbv}^\twis$. On the other hand, if $\frma$ is
	$\frmb$, then due to Proposition~\ref{prop:belupd:remainder
	syntax-independence}, $\modrr{\remof{\kba, \kbu}}^\twis =
	\modrr{\remof{\kbb, \kbv}}^\twis$, so
	\[
		\mod{\frma} = \mod{\frmb}
			= \bigcup_{\kba' \in \remof{\kba, \kbu}} \mod{\kba'}
			= \bigcup_{\kbb' \in \remof{\kbb, \kbv}} \mod{\kbb'}
			= \mod{\frmb'}
	\]
	where $\kbb \uopcross \kbv = \kbv \cup \set{\frmb'}$. Therefore,
	$\mod{\frma}$ belongs to $\modr{\kbb \uopcross \kbv}^\twis$. The proof of
	the other inclusion is symmetric.
\end{proof}

\begin{proof}
	[\textbf{Proof of Proposition~\ref{prop:belupd:formula-based properties}}]
	\label{proof:belupd:formula-based properties}
	Follows from Propositions~\ref{prop:belupd:widtio properties},
	\ref{prop:belupd:bold properties} and \ref{prop:belupd:cp properties}.
\end{proof}

\begin{proposition}
	\label{prop:belupd:formula-based by exception-based:1}
	If $\uopf$ is an update operator that satisfies \fu{1}, \fu{2.1} and \fu{4},
	then there exists an exception function \te{} such that for every \te-based
	update operator $\uope$ and all finite sequences of knowledge bases $\dkb$,
	$\modr{\biguopf \dkb}^\twis = \modr{\biguope \dkb}^\twis$.
\end{proposition}
\begin{proof}
	Let the exception function \te{} be defined for all sets of interpretations
	$\stwi$ and all sets of sets of interpretations $\sstwia$, $\sstwib$ as
	\[
		\e(\stwi, \sstwia, \sstwib) = \begin{cases}
			\emptyset & \stwi \in \modr{\kba \uopf \kbu}^\twis \enspace; \\
			\twis & \stwi \notin \modr{\kba \uopf \kbu}^\twis \enspace,
		\end{cases}
	\]
	where $\kba$, $\kbu$ are some knowledge bases such that $\modr{\kba}^\twis =
	\sstwia \cup \set{\twis}$ and $\modr{\kbu}^\twis = \sstwib \cup
	\set{\twis}$. Note that this definition is unambiguous since the existence
	of such $\kba$ and $\kbu$ is guaranteed and regardless of which pair of knowledge
	bases with these properties we choose, we obtain the same result due to the
	assumption that $\uopf$ satisfies $\fu{4}$. Take some \te-based operator
	$\uope$. We proceed by induction on the length $\lng$ of $\dkb =
	\seq{\kb_\lia}_{\lia < \lng}$.
	\begin{textenum}[1$^\circ$]
		\item If $\lng = 0$, then it immediately follows that
			$
				\modr{\biguopb \dkb}
				= \modr{\biguopb \seq{\kb_0}}
				= \modr{\kb_0}
				= \modr{\biguope \seq{\kb_0}}
				= \modr{\biguope \dkb}
			$.

		\item Suppose that the claim holds for $\lng$, i.e.\ for $\dkb =
			\seq{\kb_\lia}_{\lia < \lng}$ we have
			$
				\modr{\biguopb \dkb}^\twis
				=
				\modr{\biguope \dkb}^\twis
			$.
			Our goal is to show that it also holds for $\lng + 1$, i.e.\ for $\dkb'
			= \seq{\kb_\lia}_{\lia < \lng + 1}$. It follows that
			\[
				\modr{\biguope \dkb'}
				= \modr{\biguope \dkb \uope \kb_\lng}
				= \Set{
					\mod{\frm}
						\cup \e(\mod{\frm}, \modr{\biguope \dkb}, \modr{\kb_\lng})
					|
					\frm \in \biguope \dkb
				} \cup \modr{\kb_\lng} \enspace.
			\]
			Thus,
			$
				\modr{\biguope \dkb'}^\twis
				= 
				\Set{
					\mod{\frm} \cup \e(\mod{\frm}, \modr{\biguope \dkb}, \modr{\kb_\lng})
					|
					\frm \in \biguope \dkb
				}
				\cup
				\modr{\kb_\lng} \cup \Set{\twis}
			$
			which in turn can be written as
			\[
				\Set{
					\stwi
					|
					\stwi \in \modr{\biguope \dkb}
					\cap
					\modr{\biguope \dkb \uopf \kb_\lng}^\twis
				}
				\cup
				\Set{
					\twis
					|
					\stwi \in \modr{\biguope \dkb}
					\setminus
					\modr{\biguope \dkb \uopf \kb_\lng}^\twis
				}
				\cup
				\modr{\kb_\lng}
				\cup
				\Set{\twis}
			\]
			and simplified into
			$
				\Br{
					\modr{\biguope \dkb}^\twis
					\cap
					\modr{\biguope \dkb \uopf \kb_\lng}^\twis
				}
				\cup \modr{\kb_\lng}^\twis
			$.
			Since $\uopf$ satisfies \fu{4}, it follows from the inductive assumption
			that $\modr{\biguope \dkb \uopf \kb_\lng}^\twis = \modr{\biguopf \dkb
			\uopf \kb_\lng}^\twis$. Thus, we obtain the set
			\[
				\Br{
					\modr{\biguopf \dkb}^\twis
					\cap
					\modr{\biguopf \dkb \uopf \kb_\lng}^\twis
				}
				\cup \modr{\kb_\lng}^\twis
			\]
			and by distributing $\cup$ over $\cap$ and using \fu{1} and \fu{2.1} we
			obtain
			\begin{align*}
				\Br{
					\modr{\biguopf \dkb}^\twis
					\cup
					\modr{\kb_\lng}^\twis
				}
				\cap
				\Br{
					\modr{\biguopf \dkb \uopf \kb_\lng}^\twis
					\cup
					\modr{\kb_\lng}^\twis
				}
				&=
				\modr{\biguopf \dkb \cup \kb_\lng}^\twis
				\cap
				\modr{\biguopf \dkb \uopf \kb_\lng}^\twis \\
				&= 
				\modr{\biguopf \dkb \uopf \kb_\lng}^\twis
				=
				\modr{\biguopf \dkb'}^\twis
				\enspace. \qedhere
			\end{align*}
	\end{textenum}
\end{proof}

\begin{proposition}
	\label{prop:belupd:formula-based by exception-based:2}
	If $\uopf$ an update operator that satisfies \fu{1}, \bu{2.1} and \fu{4},
	then there exists an exception function \te{} such that for every \te-based
	update operator $\uope$ and all knowledge bases $\kba$, $\kbu$, $\mod{\kba
	\uopf \kbu} = \mod{\kba \uope \kbu}$.
\end{proposition}
\begin{proof}
	Let the exception function \te{} be defined for all sets of interpretations
	$\stwi$ and all sets of sets of interpretations $\sstwia$, $\sstwib$ as
	\[
		\e(\stwi, \sstwia, \sstwib) = \mod{\kba \uopf \kbu} \enspace,
	\]
	where $\kba$, $\kbu$ are some knowledge bases such that $\modr{\kba}^\twis =
	\sstwia \cup \set{\twis}$ and $\modr{\kbu}^\twis = \sstwib \cup
	\set{\twis}$. Note that this definition is unambiguous since the existence
	of such $\kba$ and $\kbu$ is guaranteed and regardless of which pair of
	knowledge bases with these properties we choose, we obtain the same result
	due to the assumption that $\uopf$ satisfies $\fu{4}$.

	Take some \te-based operator $\uope$. Then
	$
		\mod{\kba \uope \kbu}
		=
		\bigcap \Br{
			\Set{
				\mod{\frm} \cup \e(\mod{\frm}, \modr{\kba}, \modr{\kbu})
				|
				\frm \in \kba
			}
			\cup \modr{\kbu}
		}
	$,
	which can be written as
	$
		\bigcap \Set{
			\mod{\frm} \cup \mod{\kba \uopf \kbu}
			|
			\frm \in \kba
		}
		\cap \bigcap \modr{\kbu}
	$
	and simplified into
	$
		(\mod{\kba \uopf \kbu} \cup \mod{\kba}) \cap \mod{\kbu}
	$.
	Furthermore, due to \bu{1} and \bu{2.1},
	\begin{align*}
		(\mod{\kba \uopf \kbu} \cup \mod{\kba}) \cap \mod{\kbu}
		&= (\mod{\kba \uopf \kbu} \cap \mod{\kbu})
			\cup
			(\mod{\kba} \cap \mod{\kbu})
		= \mod{\kba \uopf \kbu} \cup \mod{\kba \cup \kbu}
		= \mod{\kba \uopf \kbu}
		\enspace. \qedhere
	\end{align*}
\end{proof}

%
\begin{proof}
	[\textbf{Proof of Theorem~\ref{thm:belupd:formula-based by exception-based}}]
	\label{proof:belupd:formula-based by exception-based}
	Follows from Propositions~\ref{prop:belupd:formula-based by
	exception-based:1} and \ref{prop:belupd:formula-based by exception-based:2}.
\end{proof}

{
	\bibliographystyle{abbrvnat}

}

\end{document}